\documentclass[11pt]{article}
\usepackage[utf8]{inputenc}
\usepackage[margin=1in]{geometry}
\usepackage{graphicx} % Required for inserting images
\usepackage{enumitem}
\usepackage{comment}
\usepackage{amsmath}
\usepackage{amsfonts}
\usepackage{amsthm, amssymb}
\usepackage{thm-restate}
\usepackage{booktabs}
\usepackage[]{hyperref}
\hypersetup{
    colorlinks = true,
    linkcolor = blue,
    citecolor = blue
    }
    
\usepackage[
        noabbrev,
        capitalise,
        nameinlink,
    ]
    {cleveref}
\usepackage[dvipsnames]{xcolor}
\usepackage{bbm}
\usepackage{dirtytalk}
\usepackage{subcaption}
\usepackage{float}
\usepackage{tikz}
\usepackage{algorithm}
\usepackage[noend]{algpseudocode}
\usepackage{makecell}
\usetikzlibrary{arrows.meta,decorations.markings, calc, positioning}



\newcommand{\man}{M}

\definecolor{cyan}{cmyk}{1, 0.4, 0, 0}

\newcommand{\PT}{\operatorname{PT}}

\usepackage[round]{natbib} 
\title{Wasserstein Parallel Transport for Predicting the Dynamics of Statistical Systems}

\author{
\begin{tabular}{cccc} % Changed to 4 columns
\multicolumn{2}{c}{\makecell[c]{Tristan Luca Saidi\textsuperscript{\(\dagger\)}\\ \href{mailto:tsaidi@andrew.cmu.edu}{\nolinkurl{tsaidi@andrew.cmu.edu}}}} &
\multicolumn{2}{c}{\makecell[c]{Gonzalo Mena\textsuperscript{\(\dagger\)}\\ \href{mailto:gmena@andrew.cmu.edu}{\nolinkurl{gmena@andrew.cmu.edu}}}} \\[2em]
\multicolumn{2}{c}{\makecell[c]{Larry Wasserman\textsuperscript{\(\dagger, \ddagger\)}\\ \href{mailto:larry@stat.cmu.edu}{\nolinkurl{larry@stat.cmu.edu}}}} &
\multicolumn{2}{c}{\makecell[c]{Florian Gunsilius\textsuperscript{\(*\)}\\ \href{mailto:florian.felix.gunsilius@emory.edu}{\nolinkurl{fgunsil@emory.edu}}}} \\[1.5em]
\end{tabular}
\small\\ % Added a line break here to force affiliations down
\textsuperscript{\(\dagger\)}Department of Statistics and Data Science, Carnegie Mellon University\\
\textsuperscript{\(\ddagger\)}Machine Learning Department, Carnegie Mellon University\\
\textsuperscript{\(*\)}Department of Economics, Emory University
}

\begin{document}

\maketitle

\begin{abstract}
    Many scientific systems, such as cellular populations or economic cohorts, are naturally described by probability distributions that evolve over time. Predicting how such a system would have evolved under different forces or initial conditions is fundamental to causal inference, domain adaptation, and counterfactual prediction. However, the space of distributions often lacks the vector space structure on which classical methods rely. To address this, we introduce a general notion of parallel dynamics at a distributional level. We base this principle on parallel transport of tangent dynamics along optimal transport geodesics and call it ``Wasserstein Parallel Trends''. By replacing the vector subtraction of classic methods with geodesic parallel transport, we can provide counterfactual comparisons of distributional dynamics in applications such as causal inference, domain adaptation, and batch-effect correction in experimental settings. The main mathematical contribution is a novel notion of fanning scheme on the Wasserstein manifold that allows us to efficiently approximate parallel transport along geodesics while also providing the first theoretical guarantees for parallel transport in the Wasserstein space. We also show that Wasserstein Parallel Trends recovers the classic parallel trends assumption for averages as a special case and derive closed-form parallel transport for Gaussian measures. We deploy the method on synthetic data and two single-cell RNA sequencing datasets to impute gene-expression dynamics across biological systems.
\end{abstract}

\tableofcontents

\section{Introduction}
Many scientific systems are naturally described by probability distributions that evolve over time, such as  cellular populations undergoing differentiation, income distributions responding to policy changes, or ecological communities adapting to environmental shifts. A fundamental task is to take the dynamics observed for one such system and predict how a second system with different properties would have evolved under similar dynamics. This counterfactual question shows up in several distinct guises; for instance, as causal effect estimation in observational studies, as out-of-sample prediction when only one system is observed, and even in dynamic randomized treatment estimation as we show below. In each case, the challenge is the same, namely, to transfer a temporal trend from one distributional baseline to another.

Existing approaches to transferring dynamics rely on taking differences. In causal inference with observational data for instance, Difference-in-Differences subtracts the observed change in the reference group, relying on the parallel trends assumption \citep{abadie2005semiparametric, ashenfelter1985using}. Powerful recent extensions \citep[e.g.][]{zhou2025geodesic} generalize this to geodesic spaces by requiring an abelian group structure, which is needed to define an analogous notion of ``difference''. We take a fundamentally different approach: rather than computing differences between systems, \emph{we directly transport the instantaneous dynamics} (the tangent velocity driving the evolution of one system) to the baseline of another via geodesic parallel transport on the space of probability measures. This literally constructs a parallel trend. The method operates in continuous time and acts on full distributions rather than moments. Moreover, because the method constructs counterfactual trajectories directly rather than backing out treatment effects from assumed parallelism, it is widely applicable. For instance, it allows for the correction of batch effects in randomized trials or the imputation of unobserved dynamics when only one system's dynamics are observed. 

Our main technical contribution is the introduction of a novel fanning scheme \citep{louis2018fanning} on the Wasserstein manifold that allows us to efficiently approximate parallel transport along geodesics while providing the first theoretical guarantees for parallel transport in the Wasserstein space. By leveraging Jacobi fields on the base manifold \citep{gigli2012second}, the scheme circumvents the nonlinear PDE that characterizes exact Wasserstein parallel transport \citep{ambrosio2008construction}. We further show that curves in $\mathcal{P}_2(\mathbb{R}^d)$ that are parallel in the sense of the Levi-Civita connection on the Wasserstein manifold necessarily have parallel means, recovering the classical equibias condition \citep[e.g.][]{sofer2016} for averages as a special case. We also derive closed-form expressions for parallel transport between Gaussian measures by reducing the problem to a continuous Lyapunov equation.

Our main application domain is single-cell genomics, where transcriptomic technologies produce snapshots of gene-expression distributions at discrete time points under varying biological conditions \citep{schiebinger2019optimal, schiebinger2021reconstructing, lavenant2024toward}. Because measuring a cell's state destroys it, individual trajectories are unobservable, and one must reason at the population level---a setting where optimal transport has become a standard modeling tool \citep{bunne2024optimal}. A natural use case arises in controlled perturbation experiments, where two nominally clonal populations are cultured in parallel and one receives treatment. In many such settings—particularly in vitro systems or carefully designed time-course experiments—samples from control and treated arms are collected and processed in a coordinated manner across timepoints, so that technical variation is largely shared over time. If both populations were distributionally identical at baseline, the control trajectory could serve directly as the counterfactual for the treated population. In practice, however, even under these controlled conditions, replicates often exhibit systematic baseline differences, commonly referred to as \emph{batch effects}, arising from factors such as sample handling or sequencing variation \citep{tran2020benchmark, luecken2022benchmarking}. In these regimes, it is often reasonable to approximate such discrepancies as structured, approximately time-invariant shifts between distributions. Under this assumption, our method provides a geometric framework for aligning dynamics across baselines: by parallel transporting the control dynamics along the geodesic connecting the two initial distributions, we transfer the observed temporal trend to the treated population while accounting for their initial misalignment. This approach operates on full distributions rather than low-dimensional summaries, capturing differences beyond simple location shifts. Its validity, however, relies on the extent to which the baseline discrepancy can be meaningfully represented as a stable transformation across time, and is therefore best suited to settings where time-varying technical effects or large compositional changes are limited.

Our work connects to several strands of literature. The problem of predicting counterfactual dynamics is ubiquitous in causal inference and domain adaptation. In the causal inference literature, Difference-in-Differences (DiD) is the most prominent framework for leveraging parallel trends to estimate treatment effects \citep{ashenfelter1985using, card1994minimum, abadie2005semiparametric, roth2023s}. Distributional extensions have been pursued by \citet{athey2006identification}, who propose univariate changes-in-changes using monotone transport, \citet{torous2024optimal}, who generalize the setting to multivariate outcomes, and by \citet{sofer2016}, who formulate distributional DiD through quantile transport maps. \citet{callaway2019quantile} introduce parallel trends via a copula stability assumption. \citet{zhou2025geodesic} extend DiD to geodesic spaces with an abelian group structure. Other closely related approaches are distributional extensions of the synthetic controls method \citep{abadie2010synthetic, gunsilius2023distributional, gunsilius2024tangential}, which also create parallel trends, but via averaging of possible control dynamics. Our method is applicable to all of these settings but is not limited to causal inference \citep{gunsilius2025primer}: it provides a general-purpose tool for constructing counterfactual trajectories whenever one wishes to transfer observed dynamics to a new distributional baseline.

Prior seminal uses of parallel transport on the Wasserstein manifold are due to \citet{petersen2019wasserstein} and \citet{chen2023wasserstein}, who employ it for Fr\'{e}chet regression; both focus on the univariate setting and address a fundamentally different problem. Our framework can also be viewed as a dynamic counterpart to optimal transport-based domain adaptation \citep{courty2016optimal}, which aligns source and target distributions at a single time point; we instead transport temporal dynamics themselves across distributional baselines.

In single-cell genomics, the main application in this paper,  \citet{bunne2023learning} use neural optimal transport to learn maps from control to perturbed cell distributions. Their approach is static in that it does not model temporal dynamics, and learns a direct map between observed conditions rather than constructing a counterfactual against which a treatment effect can be measured. Our framework, by contrast, operates on evolving distributions over time, produces explicit counterfactual trajectories that enable treatment-control comparisons, and provides a testable identifying assumption---Wasserstein Parallel Trends---supporting causal inference from observational genomics data alongside pure counterfactual prediction in experimental settings.

The remainder of the paper is structured as follows. \Cref{sec:background} reviews the background on optimal transport and the Riemannian geometry of the Wasserstein space. \Cref{sec:parallel} develops the theory of parallel transport on Wasserstein space over manifolds, introduces the novel fanning scheme, and establishes its theoretical guarantees. \Cref{sec:counterfactual_dyn} presents the counterfactual dynamics prediction procedure, introduces the Wasserstein Parallel Trends assumption, and shows that it recovers the classical equibias condition. \Cref{sec: experiments} contains experiments on synthetic Gaussian systems and two single-cell RNA sequencing datasets. All proofs are collected in the appendix.

\section{Setup and Background}\label{sec:background}

In this section we give an overview on the mathematical underpinnings of optimal transport and Riemannian geometry as they pertain to our goals. For a fully detailed treatment of the theory of optimal transport, we recommend \citet{villani2009optimal}. For a comprehensive treatment of Riemannian and differential geometry, we recommend either \citet{lee2018introduction} or \citet{petersen2006riemannian}, and \citet{do2016differential}. The Riemannian perspective on optimal transport that we describe in this section follows closely the constructions described by \citet{otto2001geometry}, \citet{gigli2012second} and \citet{lott2006some}, and later expanded on by \citet{clancy2021interpolating}.  

\subsection{Optimal Transport} Define $\mathcal{P}_{2}(\man)$ to be the set of probability measures with finite second moment over a complete, connected and $C^{\infty}$ Riemannian manifold $\man$ with metric tensor $g$. For the sake of generality and applicability to concurrent work, we will take $\man$ to be a Riemannian manifold satisfying the properties described above, but all of our algorithmic instantiations and experiments will consider $\man = \mathbb{R}^d$. The $2$-Wasserstein distance between two probability measures $\mu, \nu \in \mathcal{P}_{2}(\man)$ is defined by
\begin{equation}
    W_2(\mu, \nu) = \inf_{\gamma \in \Gamma_{\mu, \nu}}\left(\int_{\man \times \man} d_M(x,y)^2\,\gamma(dx, dy)\right)^{1/2} \label{eq: 2 wasserstein dist}
\end{equation}
where $\Gamma_{\mu, \nu}$ denotes the set of couplings of $\mu$ and $\nu$. The $2$-Wasserstein distance (which we will henceforth refer to as the Wasserstein distance) is indeed a metric, which renders $(\mathcal{P}_{2}(\man), W_2)$ a metric space. A key result is that of Brenier, who showed that for $\man = \mathbb{R}^d$, the optimal coupling takes the form $(X, \nabla \varphi(X)), X \sim \mu$ for some convex $\varphi$ when $\mu$ is absolutely continuous with respect to the Lebesgue measure.

\begin{restatable}[Brenier]{thm}{} \label{brenier}
    Let $\mu, \nu \in \mathcal{P}_2(\mathbb{R}^d)$ be probability measures such that $\mu$ has a density, and let $X \sim \mu$. If $\gamma^*$ is optimal for \cref{eq: 2 wasserstein dist} with $\man = \mathbb{R}^d$, then there exists a convex function $\varphi:\mathbb{R}^d \rightarrow \mathbb{R}$ such that $(X, \nabla\varphi(X))\sim \gamma^*$. 
\end{restatable}

This theorem guarantees that if the source measure has a density, then the optimal transport coupling can be written as the source measure $\mu$ and $\nabla \varphi_\#\mu$, the pushforward of the source under some convex function -- note that the pushforward of a measure $\mu$ under a map $T$ is simply the measure $(T_\#\mu) (B) = \mu(T^{-1}(B))$ for all measurable $B$. The function $\nabla\varphi$ from \Cref{brenier} is often referred to as the \textit{Brenier map}. This key theorem of Brenier will be central to our construction of the Riemannian structure on the space of probability measures. We note that Brenier's theorem was later generalized to Riemannian manifolds, as described below. The result requires generalizing the notion of convexity and concavity to non-Euclidean settings. In particular, given a function $\psi: M \rightarrow \mathbb{R} \cup \{\pm \infty\}$ its \textit{infimal convolution} $\psi^c$ with a cost function $c$ is defined by 
\[\psi^c(y) = \inf_{x \in M}\left\{c(x,y) - \psi(x)\right\}.\]
We say $\psi$ is $c$-concave if and only if $\psi^{cc} = (\psi^c)^c = \psi.$

\begin{restatable}[Brenier-McCann, \citet{mccann2001polar}]{thm}{} \label{brenier-mcann}
    Let $(M, g)$ be a complete Riemannian manifold and let $\mu, \nu \in \mathcal{P}_2(\man)$ with $\mu \ll \text{vol}_g$. Then there exists a $c$-concave function $\varphi: \man \rightarrow \mathbb{R}$ with $c(x,y) = d_M(x,y)^2$ such that the optimal plan (in the sense of \Cref{eq: 2 wasserstein dist}) is induced by a $\mu$-a.e. unique map
    \[T(x) = \exp_x\left(-\nabla\varphi(x)\right), \quad \gamma^* = (\text{id}, T)_\# \mu\]
    where $\nabla$ is the Riemannian gradient and $\exp_p(v)$ is the Riemannian exponential map.
\end{restatable}

\subsection{Dynamic Formulation}

The optimal transport problem described above -- one which looks for an optimal coupling of source and target measures -- is often referred to as the \textit{static} formulation of optimal transport. An equivalent perspective comes from fluid mechanics, which arrives at the same metric structure through a different formulation. In particular, let $(v_t)_{t \geq 0}$ be a time dependent family of vector fields over $\man$, and consider the ODE $\dot X_t = v_t(X_t)$. Let $\mu_t$ denote the law of $X_t$, where $X_0 \sim \mu_0$ and $X_t$ evolves according to the ODE described previously. Then, the dynamics of $\mu_t$ obey the so-called \textit{continuity equation},
\begin{equation}
    \partial_t\mu_t + \operatorname{div}_g (\mu_t v_t) = 0 \label{eq: balanced continuity eqn}
\end{equation}
in the weak (or distributional) sense, where $\operatorname{div}_g$ is the Riemannian divergence. 
\begin{restatable}[Weak solutions to \Cref{eq: balanced continuity eqn}, \citet{santambrogio2015optimal}]{defn}{} \label{def: weak solutions}
    We say that a family of pairs $(\mu_t, v_t)$ solves the continuity equation on $(0, T)$ in the distributional sense if for any bounded and Lipschitz test function $\varphi \in C^1_c((0, T)\times \man)$ we have
    \[\int_0^T\int_{\man} \partial_t\varphi \, d\mu_tdt + \int_0^T\int_{\man} \langle\nabla\varphi,  v_t\rangle_g \, d\mu_tdt = 0.\]
\end{restatable}
It turns out, one can identify optimal transport maps with vector fields satisfying the continuity equation that are optimal in a certain sense. This gives rise to the \textit{dynamic} formulation of optimal transport.
\begin{restatable}[Benamou-Brenier, \citet{chewi2024statisticaloptimaltransport, ambrosio2012user}]{thm}{}
    Let $\mu_0, \mu_1 \in \mathcal{P}_{2}(\man)$ be absolutely continuous with respect to $\text{vol}_g$. Then 
    \begin{equation*}
        W_2^2(\mu_0, \mu_1) = \inf \left\{\int_{0}^1\|v_t\|^2_{L^2(\mu_t)} dt \, \Bigg| \, (\mu_t, v_t)_{t \in [0,1]} \text{ satisfies \cref{def: weak solutions}}\right\}.
    \end{equation*}
    Moreover, if $M = \mathbb{R}^d$ then the optimal curve $(\mu_t)_{t\geq 0}$ is unique and is described by $X_t \sim \mu_t$, where $X_t = (1-t)X_0 + tX_1$ and $(X_0, X_1) \sim \gamma^* \in \Gamma_{\mu_0, \mu_1}$ with $\gamma^*$ being an optimal coupling.
\end{restatable}

\subsection{Riemannian Structure of the Wasserstein Space}

\noindent \textbf{The Tangent Space.} The dynamic formulation of optimal transport will be a key ingredient in establishing the Riemannian structure of $(\mathcal{P}_2(\man), W_2)$. In particular, we will construct the tangent space from the space of solutions to the continuity equation, as they represent the vector space of infinitesimal perturbations to a measure. But seeing as there are an infinite number of solutions to the continuity equation, (as adding a divergence free field does not change the marginal behavior of $\mu_t$) we need to establish a selection principle. To do so, we will define a notion of a \say{derivative} in a general metric space.

\begin{restatable}[Metric Derivative]{defn}{}
    Let $(\mathcal{X}, d)$ be a metric space and $(x_t)_{t \geq 0}$ be a curve in $\mathcal{X}$. The metric derivative of the curve at time $t$ is given by
    \begin{equation*}
        |\dot x|(t) \triangleq \lim_{s\rightarrow t}\frac{d(x_s, x_t)}{|s - t|}
    \end{equation*}
    provided that the limit exists. 
\end{restatable}

Now, for a pair of probability measures $\mu, \nu \in \mathcal{P}_{2}(\man)$ where $\mu$ admits a density, we will write $T_{\mu \rightarrow \nu}$ for the Brenier (or Brenier-McCann) map from $\mu$ to $\nu$ and we will write $|\dot \mu|$ for the metric derivative of a curve in $\mathcal{P}_{2}(\man)$ with respect to the Wasserstein metric.

\begin{restatable}[\citet{ambrosio2005gradient}]{thm}{}
    Let $(M, g)$ be a smooth and complete Riemannian manifold without boundary and let $(\mu_t)_{t\geq 0}$ be an absolutely continuous curve, i.e. $\mu_t$ admits a Riemannian density and $|\dot \mu_t|$ exists for all $t \geq 0$. Then for every family of vector fields $(v_t)_{t\geq 0}$ for which \cref{eq: balanced continuity eqn} holds, it holds that $|\dot \mu_t|(t) \leq \|v_t\|_{L^2(\mu_t)}$ for all $t \geq 0$. Moreover, there exists a unique family $(v_t)_{t\geq 0}$ such that \cref{eq: balanced continuity eqn} holds and for which $|\dot \mu_t|(t) = \|v_t\|_{L^2(\mu_t)}$ for every $t \geq 0$. This family is such that
    \begin{equation*}
        v_t = \lim_{h \rightarrow0^+}\left(h^{-1}\log_x(T_{\mu_t\rightarrow \mu_{t+h}}(x))\right) \qquad \text{ in $L^2(\mu_t)$}
    \end{equation*}
    and
    \begin{equation*}
        v_t = \min \left\{\|\tilde{v}_t\|_{L^2(\mu_t)}^2 \, \Big|\, \partial_t\mu_t + \operatorname{div}_g(\tilde v_t\mu_t) = 0\right\}
    \end{equation*}
    where $\log$ is the Riemannian logarithmic map. 
\end{restatable}

Note that result coupled with Brenier's theorem indicates that the velocity field that coincides (in an $L^2(\mu_t)$ sense) with the minimal norm solution and the metric derivative of the path is a limit of gradients. This key fact gives rise to the definition of the tangent space and the metric tensor for $(\mathcal{P}_{2}(\man), W_2)$ at a measure $\mu$. 

\begin{restatable}[Wasserstein Tangent Space, \citet{ambrosio2012user}]{defn}{}
    Let $\mu \in \mathcal{P}_2(\man)$. We define the tangent space to $\mathcal{P}_2(\man)$ at $\mu$ to be 
    \begin{equation}
        T_\mu \mathcal{P}_2(\man) = \overline{\left\{\nabla \varphi \, | \, \varphi \in C_c^{\infty}(\man)\right\}}^{L^2(\mu)}
    \end{equation}
    where $\overline{\{\cdot \}}^{L^2(\mu)}$ denotes the $L^2(\mu)$ closure and $C_c^{\infty}$ denotes the set of compactly supported and smooth maps. We also endow $T_\mu \mathcal{P}_2(\man)$ with the metric tensor
    \begin{equation}
        \langle \nabla \varphi_1, \nabla \varphi_2\rangle_{\mu} \triangleq \int_\man \langle\nabla\varphi_1, \nabla\varphi_2 \rangle_g \,d\mu 
    \end{equation}
    where $\langle \cdot, \cdot \rangle_g$ is the inner product defined by the metric tensor $g$.
\end{restatable}

\noindent \textbf{The Covariant Derivative.} 
On an abstract manifold $(\man, g)$ that isn't embedded in an ambient space, we have no obvious way to compare vectors in the tangent space at two points $p, q \in \man$, $p \neq q$. Therefore, we need a way of connecting the separate vector spaces $T_p\man$ and $T_q\man$. One can achieve this by defining a rule $\nabla$ for differentiating vector fields against each other on $\man$ in a way that preserves the structure of the metric $g$. Note that we will adopt the standard notation that a vector field $X$ over $\man$ is an operator on functions -- in particular, for some $f: \man \rightarrow \mathbb{R}$, the notation $X(f)$ denotes the derivative of $f$ in the direction described by $X$. 

\begin{restatable}[The Covariant Derivative, \citet{petersen2006riemannian}]{defn}{}
    For a manifold $(\man, g)$ we define the covariant derivative $\nabla$ to be a rule that assigns to each pair of vector fields $X$, $Y$ over $\man$ another vector field $\nabla_XY$ satisfying 
    \begin{enumerate}
        \item \textit{Linearity:} $\nabla_{aX+bZ}Y = a\nabla_XY + b\nabla_ZY.$
        \item \textit{Leibniz/derivation property:} for $f \in C^{\infty}(\man)$ we have $\nabla_X(fY) = X(f)Y + f\nabla_XY$, where $X(f) = g(X, \nabla f)$ is the derivative of $f$ in the direction $X$.
    \end{enumerate}
\end{restatable}

While this gives us a way to connect tangent spaces, the covariant derivative might distort the geometry induced by the metric $g$ -- i.e. it might not be \say{metric compatible}. A fundamental result of Riemannian geometry, however, is that for any Riemannian manifold $(\man, g)$ there exists a \textit{unique} connection that is torsion free and is \say{metric compatible}. 

\begin{restatable}[The Fundamental Theorem of Riemannian Geometry, \cite{petersen2006riemannian}]{thm}{}
    If $\man$ is a finite dimensional manifold endowed with a Riemannian metric $g$, then there exists a unique connection $\nabla$ called the \textit{Levi-Civita connection} that is
    \begin{enumerate}
        \item \textit{Torsion free:} $\nabla_XY - \nabla_YX = [X,Y]$, where $[X,Y](f) = X(Y(f)) -Y(X(f))$ is the Lie bracket of $X$ and $Y$. 
        \item \textit{Metric compatible:} $Xg(Y, Z) = g(\nabla_XY, Z) + g(Y, \nabla_XZ).$ 
    \end{enumerate}
\end{restatable}
In the case of $(\mathcal{P}_2(\man, W_2)$ we can explicitly construct the covariant derivative, but one needs to manually verify that it is torsion free and metric compatible, as the space is an infinite-dimensional Riemannian Manifold.

\begin{restatable}[Wasserstein Covariant Derivative, \citet{clancy2021interpolating, gigli2012second}]{prop}{} \label{Wasserstein covariant deriv}
    Let $(\mu_t)_{t\geq 0}$ be a curve through $\mathcal{P}_2(\man)$ with tangent field $\nabla \varphi_t$ solving \cref{eq: balanced continuity eqn}, and therefore driving the dynamics of $\mu_t$. Also let $v_t$ be another vector field along $\mu_t$ and let $\Pi_{\mu_t}$ be the orthogonal projection onto $T_{\mu_t}\mathcal{P}_2(\man)$ in $L^2(\mu_t)$. Then the differential operator $\nabla_{(\nabla\varphi_t)}^{W_2}$ given by
    \begin{equation*}
        \nabla_{(\nabla\varphi_t)}^{W_2}v_t = \Pi_{\mu_t}\left(\partial_t v_t + \nabla^{\man}_{( \nabla \varphi_t)} v_t \right)
    \end{equation*}
    is a valid covariant derivative, is torsion free, and is metric compatible, where $\nabla^{\man}_{( \nabla \varphi_t)}$ is the Levi-Civita connection on $\man$. Moreover, when $\man = \mathbb{R}^d$, we have
    \begin{equation*}
        \nabla_{(\nabla\varphi_t)}^{W_2}v_t = \Pi_{\mu_t}\left(\partial_t v_t + \nabla v_t \cdot \nabla \varphi_t\right).
    \end{equation*}
\end{restatable}
This result establishes a closed form PDE describing a differential operator with the desired characteristics of the Levi-Civita connection. We remark that $\Pi_{\mu_t}: L^2(\mu_t; \man) \rightarrow T_{\mu_t}\mathcal{P}_2(\man)$ is the orthogonal projection from the space of $L^2(\mu_t; \mathbb{R}^d)$ vector fields over $\man$ to the closure of $L^2(\mu_t; \mathbb{R}^d)$ \textit{gradient} fields over $\man$. In practice, this operation amounts to a (Helmholtz-Hodge) decomposition $v_t = \nabla \varphi_t + w_t$ where $w_t$ is the (weighted) divergence-free component of the vector field $v_t$. We discuss estimation of this Helmholtz-Hodge decomposition in \Cref{sec: helmholtz decomp}.
\\

\noindent \textbf{The Exponential and Logarithmic Maps.} For $\mathcal{P}_{2}(\man)$, the Wasserstein exponential map is defined to be $\mathbf{exp}_{\mu}(u) \triangleq (\exp(u))_\#\mu$ where $\exp(\cdot)$ is the exponential map of $\man$ and $u \in L^2(\mu; \mathbb{R}^d)$. In the case of $\man = \mathbb{R}^d$, this operation is trivial -- we have
\begin{equation}\label{eq: wasserstein euclidean expmap}
    \mathbf{exp}_\mu(u) = (\text{id} + u)_\# \mu
\end{equation}
where $\text{id}$ is the identity map, $x \mapsto x$. In the case of the logarithmic map, things are analogous. We define the Wasserstein logarithmic map to be $(\mathbf{log}_\mu\nu) (x) \triangleq \log_x(T_{\mu \rightarrow \nu}(x))$ where $T_{\mu \rightarrow \nu}$ is the Brenier map from $\mu$ to $\nu$.  Again, when $\man = \mathbb{R}^d$ this reduces to
\begin{equation}\label{eq: wasserstein euclidean logmap}
    (\mathbf{log}_\mu\nu )(x) = (T_{\mu \rightarrow \nu} - \text{id})(x).
\end{equation}

\section{Parallel Transport}\label{sec:parallel}

Having laid the groundwork of Wasserstein geometry, we will now move on to discussing parallel transport in detail. In the first part of this section we will describe the characterization of parallel transport via the covariant derivative. In the latter parts of this section, we will discuss a fully general approximation scheme for parallel transport on $\mathcal{P}_2(\man)$ along well behaved Wasserstein geodesics that uses Jacobi fields to avoid the PDE obtained from the covariant derivative.   

\subsection{Exact Parallel Transport via the Covariant Derivative.} 
Intuitively, a vector field along a curve is \say{unchanging} if its derivative is zero. In the context of abstract manifolds, parallel transport is defined based on this principle: a vector field along a curve is the parallel transport of a source vector if its covariant derivative along the curve is zero.
\begin{restatable}[\citet{lee2018introduction}]{defn}{} \label{parallel transport on general manifolds}
    Let $\man$ be a smooth Riemannian manifold. A smooth vector field $X$ along a smooth curve $\gamma$ is said to be parallel along $\gamma$ with respect to the Levi-Civita connection if $\nabla_{\dot \gamma}X \equiv 0$.
\end{restatable}
This characterization now allows us to define parallel transport on $(\mathcal{P}_2(\man), W_2)$ using the covariant derivative described in \Cref{Wasserstein covariant deriv}. Consider a smooth curve $\mu_t$ through $\mathcal{P}_2(\man)$ indexed by $t \in (0, 1)$ with the tangent field $\nabla \varphi_t$ driving its dynamics. Then we have the following definition. 

\begin{restatable}[Wasserstein Parallel Transport PDE]{prop}{} \label{Wasserstein parallel transport pde}
    A vector field $v_t$ along $\mu_t$ is parallel along $\mu_t$ with respect to $\nabla_{(\nabla\varphi_t)}^{W_2}$ if
    \[
    \operatorname{div}_g\!\left(\mu_t\left(\partial_t v_t
    + \nabla^M_{(\nabla \varphi_t)} v_t\right)\right)=0
    \qquad\text{for a.e. } t\in(0,1).
    \]
\end{restatable}
\noindent \textit{Proof.} Due to \Cref{parallel transport on general manifolds} and \Cref{Wasserstein covariant deriv} we know that $v_t$ is a parallel vector field along the curve $\mu_t$ if $\Pi_{\mu_t}\left(\partial_tv_t + \nabla^M_{(\nabla \varphi_t)} v_t\right) = 0$ for almost every $t \in (0,1).$ This is tantamount to the requirement that $\operatorname{div}_g \left(\mu_t(\partial_tv_t + \nabla^M_{(\nabla \varphi_t)} v_t)\right) = 0$ since the $L^2(\mu_t)$ projection simply discards the divergence free portion of the vector field. This can be seen by the fact that the orthogonal complement of the tangent space at a measure $\mu_t$ is 
\[T^\perp_{\mu_t}\mathcal{P}_2(\man) = \left\{w \in L^2(\mu_t) 
\,\big | \, \operatorname{div}_g(w\mu_t) = 0\right\}\]
as stated in definition 1.29 of \citet{gigli2012second}. Thus, $\Pi_{\mu_t}(w) = 0 \iff \operatorname{div}_g(\mu_t w) = 0$.
\qed{}

While \Cref{Wasserstein parallel transport pde} describes parallel transport along any curve, solving this PDE in practice may be challenging, especially in high-dimensional settings. To this end, in the next section we will describe an alternative approximate but tractable approach for computing parallel transport \textit{along geodesics} -- our approach leverages the connection between parallel transport on the base space $\man$ and parallel transport on $\mathcal{P}_{2}(\man)$ for any Riemannian manifold $\man$ (as established by \citet{gigli2012second}) to enable parallel transport on $\mathcal{P}_2(\man)$. Since computing parallel transport on $M$ is not always feasible, we also describe a \say{fanning} scheme in which one can use \textit{Jacobi fields} to approximate the underlying parallel transport on $\man$. But before doing that, we will derive and analyze Wasserstein parallel transport on Gaussians to develop some intuition.

\subsubsection{Example: Wasserstein Parallel Transport with Gaussians} \label{sec: WPT with Gaussians}

To build up intuition for the behavior of Wasserstein parallel transport, we will explore the setting where all measures are Gaussian and $\man = \mathbb{R}^d$. Fortunately, as is often the case, our object of interest is available in closed form when dealing with Gaussian measures. This closed form expression for parallel transport with Gaussian measures is detailed in \Cref{Gaussian parallel transport}. We provide the proof of \Cref{Gaussian parallel transport} in \Cref{sec: proof of Gaussian parallel transport}.

\begin{restatable}[Parallel Transport along Gaussian Geodesics]{thm}{} \label{Gaussian parallel transport}
    Let $\mu_0, \mu_1$ be two Gaussian probability measures parameterized by means $m_0, m_1$ and covariances $\Sigma_0, \Sigma_1$. Further let 
    \[v_0(x) = a_0 + A_0(x - m_0) \in T_{\mu_0}\mathcal{P}_2(\mathbb{R}^d), \quad A_0 \in \mathbb{S}^d\]
    be a tangent vector to $\mu_0$. Then the parallel tangent field $v_t$ along $\mu_t$ with initial condition $v_0$ is given by
    \[v_t(x) = a_0 + A_t(x - m_t) \quad \text{with} \quad A_t = A_0 + \int_0^t \dot A_s\, ds\]
    where $m_t = (1-t)m_0 + tm_1$, and $\dot A_t$ solves the continuous Lyapunov equation
    \[\dot A_tQ_t + Q_t \dot A_t = S^{\top}_tA_tQ_t + Q_tA_tS_t\]
    with $Q_t = M_t^{-\top}\Sigma_0^{-1}M_t^{-1}$, $M_t = (1-t)I_d + tB$, $S_t = (I_d - B)M_t^{-1}$ and 
    \[B = \Sigma_0^{-1/2}\left(\Sigma_0^{1/2}\Sigma_1\Sigma_0^{1/2}\right)^{1/2}\Sigma_0^{-1/2}.\]
\end{restatable}

For further intuition, \Cref{fig:Gaussians} instantiates the result presented in \Cref{Gaussian parallel transport}. The figure illustrates how parallel transporting a tangent vector $v \in T_{\nu_i}\mathcal{P}_2(\mathbb{R}^2)$ along the geodesic between $\nu_i$ and $\mu_i^*$ changes its behavior when used to push forward the measures $\nu_i$ and $\mu_i^*$ respectively. In particular, we see the two ways in which Gaussian parallel transport preserves the nature of the deformation of the measures -- the left hand panel depicts how it preserves deformation to the covariance, while the right hand panel depicts how it preserves changes to the mean. 

\begin{figure}[h]
     \centering
     % 1. Set a common height as a variable for easy tweaking
     \def\myheight{10cm} 
     % 2. Use minipages or subfigures without forcing a width larger than the image
     \hspace{-1.25cm}
     \begin{subfigure}[b]{0.45\textwidth}
         \centering
         % Use keepaspectratio to prevent stretching
         \includegraphics[height=\myheight, trim=0.5cm 0cm 0.0cm 0cm, clip]{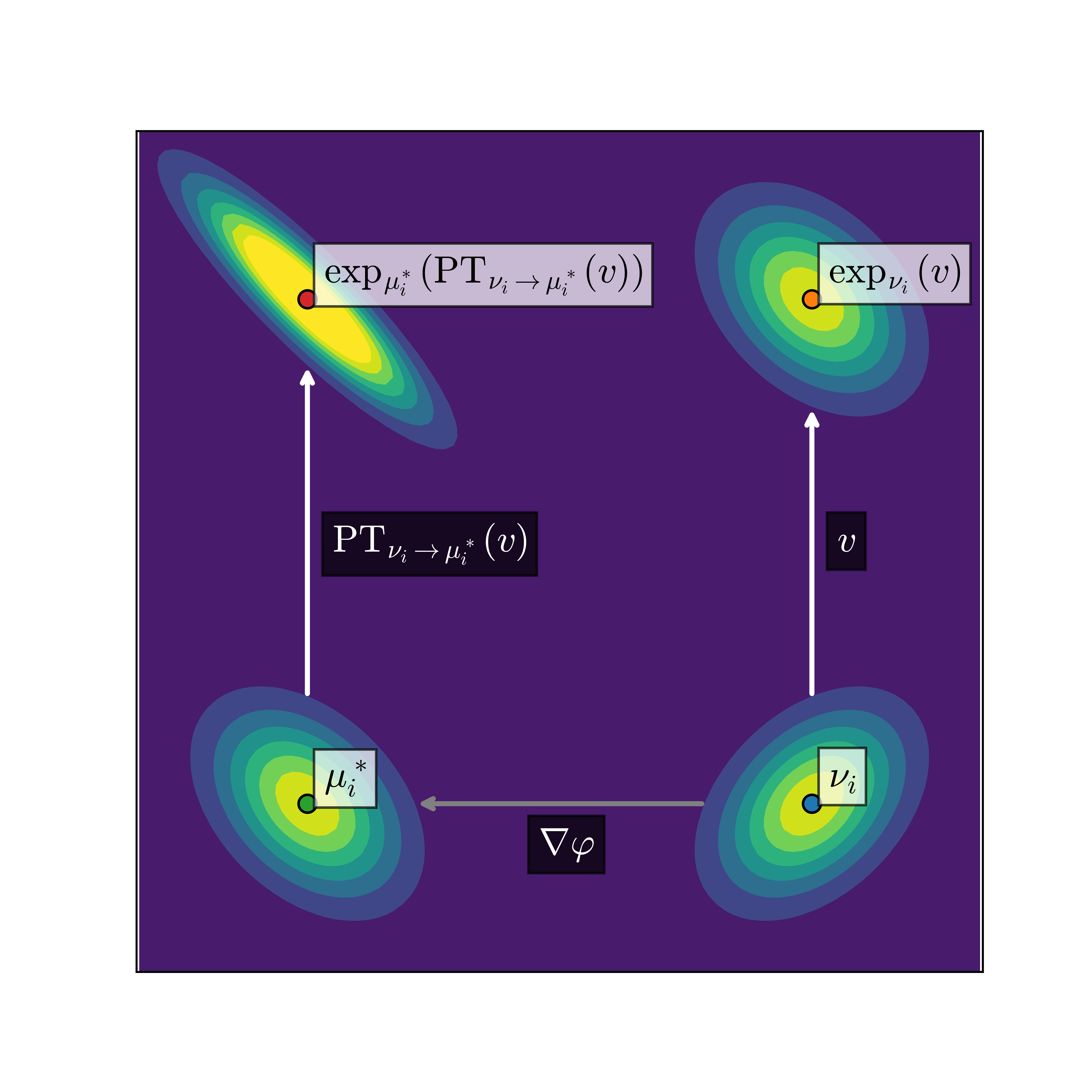}
     \end{subfigure}
     \hspace{1.5cm}
     \begin{subfigure}[b]{0.45\textwidth}
         \centering
         \includegraphics[height=\myheight, trim=0.5cm 0cm 0.5cm 0cm, clip]{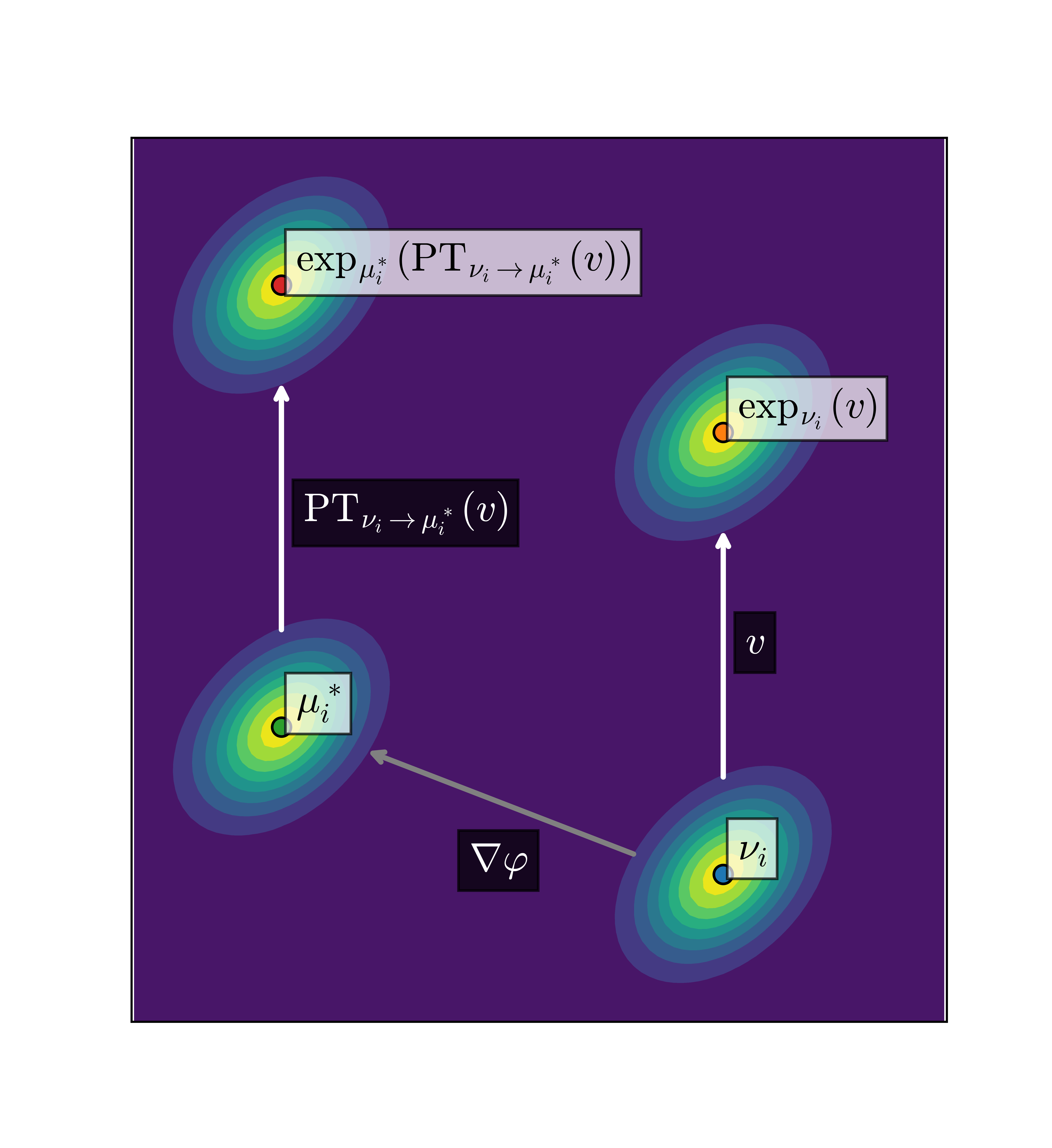}
     \end{subfigure}
     \vspace{-20pt} % Space between sub-captions and main caption
     \caption{Visualizations of Wasserstein Parallel Transport with Gaussian measures in $\mathbb{R}^2$. The left panel illustrates how parallel transport captures deformations to the covariance, while the right panel illustrates how it captures changes in the mean.}
     \label{fig:Gaussians}
\end{figure}

\subsection{Approximation via Base Parallel Transport and Jacobi Fields.} \label{sec: wpt approx}

When dealing with non-Gaussian measures, expressions for Wasserstein parallel transport are not available in closed form. To this end, we will describe a fully general procedure that allows one to approximate parallel transport along geodesics on $\mathcal{P}_2(\man)$ for \textit{any} smooth, complete and connected Riemannian manifold $\man$. To do this we will leverage two key results: the first is the fact that Wasserstein parallel transport of a tangent field $v$ can be \textit{locally} approximated by the parallel transport of the tangent field on the \textit{base} space $\man$ along the Lagrangian paths induced by the optimal transport map associated with the geodesic (\citet{gigli2012second}, Equation 4.18). To give some intuition for this phenomenon, we provide an illustration in \Cref{approximation visualization}. The second key result will be the fact that parallel transport on the \textit{base} space $\man$ can also be locally approximated by a specific Jacobi field (\citet{louis2018fanning}) -- on many Riemannian manifolds, parallel transport is not available in closed form, but we may have access to important objects like the Riemann curvature tensor, the covariant derivative and the geodesic equations. In settings like this, one can use the Jacobi field to approximate parallel transport. Although our experiments in this work consider the case where $\man = \mathbb{R}^d$, in a concurrent work that develops the theory of parallel transport on the space of general positive measures over $\mathbb{R}^d$ we require a non-Euclidean choice of $\man$. 
\\

\begin{figure}
    \centering
    \includegraphics[width=\linewidth]{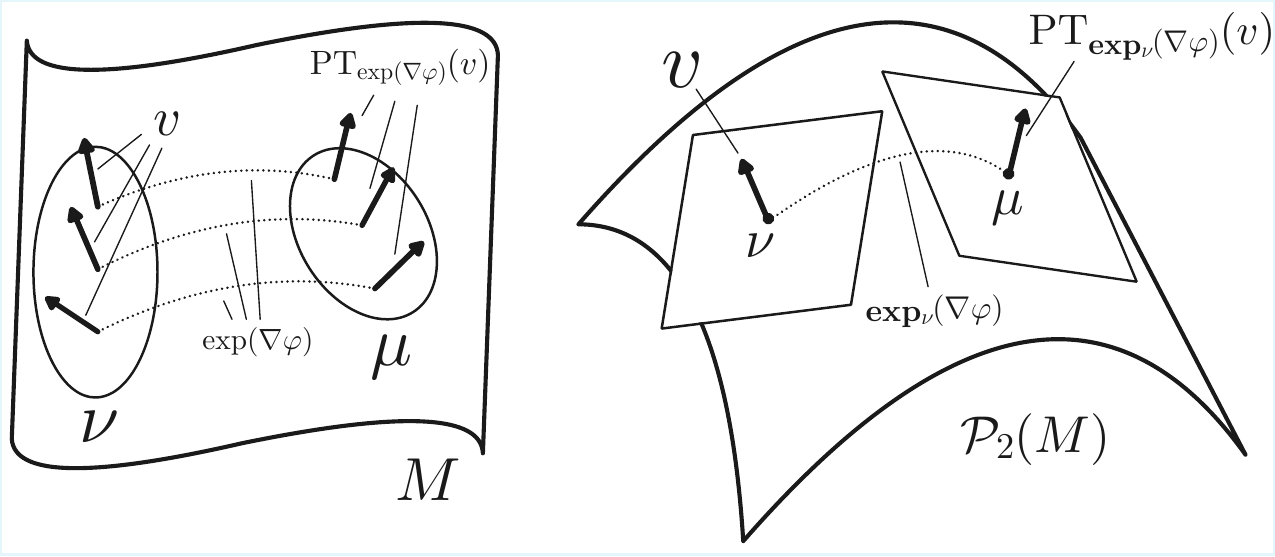}
    \caption{Visualization of \textit{approximate} Wasserstein parallel transport, where the approximate tangent vector is given by the map $p \mapsto {\PT}_{\exp_p(\nabla\varphi(p))}(v(p))$ (left), and \textit{exact} Wasserstein parallel transport, where the true tangent vector is given by the map $p \mapsto({\PT}_{\mathbf{exp}_{\nu}(\nabla\varphi)}(v))(p)$ (right).}
    \label{approximation visualization}
\end{figure}

\noindent \textbf{Jacobi Fields.} Jacobi fields allow us to measure the effect of curvature on a one parameter family of geodesics. A \textit{Jacobi field} is a vector field along a geodesic $\gamma$ capturing the difference between a geodesic and a perturbed geodesic along the parameter of the family. Concretely, let $\gamma_{\tau}$ be a smooth, one-parameter family of geodesics with $\gamma_0 = \gamma$. Then
\begin{equation}
    J(s) = \frac{\partial}{\partial \tau}\gamma_{\tau}(s)\bigg|_{\tau = 0} \in T_{\gamma(s)}\man
\end{equation}
is a Jacobi field, and it describes the infinitesimal behavior of the geodesic family about $\tau = 0$. Jacobi fields satisfy the so-called Jacobi equation
\begin{equation}
    \nabla_{\dot\gamma}\nabla_{\dot\gamma}J(s) + R(J(s), \dot{\gamma}(s))\dot{\gamma}(s) = 0
\end{equation}
where $\nabla$ is the covariant derivative (typically with respect to the Levi-Civita connection) $\dot{\gamma}(s) = d\gamma(s)/ds$, and $R(\cdot, \cdot)$ is the Riemann curvature tensor. In differential geometry, the curvature tensor of a Riemannian manifold $\man$ with a connection $\nabla$ measures the noncommutativity of the covariant derivative, and is defined as
\begin{equation*}
    R(X,Y)Z = \nabla_X \nabla_YZ - \nabla_Y\nabla_XZ - \nabla_{[X,Y]}Z
\end{equation*}
where $X,Y, Z \in \mathfrak{X}(\man)$ are vector fields over the manifold. A key fact about Jacobi fields is the following, which arises from the properties of the second-order ODE above. 

\begin{restatable}[Existence and Uniqueness of Jacobi Fields, Proposition 10.2 of \citet{lee2018introduction}]{prop}{}
Let $\gamma_p(s) = \exp_p(su)$ be a geodesic through a Riemannian or pseudo-Riemannian manifold $\man$ for some $u \in T_p\man$. Then for every pair of vectors $v, w \in T_p\man$ there is a unique Jacobi field denoted $j_{p, u}(v, w )(s)$ along $\gamma_p(s)$ satisfying the initial conditions $j_{p, u}(v, w )(0) = v$ and $\nabla_{\dot \gamma_p}j_{p, u}(v, w)(0) = w$. 
\end{restatable}

Observe that this result indicates that one can also define Jacobi fields along a geodesic by solely specifying its zero and first order initial conditions. Henceforth, we will denote with $J$ a Jacobi field defined by a variation through geodesics, while we will denote with $j$ a Jacobi field defined through initial conditions. With this in mind, we will show that a particular Jacobi field along $\man$ can be used to approximate parallel transport on $\man$.

\begin{restatable}[Fanning Scheme, Proposition B.3 of \citet{louis2018fanning}]{prop}{} \label{fanning approximation}
    Let $\Omega \subseteq \man$ be a compact subset of $\man$ with injectivity radius lower bounded by $\eta$, and let $\gamma_p(s) = \exp_p(su)$ be a geodesic through $(\man, g)$ with $\gamma_p([0, s]) \subset \Omega$. Further denote ${\PT}_{0 \rightarrow s}^\man(v)$ as the parallel transport of $v \in T_{p}\man$ from $p = \gamma_p(0)$ to $\gamma_p(s)$. Then there exists an $A \geq 0$ such that for all $s < \frac{\eta}{\|u\|_g}$ we have   
    \[\left\|s^{-1}j_{p, u}(0, v)(s) - {\PT}_{0 \rightarrow s}^\man(v)\right\|_g \leq As^2\|v\|_g. \]
\end{restatable}

We will leverage this result to approximate parallel transport on $\mathcal{P}_2(\man)$ by using parallel transport on $\man$ and a result from \citet{gigli2012second} that guarantees that parallel transport on $\man$ and $\mathcal{P}_2(\man)$ along \textit{regular curves} are intimately linked. 

\subsection{Approximate Wasserstein Parallel Transport}

To approximate $\mathcal{P}_2(\man)$ parallel transport we will use the fanning scheme described in \Cref{fanning approximation} and the connection between $\man$ parallel transport and $\mathcal{P}_2(\man)$ parallel transport along regular curves established by Gigli. Before we describe our approximation scheme, we need to rigorously establish when Wasserstein parallel transport exists.

\begin{restatable}[Regular Curves, \citet{gigli2012second}]{defn}{} \label{def: regular curves}
    Let $(\mu_t)_{t \in [0,1]}$ be an absolutely continuous curve. We say that $(\mu_t)$ is regular if its velocity vector field $(v_t)$ satisfies 
    \[\int_0^1\|v_t\|_{L^2(\mu_t)}^2\,dt < \infty \quad \text{and} \quad \int_0^1{\operatorname{Lip}}(v_t)\,dt < \infty\]
    where ${\operatorname{Lip}}(v_t)$ denotes the spatial Lipschitz constant of the field $v_t$. Moreover, we say that $(\mu_t)$ is \textit{strongly} regular if 
    \[\int_0^1\|v_t\|_{L^2(\mu_t)}^2\,dt < \infty \quad \text{and} \quad \sup_{t \in [0,1]}{\operatorname{Lip}}(v_t) < \infty.\]
\end{restatable}

Gigli's notion of regularity on curves of measures is a very important one. In particular, Wasserstein parallel transport is only well defined along these regular curves. Along \textit{non-regular} curves of measures, the tangent space does \textit{not} vary smoothly, and therefore parallel transport cannot exist. Thus, for the remainder of this paper, we will consider curves of measures that are regular in the sense described in \Cref{def: regular curves}. Fortunately, when $\man = \mathbb{R}^d$, standard assumptions on the densities of the source and target measures guarantee this regularity. 

\begin{restatable}[Existence of Wasserstein parallel transport, $\man = \mathbb{R}^d$]{prop}{} \label{assumption implications}
    Assume that for the pair of measures ($\mu$, $\nu$), their supports $\Omega, \Omega^* \subset \mathbb{R}^d$ are bounded and mutually uniformly convex domains with respect to each other and have $C^4$ boundaries. Also assume that their Lebesgue densities $f \in C^2(\overline\Omega)$ and $g \in C^2(\overline{\Omega^*})$ satisfy 
    \[\lambda \leq f \leq \Lambda \quad \forall \,x \in \Omega \qquad \text{and} \qquad \lambda \leq g \leq \Lambda \quad \forall \, x \in \Omega^*\]
    for some constants $0 < \lambda \leq \Lambda < \infty$. Then, the quadratic cost optimal OT map pushing $\mu$ to $\nu$, $\nabla \varphi$, is a globally $C^{2}$ diffeomorphism and is bi-Lipschitz. Furthermore, the Wasserstein geodesic corresponding to $\nabla \varphi$ is regular in the sense of \Cref{def: regular curves}, and Wasserstein parallel transport therefore exists.
\end{restatable}

We provide the proof of \Cref{assumption implications} in \Cref{sec: proof of assumption implications}. This result describes sufficient conditions on the measures that guarantees the needed regularity for the existence of Wasserstein parallel transport when $\man = \mathbb{R}^d$. Note that when $\man \neq \mathbb{R}^d$, we simply directly assume the needed regularity of the curves of measures for parallel transport to exist. With this in mind, the following result proposes an approximation scheme for parallel transport along \textit{regular geodesics} between a source $\nu$ and a target $\mu$.

\begin{restatable}[One step approximation of $\mathcal{P}_{2}(\man)$ Parallel Transport along regular curves]{lemma}{}\label{one step parallel transport approx result} 
    Let $(\man, g)$ be a compact Riemannian manifold with injectivity radius bounded from below and suppose $\nu \in  \mathcal{P}_{2}(\man)$ is a measure with $v \in T_{\nu}\mathcal{P}_2(\man)$. Further let $\mu \in \mathcal{P}_{2}(\man)$ be the target measure and assume that there exists a $u$ such that $t \mapsto \nu_t \triangleq \left(F_t\right)_\# \nu$, $t \in [0, 1]$ (with $F_t \triangleq \exp(tu)$) traces out a strongly regular geodesic between $\nu$ and $\mu$ with velocity field $\nabla \psi_t$. Then for $s$ sufficiently small and for $w_s$ defined as
    \[w_s = \Pi_{\nu_s}\left(s^{-1}j_{\nu, u}\left(0, v\right)(s) \right) \quad \text{with}\quad j_{\nu, u}\left(0, v\right)(s) \triangleq \left(x \mapsto j_{x, u(x)}\left(0, v(x)\right)(s)\right) \circ F_s^{-1}\]
    we have
    \[\|w_s - {\PT}_{\nu \rightarrow \nu_s}(v)\|_{L^2(\nu_s)} \leq As^2\|v\|_{L^2(\nu)} + \left(e^{\int_0^1{\operatorname{Lip}}(\nabla \psi_r)dr} - 1\right)^2\|v\|_{L^2(\nu)}\left(\int_0^s{\operatorname{Lip}}(\nabla\psi_r)dr\right)^2.\]
    Moreover, the strong regularity of $t \mapsto \nu_t$ guarantees that $w_s$ approximates ${\PT}_{\nu \rightarrow \nu_s}(v)$, the parallel transport of $v$ along the geodesic $\mathbf{exp}_{\nu}(tu)$, in the sense that
    \[\|w_s - {\PT}_{\nu \rightarrow \nu_s}(v)\|_{L^2(\nu_s)} \leq Cs^2 \quad \text{for some $C>0$ independent of $s$.}\]
\end{restatable}

We provide the proof of \Cref{one step parallel transport approx result} in \Cref{sec: proof of one step parallel transport approx}. This result allows us to approximate parallel transport of a tangent vector $v$ on $(\mathcal{P}_{2}(\man), W_2)$ along \textit{strongly regular geodesics}, provided that we are able to solve for the Jacobi field $j_{\nu, u}$ on the base space $\man$ -- in particular, this scheme approximates Wasserstein parallel transport along an entire strongly regular geodesic by dividing it into a large number of small segments, and locally parallel transporting the tangent field along the paths defined by the optimal transport map. Note that if one takes the base space $\man = \Omega \subset \subset \mathbb{R}^d$, the Jacobi field approximation can be replaced with the exact Euclidean parallel transport, which is simply the identity map composed with the pullback. In other cases of constant curvature, Jacobi field solutions are available in closed form. Equipped with this general result, we will now describe a procedure to approximate geodesic Wasserstein parallel transport that achieves $O(N^{-1})$ accuracy (where $N$ is a user-chosen approximation resolution) detailed in \Cref{full approximation of parallel transport along regular curves}. We provide the proof in \Cref{sec: proof of full approximation of parallel transport along regular curves}.

\begin{restatable}[Approximation of $\mathcal{P}_{2}(\man)$ Parallel Transport along strongly regular curves]{thm}{} \label{full approximation of parallel transport along regular curves}
    Consider the setting of \Cref{one step parallel transport approx result}, where $\nu$ and $\mu$ are connected by a strongly regular geodesic and define the approximate parallel transport map $\widehat{{\PT}}_{s_1 \rightarrow s_2}$ to be the one described by the procedure of \Cref{one step parallel transport approx result}. Fix some $n\in \mathbb{N}$, let $s = 1/N$, define the maps 
    \[\widehat{{\PT}}_k \triangleq \widehat{{\PT}}_{(k-1)s \rightarrow ks} \quad \text{and} \qquad {\PT}_k \triangleq {\PT}_{(k-1)s \rightarrow ks}\]
    and let $v_k = ({\PT}_k \circ \dots \circ {\PT}_1)(v)$ and $\hat v_k = (\widehat{{\PT}}_k \circ \dots \circ \widehat{{\PT}}_1)(v)$. Then we have
    \[\|\hat v_N - v_N\|_{L^2(\mu)} = O(N^{-1}).\]
\end{restatable}
\begin{algorithm}[h]
\caption{Approximate Wasserstein Parallel Transport (WPT)}
\label{alg: W parallel transport}
\begin{algorithmic}[1]
    \Require Abs. continuous $\nu,\mu\in\mathcal P_2(\man)$, tangent vector $v\in T_\nu\mathcal P_2(\man)$, discretization level $N$.
    \State Compute the Brenier map $T_{\nu\to\mu}$.
    \State Set $u(x) = \log_x(T_{\nu\to\mu}(x))$ and $s = 1/N$. \Comment{$u$ generates the geodesic}
    \For{$i=0,\dots,N$}
        \State $F_i(x)= \exp_x\bigl(isu(x)\bigr)$. 
    \EndFor
    \State $v_0 = v$. \Comment{initialize transported field}
    \For{$i=1$ to $N$}
        \State $\nu_i= (F_{i})_\#\nu$, $S_i= F_i\circ F_{i-1}^{-1}$, $u_i(x)= \log_x(S_i(x))$. \Comment{one-step transport}
        \State $w_i(x)= s^{-1}j_{(x,u_i(x))}(0,v_{i-1}(x))(s)$. \Comment{Jacobi propagation on $\man$}
        \State $v_{i}= \Pi_{\nu_{i}}(w_i\circ S_i^{-1})$. \Comment{project to $T_{\nu_{i}}\mathcal P_2(\man)$}
    \EndFor
    \State \Return $v_N\in T_\mu\mathcal P_2(\man)$.
\end{algorithmic}
\end{algorithm}

\subsection{Weighted Helmholtz-Hodge Decomposition} \label{sec: helmholtz decomp}

The $L^2(\mu)$ projection of velocity fields required in \Cref{alg: W parallel transport} is a non-trivial computational procedure, and it warrants further discussion. In this section we will propose a procedure for estimating this vector field projection using \textit{Reproducing Kernel Hilbert Space (RKHS)} regression. In this section we will assume that the measure $\mu$ has a Lebesgue density $\rho$. The projection $\Pi_{\mu}: L^2(\mu; \mathbb{R}^d) \rightarrow T_\mu\mathcal{P}_2(\mathbb{R}^d)$ maps $L^2(\mu; \mathbb{R}^d)$ vector fields to the $L^2(\mu)$ closure of gradient fields through the following optimization procedure,:
\[\Pi_{\mu}(v) = \arg\min_{u \in \mathcal{G}_\mu}\|u - v\|_{L^2(\mu)}^2\]
where $\mathcal{G}_\mu \triangleq \overline{\{\nabla\phi\,:\, \phi \in C^{\infty}_c(\mathbb{R}^d) \}}^{L^2(\mu)}$. To estimate this projection from samples, we will use Reproducing Kernel Hilbert Space (RKHS) regression. 

A Reproducing Kernel Hilbert Space (RKHS) is a Hilbert space $\mathcal{H}$ of functions $f: \mathcal{X} \rightarrow \mathbb{R}$ with a so-called \textit{reproducing kernel} $K: \mathcal{X}\times \mathcal{X} \rightarrow \mathbb{R}$ characterized by the property that $K(x, \cdot ) \in \mathcal{H}$ and $f(x) = \langle K(x, \cdot), f \rangle_{\mathcal{H}}$ \citep{aronszajn1950theory}. A key reason for the widespread adoption of RKHSs in statistical applications is the representer theorem, which characterizes solutions of regression problems when optimizing over the function class $\mathcal{H}$ subject to a norm penalty.

\begin{restatable}[Representer Theorem]{thm}{}
    For a dataset $\mathcal{X} = \{x_i\}_{i = 1}^n$, consider an RKHS $\mathcal{H}$ of functions $f: \mathcal{X} \rightarrow \mathbb{R}$ with a kernel $K$. For the optimization problem
    \[f^* \in \arg\min_{f\in \mathcal{H}}\sum_{i = 1}^n\left|f(x_i) - y_i\right|^2 + \lambda \|f\|_\mathcal{H}^2\]
    the solution can be expressed as 
    \[f^* = \sum_{i = 1}^n\alpha_i K(x_i, \cdot )\]
    for a set of scalars $\alpha_1, \dots, \alpha_n$. 
\end{restatable}
One can find a proof by \citet{ghojogh2021reproducing}. The result stems from the fact that the optimization problem is a norm-penalized $L^2(\mathbb{P}_n)$ projection onto the RKHS, and solutions therefore lie in the span of kernel sections $K(x_i, \cdot)$. The critical implication is that for a fixed dataset $\{x_1, \dots, x_n\}$ one only needs to solve for $\alpha = (\alpha_1, \dots, \alpha_n)^{\top} \in \mathbb{R}^n$ via
\[\alpha^* \in \arg\min_{\alpha \in \mathbb{R}^n}\|Y - K\alpha\|_2^2 + \lambda\alpha^{\top}K\alpha\]
which has the solution $\alpha = (K + \lambda I_n)^{-1}Y$ where $K_{ij} = K(x_i, x_j)$ when $K$ is invertible. To adapt this to our setting, where we instead want to find the best \textit{gradient-field} solution, we will make the following modification. Let $\mathcal{H}$ be a scalar RKHS on $\mathbb{R}^d$ with a twice-differentiable \textit{Mercer} kernel (see \citet{zhou2008derivative}) $K \in C^2(\mathbb{R}^d \times \mathbb{R}^d)$ and consider the function class $\{\nabla f : f\in \mathcal{H}\}$. For a population measure $\mu$ and target field $v: \mathbb{R}^d \rightarrow \mathbb{R}^d$ define the population-level solution 
\begin{equation}
    f^*_{\lambda} \in \arg\min_{f \in \mathcal{H}}\left\{\left\|\nabla f(x) - v(x)\right\|_{L^2(\mu)}^2 + \lambda\|f\|_{\mathcal{H}}^2\right\} \label{eq: population gradient RKHS problem}
\end{equation}
and the sample analogue
\begin{equation}
    \hat{f}_{\lambda} \in \arg\min_{f \in \mathcal{H}}\left\{\frac{1}{n}\sum_{i = 1}^n\left\|\nabla f(x_i) - v(x_i)\right\|_{2}^2 + \lambda\|f\|_{\mathcal{H}}^2\right\}. \label{eq: sample gradient RKHS problem}
\end{equation}
We then have the following result, which we call the \textit{gradient} representer theorem. We provide the proof of \Cref{gradient representer thm} in \Cref{sec: proof of gradient representer thm}.

\begin{restatable}[Gradient Representer Theorem]{thm}{} \label{gradient representer thm}
    The optimization problem in \Cref{eq: sample gradient RKHS problem} has a solution of the form $\hat f_\lambda (\cdot) = \sum_{i = 1}^n\langle c_i, \nabla K(x_i, \cdot) \rangle$ for some set of vectors $c_{i} \in \mathbb{R}^d$ and the inner product is the standard Euclidean inner product on $\mathbb{R}^d$. 
\end{restatable}

As with classical RKHS regression, we can obtain a closed-form solution for the empirical optimizers $c_1^*, \dots, c_n^*$ by solving a linear system. Let $D \in \mathbb{R}^{nd \times nd}$ be the block matrix where $D_{ij} = \nabla_2^2K(x_i, x_j) \in \mathbb{R}^{d\times d}$, let $G \in \mathbb{R}^{nd \times nd}$ be a block Gram matrix where $G_{ij} = \nabla_1 \nabla_2 K(x_i, x_j) \in \mathbb{R}^{d\times d}$ and define the stacked vectors 
\[c \triangleq (c_1, \dots, c_n)^{\top} \in \mathbb{R}^{nd} \quad \text{and} \quad v \triangleq (v(x_1), \dots, v(x_n))^{\top}.\]
It follows that the RKHS norm of $\hat f$ can be written as $\|\hat f\|_\mathcal{H}^2 = c^{\top}Gc$, and our empirical objective function is 
\[J(c) = \frac{1}{n}\|Dc - v\|_2^2 + \lambda c^{\top}Gc.\]
If $D^\top D + n\lambda G$ is invertible, then the objective function admits a unique minimizer $c^* = (D^\top D + n\lambda G)^{-1}D^{\top}v$, and our estimated projected field is given by 
\[\nabla \hat f_{n, \lambda} = \begin{bmatrix}
    \nabla \hat f(x_1) \\
    \vdots \\
    \nabla \hat f(x_n)
\end{bmatrix} = Dc.\]
We provide the full derivation of the empirical objective function $J$ and the solution $\nabla \hat f_{n, \lambda}$ in \Cref{sec: RKHS empirical objective derivation}. We will now show that this procedure is consistent for estimating the Helmholtz-Hodge projection.

\begin{restatable}{thm}{} \label{RKHS consistency}
Let $\Omega\subset \mathbb R^d$ be compact, let $\mu$ be a probability measure supported on $\Omega$, and let $v\in L^2(\mu;\mathbb R^d)$ such that $\|v(x)\|_2 \leq M < \infty$ for $\mu$-a.e. $x$. Let $\mathcal H$ be a scalar RKHS on $\mathbb{R}^d$ with kernel $K$, and assume: 
\begin{enumerate} 
    \item For each $j\in\{1,\dots,d\}$ and $x\in\mathbb{R}^d$, the derivative representer $\psi_j$ exists, \[ \psi_j(x)\triangleq \partial_{j}K(x, \cdot)\in \mathcal H \quad \text{and} \quad \kappa^2 \triangleq \sup_{x\in\mathbb{R}^d}\sum_{j=1}^d \|\psi_j(x)\|_{\mathcal H}^2 < \infty. \] 
    \item The RKHS gradient class is dense in the gradient subspace, i.e. $\overline{\{\nabla f:f\in \mathcal H\}}^{\,L^2(\mu)}=\mathcal G_\mu.$
\end{enumerate} 
Define \[ \hat f_{n,\lambda}\in \arg\min_{f\in\mathcal H} \left\{ \frac1n\sum_{i=1}^n \|\nabla f(x_i)-v(x_i)\|_2^2+\lambda\|f\|_{\mathcal H}^2 \right\}\]
where $x_1, \dots, x_n \sim \mu$. Then if $\lambda = O(n^{-\ell})$ for any $\ell \in  (0, 1/2)$, we have
\[\|\nabla \hat f_{n, \lambda} - \Pi_{\mu}(v)\|_{L^2(\mu)} \overset{p}\longrightarrow 0.\]
\end{restatable}
\noindent \textit{Proof.} Let $f^*_\lambda$ be defined as in \Cref{eq: population gradient RKHS problem}. To prove this result we will leverage the results of \citet{de2005learning}, where the consistency that we desire is proven for bounded linear operators $A: \mathcal{H} \rightarrow \mathcal{K}$ where $\mathcal{K}$ is another Hilbert space. Under the hypotheses stated, \Cref{bounded linear operator} indicates that the operator $A = \nabla$ is indeed a bounded linear operator.
Theorem 5 of \citet{de2005learning} gives $\|\nabla \hat f_\lambda - Pv\|_{L^2(\mu)} \rightarrow 0$ where $P$ is the orthogonal projection onto $\overline{\text{Ran}(A)}$, which in our case satisfies \[\overline{\text{Ran}(A)} = \overline{\{\nabla f: f \in \mathcal{H}\}}^{L^2(\mu)}.\] 
Since the space $\{\nabla f : f \in \mathcal{H}\}$ is assumed to be dense in $\overline{\{\nabla f: f \in C^\infty_c(\mathbb{R}^d)\}}^{L^2(\mu)}$, we know that $Pv = \Pi_{\mu}(v)$. \qed{}

\Cref{RKHS consistency} guarantees consistency of the RKHS procedure under the assumption that the function class is dense in the space of gradients of functions of interest. We will now show that this property is satisfied when one chooses a kernel inducing an RKHS that is norm-equivalent to a Sobolev space of finite smoothness. Recall that a Sobolev space $W^{k, p}(\mathbb{R})$ for $1 \le p \leq \infty$ is the subset of functions $f$ in $L^p(\mathbb{R})$ such that $f$ and its weak derivatives up to order $k$ have finite $L^p$ norm. When $p = 2$, the function space forms a Hilbert space and we denote $H^k = W^{k, 2}.$ 

\begin{restatable}[Sufficient conditions for density]{prop}{} \label{suff cond for density}
    Let $\mu$ be a measure supported on a compact domain $\Omega \subset \mathbb{R}^d$. Suppose that $\mu$ admits a Lebesgue density $\rho$ satisfying $\rho \leq M < \infty$, and let $K$ be a kernel such that $\mathcal{H}$, the RKHS induced by $K$, is set and norm equivalent (denoted $\cong$) to the Sobolev space $H^\tau(\mathbb{R}^d)$ for some $\tau \in \mathbb{N}$. Then
    \[\overline{\left\{\nabla f : f \in \mathcal{H}\right\}}^{L^2(\mu)} = \overline{\{\nabla f: f \in C^\infty_c(\mathbb{R}^d)\}}^{L^2(\mu)}.\]
    and thus $\left\{\nabla f : f \in \mathcal{H}\right\}$ is dense in $\overline{\{\nabla f: f \in C^\infty_c(\mathbb{R}^d)\}}^{L^2(\mu)}.$
\end{restatable}

We provide the proof of \Cref{suff cond for density} in \Cref{sec: proof of suff cond for density}. The result indicates that any choice of kernel yielding an RKHS that is equivalent (as sets and in norm) to the Sobolev space $H^{\tau}(\mathbb{R}^d)$ for some $\tau \geq 1$ yields the density condition needed for consistency (\Cref{RKHS consistency}). An example of such a kernel is the M\'atern kernel \citep{matern1960spatial},
\begin{equation}
    K_{a, p}(x,y) = \frac{2^{1-p}}{\Gamma(p)}\left(\frac{\|x - y\|_2}{a}\right)K_p\left(\frac{\|x - y\|_2}{a}\right)
\end{equation}
for some $a, p > 0$ where $K_p(\cdot)$ is the \textit{modified Bessel function of the second kind} \citep{emery2025towards}. When one chooses this kernel with $p > 0$, the kernel's spectral density $\widehat K(\omega)$ satisfies $ \widehat K(\omega)\asymp (1 + \|\omega\|_2^2)^{-p-d/2}$, which makes it norm equivalent to $H^{p+d/2}(\mathbb{R}^d)$ \citep{emery2025towards}.

\begin{figure}
    \centering \includegraphics[width=\linewidth]{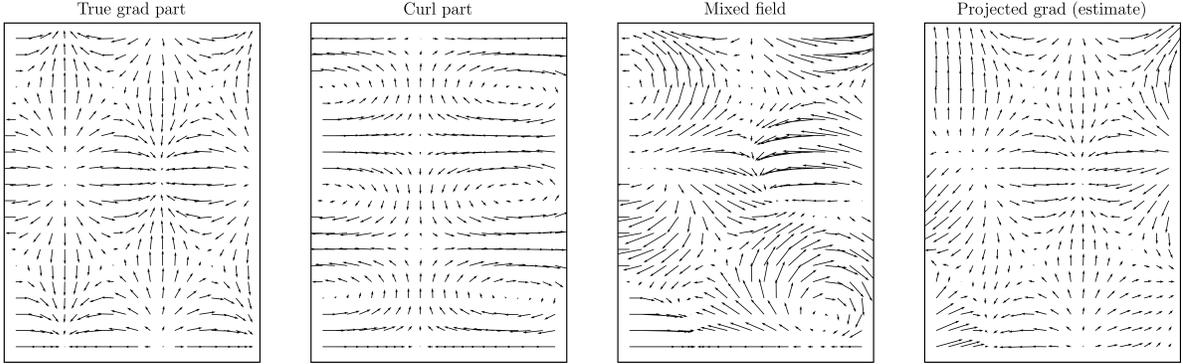}
    \caption{Estimated projection of a non-conservative vector field onto the space of gradient fields using the procedure described in \Cref{sec: helmholtz decomp}}
    \label{fig: helmholtz}
\end{figure}

\section{Reconstructing Counterfactual Dynamics}\label{sec:counterfactual_dyn}

Equipped with an algorithm for approximating Wasserstein parallel transport, we will now leverage it to construct a procedure for predicting dynamics on the space of measures under a parallel trends assumption. In this section we will explicitly state the algorithm, and we will prove an error bound that quantifies the distance between the predicted counterfactual dynamics and the true dynamics; we note that this error bound represents population level \textit{approximation} error, and does not include statistical error. We defer a rigorous treatment of the statistical aspects of estimating parallel transport to future work. After describing the algorithm, we will illustrate the conceptual utility of this procedure by instantiating it in the Difference-in-Differences framework in Causal Inference, and showing that it recovers the well known \textit{parallel trends} assumption.

\subsection{Wasserstein Counterfactual Dynamics Prediction}

\begin{figure}[t]
  \centering
  \includegraphics[width=\linewidth]{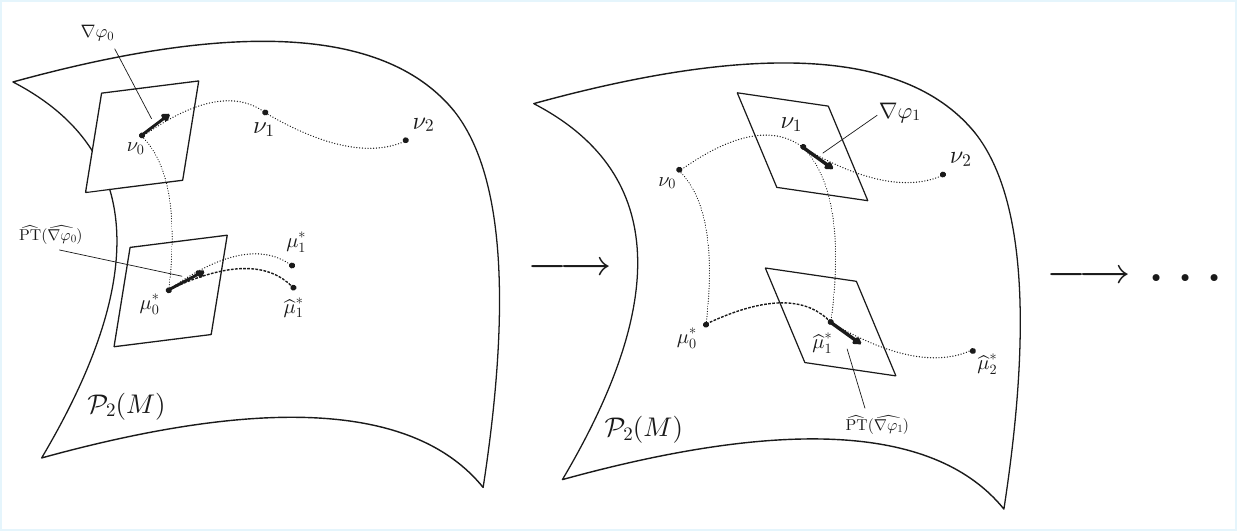}
  \caption{Visualization of counterfactual dynamics prediction procedure described in \Cref{alg: W trajectory reconstruction}. }
  \label{fig:counterfactual-prediction}
\end{figure}

Our procedure for predicting counterfactual dynamics will follow that depicted in \Cref{fig:counterfactual-prediction}. In particular, given a control curve $\{\nu_i\}_{i \in \{1, \dots, T\}} \subset \mathcal{P}_2(\man)$ and an initial condition $\mu_1^*$, we will predict $\{\hat \mu_i^*\}_{i \in [2, \dots, T]}$ as follows: at a given iteration $i$ we will first estimate the tangent velocity driving the dynamics of $\nu_i$ at the current time. We will then approximate the parallel transport of this tangent velocity on the geodesic between the current $\nu_i$ and the most recent prediction $\hat\mu_i^*$ -- note that when $i = 1$, $\hat \mu_1^*$ is set to the initial condition $\mu_1$. Finally, we will push $\hat \mu_i^*$ through the Wasserstein exponential map with the parallel transported velocity. 

While our theoretical results in previous sections allowed for non-Euclidean choices of $\man$, in this section our error bounds consider the setting $\man = \mathbb{T}^d$, where $\mathbb{T}^d$ is the \textit{flat torus} defined by $\mathbb{T}^d = \mathbb{R}^d/\mathbb{Z}^d$. We restrict our analysis to this setting in analyzing our counterfactual dynamics prediction algorithm as the analysis requires stability bounds on Wasserstein parallel transport that are challenging to obtain in non-Euclidean settings, and settings where the domain has a boundary. We also note that our intention with the theoretical results presented in this section are to rigorously quantify how our approximation scheme affects predicted dynamics away from the boundary of the support; this goal makes $\man = \mathbb{T}^d$ a natural choice, as it is a boundaryless, compact and flat space.

\begin{restatable}[Uniform admissibility of the trajectory class, $\man = \mathbb{T}^d$]{asmpt}{} \label{all necessary W assumptions}
There exists a class $\mathcal C \subset \mathcal P_2(\mathbb{T}^d)$ such that the measures appearing in the procedure satisfy $\nu_i, \mu_i^*, \hat\mu_i^* \in \mathcal C$ for all $i\in\{1,\dots,T\}.$ Moreover, we require the following:

\begin{enumerate}[label=\textbf{(A.\arabic*)}, leftmargin=*, align=left]

\item \textbf{Density regularity conditions:} 
All measures $\mu \in \mathcal{C}$, and all pairwise Wasserstein interpolants $\mu_t$ thereof, admit Lebesgue densities $\rho$, $\rho_t$ such that: \label{density regularity conditions}
\begin{enumerate}[label=(\alph*)]
    \item there exist constants $0 < \lambda \leq \Lambda < \infty$ such that $\lambda \leq \rho\leq \Lambda$ and $\lambda \leq \rho_t \leq \Lambda$.
    \item the interpolating densities are uniformly bounded in a Sobolev sense, i.e. 
    \[\sup_{t \in [0,1]}\|\rho_t\|_{W^{1, \infty}(\mathbb{T}^d)} \leq K.\] 
    \item the map $t \mapsto \rho_t$ is absolutely continuous as a map from $[0,1]$ into $L^\infty(\mathbb{T}^d)$ and the interpolating densities have a uniformly bounded time derivative, i.e.
    \[\sup_{t \in [0,1]}\|\partial_t\rho_t\|_{L^{\infty}(\mathbb{T}^d)} \leq K_\rho\]
    \item the interpolating densities have a uniformly bounded score function in an $L^\infty$ sense, i.e.
    \[\sup_{t \in [0,1]}\|\nabla \log \rho_t\|_{L^\infty(\mathbb{T}^d)} \leq K_{\log}.\]
\end{enumerate}

\item \textbf{Wasserstein geodesic Lipschitz conditions:}  All triplets $(\nu, \mu, \mu') \in \mathcal{C}$ satisfy the following: \label{geodesic lipschitz conditions}
\begin{enumerate}[label=(\alph*)]
    \item the densities $\rho_t, \rho_t'$ of the Wasserstein interpolants of $(\nu, \mu)$ and $(\nu, \mu')$ satisfy
    \[\| \rho_t - \rho_t'\|_{L^\infty(\mathbb{T}^d)} \leq C_{L}\cdot W_2(\mu, \mu') \quad \text{and} \quad \|\rho_t-\rho_t'\|_{W^{1,\infty}(\mathbb{T}^d)}
    \le C_{W}\,W_2(\mu,\mu')\]
    for some $C_{L}, C_W > 0$. 
    \item the tangent velocities $\nabla \phi_t, \nabla \phi_t'$ of the Wasserstein interpolants of $(\nu, \mu)$ and $(\nu, \mu')$ satisfy
    \[\|\nabla \phi_t - \nabla \phi_t'\|_{L^\infty(\mathbb{T}^d)} \leq C_{\text{vel}}\cdot W_2(\mu, \mu').\]
\end{enumerate}

\item \textbf{Wasserstein geodesic regularity conditions:} 
All pairs $(\nu, \mu) \in \mathcal{C}$ admit geodesic interpolants $\mu_t$ with velocity fields $\nabla \phi_t$ that are uniformly bounded in a Sobolev sense
\[ \sup_{t\in[0, 1]}\|\nabla \phi_t\|_{W^{1, \infty}(\mathbb{T}^d)} \leq M\]
for some $M > 0.$ \label{geodesic regularity conditions}
\end{enumerate}
\end{restatable}

In this section, we will operate under the conditions detailed in \Cref{all necessary W assumptions}. These assumptions guarantee the existence of strongly regular geodesics (in the sense of \Cref{def: regular curves}) between all pairs of measures of interest. Moreover, the assumptions guarantee the stability of Wasserstein parallel transport needed to obtain error bounds for our dynamics prediction algorithm, \Cref{alg: W trajectory reconstruction}.

\begin{restatable}[\Cref{wpt stability}, Stability of Wasserstein Parallel Transport]{thm}{} \label{wpt stability statement}
    Suppose all statements in \Cref{all necessary W assumptions} hold. Then there exists a constant $C_{\rm WPT}>0$, depending only on the constants above, such that for every $v\in T_\nu\mathcal P_2(\mathbb{T}^d)\cap H^1(\nu;\mathbb{R}^d)$,
    \[
    \|\PT_{\nu\to\mu}(v)-\PT_{\nu\to\mu'}(v)\|_{L^2(\nu)}
    \le
    C_{\rm WPT}\,\|v\|_{H^1(\nu)}\,W_2(\mu,\mu').
    \]
\end{restatable}
We provide the proof of \Cref{wpt stability statement} in \Cref{sec: stability theory}. This stability result allows us to control the error of the iterative dynamics prediction algorithm described in \Cref{alg: W trajectory reconstruction} by ensuring that small approximation errors in the counterfactual do not cause unbounded growth in the terminal error of the parallel transported tangent at the approximate counterfactual. Concretely, we use this stability to establish the following one-step error bound described in \Cref{one step error of W traj reconstruction}.

\begin{algorithm}[h]
\caption{Counterfactual dynamics prediction via WPT}
\label{alg: W trajectory reconstruction}
\begin{algorithmic}[1]
    \Require Sufficiently regular control trajectory $(\nu_i)_{i=1}^T$, treated initial condition $\mu_1\in\mathcal P_2(\man)$, discretization level $N$.
    \State Initialize $\hat\mu_1^* \gets \mu_1$. \Comment{starting counterfactual state}
    \For{$i=1$ to $T-1$}
        \State Compute $T_{\nu_i\to\nu_{i+1}}$.
        \State Set $v^{(i)}(x)\triangleq \log_x(T_{\nu_i\to\nu_{i+1}}(x))$. \Comment{control velocity}
        \State $\hat w^{(i)} = \text{WPT}(\nu_i, \hat \mu_i^*, v^{(i)}, N)$. \Comment{\Cref{alg: W parallel transport}}
        \State $\hat\mu_{i+1}^* =  \mathbf{exp}_{\hat\mu_i^*}(\hat w^{(i)})$. \Comment{advance counterfactual state}
    \EndFor
    \State \Return $(\hat\mu_i^*)_{i=1}^T$.
\end{algorithmic}
\end{algorithm}

\begin{restatable}[One step error of \Cref{alg: W trajectory reconstruction}, $\man = \mathbb{T}^d$]{lemma}{} \label{one step error of W traj reconstruction}
    Let $\nu_i,\nu_{i+1},\mu_i^*,\hat\mu_i^* \in \mathcal C$, where $\mathcal C$ is the admissible
    class from \Cref{all necessary W assumptions}. Suppose there exists a tangent vector
    $\nabla \varphi_i \in T_{\nu_i}\mathcal P_2(\mathbb{T}^d)$ such that
    \[
    \mathbf{exp}_{\nu_i}(\nabla \varphi_i)=\nu_{i+1}.
    \]
    Let $\nabla \varphi_i^*$ denote the parallel transport of $\nabla \varphi_i$ along the strongly regular Wasserstein
    geodesic connecting $\nu_i$ and $\mu_i^*$, and define
    \[
    \mu_{i+1}^* \triangleq \mathbf{exp}_{\mu_i^*}(\nabla \varphi_i^*).
    \]
    Let $\hat v$ be the output of \Cref{alg: W trajectory reconstruction} used to predict the next counterfactual iterate
    \[
    \hat\mu_{i+1}^* \triangleq \mathbf{exp}_{\hat\mu_i^*}(\hat v).
    \]
    Then
    \[
    W_2(\hat\mu_{i+1}^*, \mu_{i+1}^*) \leq \left(1 + \operatorname{Lip}(\nabla \varphi_i^*) + \|\nabla \varphi_i\|_{H^1(\mathbb{T}^d)}C_{\text{WPT}} \right) W_2(\hat \mu_i^*, \mu_i^*) + O(N^{-1}).
    \]
\end{restatable}

We provide the proof of \Cref{one step error of W traj reconstruction} in \Cref{proof of one step error of W traj reconstruction}. The result establishes that each iteration of the procedure incurs an $O(N^{-1})$ error of the counterfactual prediction in a Wasserstein sense. Fortunately, since $N$ corresponds to the user chosen approximation resolution, the one step error can be made arbitrarily small. We also have the following corollary, which quantifies the cumulative error over the entire procedure. 

\begin{restatable}{thm}{} \label{full cf approximation error}
    In the regime described in \Cref{all necessary W assumptions} we have the following accumulated error bound on the counterfactual trajectory,
    \[\sum_{i = 1}^TW_2(\hat{\mu}_i, \mu_i^*) = O\left(\frac{T}{N}\cdot \frac{( 1 + \operatorname{L} + SC_{\text{WPT}}) ^T - 1}{\operatorname{L} + SC_{\text{WPT}} }\right) \]
    under the assumption that $\|\nabla \varphi_i\|_{H^1(\mathbb{T}^d)} \leq S $ and $\operatorname{Lip}(\nabla \varphi_i^*) \leq L$ for all $i \in \{1, \dots, T-1\}$. In the setting where $T = \Theta(1)$ and $N \rightarrow \infty$, the accumulated error is $o(1)$. 
\end{restatable}
\noindent \textit{Proof.} The result of \Cref{one step error of W traj reconstruction} implies a recurrence relation of the form $a_{t+1} = ca_t + r$ and $a_0 = 0$, the solution to which has the form 
\[a_t = \begin{cases}
    r\cdot \frac{c^t - 1}{c - 1} & \text{when $c \neq 1$} \\
    tr &  \text{when $c = 1$}
\end{cases}.\]
Plugging in $c = 1 + \operatorname{L} + SC_{\text{WPT}}$ and $r = O(N^{-1})$ yields the time-indexed error. Summing over $i$ yields the result.
\qed{}

Now we are equipped with a consistent method for imputing the dynamics of a reference curve $\{\nu_i\}_{i \in \{1, \dots, T\}}$ onto a new initial condition via Wasserstein parallel transport. In the following sections, we will discuss and explore applications of this algorithm to various settings and datasets.

\subsection{Application: Causal Inference via  Wasserstein Difference-in-Differences}

In this section we describe an instantiation of our algorithm in Causal Inference -- in particular, we illustrate how it can be used to extend the Difference-in-Differences (DiD) framework from scalar to distribution-valued outcomes. The traditional DiD framework is as follows: we observe subjects at time $0$ and time $1$ where a treatment has been administered somewhere in between, with the observations denoted $(Y_0, Y_1)$. We denote $Y_0(0)$ and $Y_1(0)$ as the counterfactual outcome, had this subject not been treated. While $Y_i(1)$ need not be equal to $Y_i$ in general, we will assume that this is the case -- this is commonly referred to as the \textit{consistency} assumption. We also observe a control group at the two times whose observations are denoted $(Z_0, Z_1).$ Since the control group is untreated, any difference between $Z_0$ and $Z_1$ stems from unobserved confounding. A standard assumption in the Difference-in-Differences framework is the \textit{equibias} or the \textit{parallel trends} assumption, which stipulates that $\mathbb{E}[Y_1(0) - Y_0(0)] = \mathbb{E}[Y_1 - Y_0] - \mathbb{E}[Z_1 - Z_0]$ ensuring that the trends in untreated individuals are \say{parallel} \citep{abadie2005semiparametric} and allowing us to recover the treatment effect. Note that this procedure typically operates on the level of first-order moments, and it does not deal with the full counterfactual distribution. 

% \citet{sofer2016}, however, propose a procedure for dealing with distribution level objects of one dimensional random variables in the form of the quantile function. In particular, the parallel trends assumption in their setting manifests as the assumption that the optimal transport map from the law of $Z \,| \, A = 0$ to $Z \, | \, A = 1$ is the same as the optimal transport map from the law of $Y(a) \, | \,A = 0$ to $Y(a) \, | \,A = 1.$ This idea is an instantiation of the generalized conceptual framework that we propose. (put this in intro)

In many cases, we care about the laws $\mu_t \triangleq \text{Law}(Y_t)$ and $\nu_t \triangleq \text{Law}(Z_t)$ as opposed to statistics like their expected value, but we cannot take differences on $\mathcal{P}_2(\man)$ to subtract off differences in initial conditions.  To remedy this, we will apply Wasserstein parallel transport. In particular, we consider the setting where the population-level treated, untreated and control objects can be represented by piecewise-geodesic trajectories through $\mathcal{P}_2(\man)$. For a set of observation times $\{t_i\}_{i \in \{1, \dots, T\}} \subset [0, 1]$, we denote $(\mu_{t_i})_{i \in \{1, \dots, T\}}$ the path corresponding to the treated population, $(\nu_{t_i})_{i \in \{1, \dots, T\}}$ the path corresponding to the untreated population, and $(\mu_{t_i}^*)_{i \in \{1, \dots, T\}}$ the path corresponding to the counterfactual treated population. Our goal is to emulate the Difference-in-Differences procedure and reconstruct the population counterfactual trajectory $(\mu_{t_i}^*)_{i \in \{1, \dots, T\}}$ given the control trajectories and the initial condition of the treated population $\mu_{t_0}^*$ (before treatment).

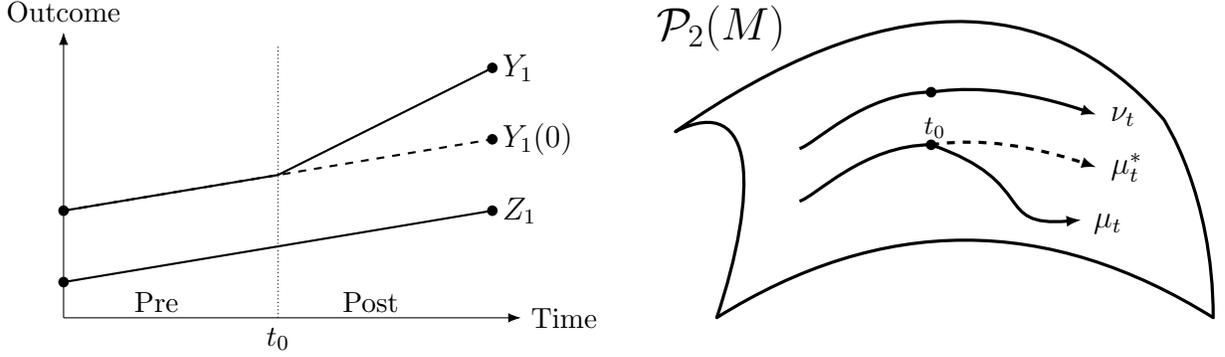
\begin{figure}[t]
  \centering
  \vspace{0.5em} % whitespace above the diagram (adjust as needed)
    \begin{subfigure}[t]{0.4 \textwidth}
    \begin{tikzpicture}[x=0.95cm,y=0.95cm,>=Latex]
    
    % -----------------------
    % Axes
    % -----------------------
    \draw[->] (0,0) -- (6.4,0) node[right] {Time};
    \draw[->] (0,0) -- (0,4.0) node[above] {Outcome};
    
    % Period labels
    \node[below] at (1.3,0.5) {Pre};
    \node[below] at (4.3,0.5) {Post};
    
    % Intervention time (vertical line)
    \draw[densely dotted] (3,0) -- (3,3.8);
    \node[below] at (3,0) {$t_0$};
    
    % -----------------------
    % Coordinates (choose any reasonable numbers)
    % -----------------------
    % Control: pre -> post
    \coordinate (Cpre)  at (0,0.5);
    \coordinate (Cpost) at (6,1.5);
    
    % Treated: observed pre -> observed post
    \coordinate (Tpre)  at (0,1.5);
    \coordinate (Tpost) at (6,3.5);
    
    % NEW: treated value at the intervention time t0 (x=4)
    \coordinate (Tt0)   at (3, 2); % choose y to set the pre-slope
    
    % Treated counterfactual under parallel trends:
    % Tcf = Tpre + (Cpost - Cpre)
    \coordinate (Tcf)   at (6,2.5);
    
    % -----------------------
    % Lines
    % -----------------------
    % Control (solid)
    \draw[thick] (Cpre) -- (Cpost);
    \node[right] at (Cpost) {\large $Z_1$};
    
    % Treated observed (solid)
    \draw[thick] (Tpre) -- (Tt0) -- (Tpost);
    \node[right] at (Tpost) {\large $Y_1$};
    % Treated counterfactual (dashed, parallel to control trend)
    \draw[thick,dashed] (Tpre) -- (Tcf);
    \node[right] at (Tcf) {\large  $Y_1(0)$};
    
    % Points
    \fill (Cpre)  circle (2pt);
    \fill (Cpost) circle (2pt);
    \fill (Tpre)  circle (2pt);
    \fill (Tpost) circle (2pt);
    \fill (Tcf)   circle (2pt);
    
    \end{tikzpicture}
    \end{subfigure}
    \hspace{0.1\textwidth}
    \begin{subfigure}[t]{0.44 \textwidth}
        \raisebox{1.4em}{
        \begin{tikzpicture}[
          scale=1.0,
          >={Latex[length=2.0mm]},
          warp/.style={draw, very thick, line cap=round, line join=round},
          lab/.style={font=\small},
          atom/.style={circle, fill=blue, inner sep=1.6pt},
        ]
        
        % One warped "square-like" manifold + 4 atoms (Diracs)
        \newcommand{\warpsquare}[2]{%
        \begin{scope}[xshift=#1cm,yshift=#2cm, scale=1.1]
          \draw[warp]
            (-3.5, 0.75)
              .. controls (-1.0, 2.5) and ( 1.0, 2.5) .. ( 2.4, 0.9)
              .. controls ( 2.4, 0.9) and ( 3.0,0.0) .. ( 3.0,-1.5)
              .. controls ( 1.0,-0.25) and (-1.0, -0.25) .. (-3.0,-1.5)
              .. controls (-2.5, 0.0) and (-2.5, 1.25) .. (-3.5, 0.75)
            -- cycle;
        
          \node[lab] at (-2.95, 2.0) {\LARGE $\mathcal{P}_2(\man)$};
        \end{scope}
        }
        
        % Two identical manifolds
        \warpsquare{-2}{0}

        % --- Overlay DiD "parallel trends" paths as CURVES on the manifold ---
        \begin{scope}[xshift=-2cm,yshift=0cm]
        
          % A constant shift to create a "parallel" treated pre-trend curve
          % (translation of the entire pre-curve shape)
          \coordinate (Tshift) at (0.0,-0.70);
        
          % -----------------------
          % CONTROL curve: Z_1
          % -----------------------
          % Pre (curved) segment: C0 -> Ct0
          \coordinate (C0)   at (-2.2, 0.60);
          \coordinate (Ct0)  at (-0.45, 1.35);
        
          % Control points chosen to align visually with your warped boundary
          \coordinate (C0c1) at (-2.00, 0.65);
          \coordinate (C0c2) at (-1.35, 1.35);
        
          % Post (curved) segment: Ct0 -> C1
          \coordinate (C1)   at ( 1.75, 1.05);
          \coordinate (C1c1) at (-0.10, 1.40);
          \coordinate (C1c2) at ( 0.55, 1.45);
        
          % Draw control path (curved)
          \draw[very thick, -{Latex[length=2mm]}]
            (C0) .. controls (C0c1) and (C0c2) .. (Ct0)
                 .. controls (C1c1) and (C1c2) .. (C1);
          \node[lab] at (2.10,1.02) {\large $\nu_t$};
        
          % -----------------------
          % TREATED curve: Y_1
          % -----------------------
          % Treated pre endpoints are translated copies of the control pre endpoints.
          \coordinate (T0)  at ($(C0)  + (Tshift)$);
          \coordinate (Tt0) at ($(Ct0) + (Tshift)$);
        
          % Treated observed post endpoint (diverges)
          \coordinate (T1)  at (1.55,-0.35);
        
          % Treated pre control points are also translated (this enforces "parallel" pre-trend)
          \coordinate (T0c1) at ($(C0c1) + (Tshift)$);
          \coordinate (T0c2) at ($(C0c2) + (Tshift)$);
        
          % Draw treated observed path: curved pre (parallel) then diverging curved post
          \draw[very thick, -{Latex[length=2mm]}]
            (T0) .. controls (T0c1) and (T0c2) .. (Tt0)
                .. controls (1.00, 0.15) and (0.35,-0.45) .. (T1);
          \node[lab] at (1.90,-0.4) {\large $\mu_t$};
        
          % -----------------------
          % TREATED counterfactual: Y_1(0)
          % -----------------------
          % Counterfactual post curve is a translated copy of the control post curve.
          \coordinate (Tcf)  at ($(C1)  + (Tshift)$);
          \coordinate (T1c1) at ($(C1c1)+ (Tshift)$);
          \coordinate (T1c2) at ($(C1c2)+ (Tshift)$);
        
          \draw[very thick, dashed, -{Latex[length=2mm]}]
            (Tt0) .. controls (T1c1) and (T1c2) .. (Tcf);
          \node[lab] at ($(Tcf)+(0.4,0.00)$) {\large  $\mu_t^*$};
        
          % Mark t0 points and label t0
          \fill (Ct0) circle (2pt);
          \fill (Tt0) circle (2pt);
          \node[lab] at ($(Ct0)+(0.05,-0.48)$) {$t_0$};
        
        \end{scope}

        \end{tikzpicture}
        }
    \end{subfigure}
  \vspace{0.5em} % whitespace below the diagram (adjust as needed)
  \caption{Difference-in-differences diagram illustrating the parallel trends assumption in the Euclidean case (left) and the Wasserstein case (right). The goal is to reconstruct the counterfactual (denoted $Y_t(0)$ on the left and $\mu_t^*$ on the right).}
  \label{fig:did-parallel-trends}
\end{figure}

\begin{restatable}[Wasserstein Parallel Trends]{asmpt}{} \label{Wasserstein parallel trends}
    We assume that the trajectories $(\nu_{t_i})_{i\in \{1, \dots, T\}}$ and $(\mu_{t_i}^*)_{i\in \{1, \dots, T\}}$ are parallel in the following sense: for all $i \in \{1, \dots, T-1\}$, the tangent velocity at $\nu_{t_i}$ given by $\nabla \varphi_{t_i} \triangleq \mathbf{log}_{\nu_{t_i}}(\nu_{t_{i+1}})$ and the tangent velocity at $\mu_i^*$ given by $\nabla\varphi_{t_i}^* =  \mathbf{log}_{\mu_{t_i}^*}(\mu_{t_{i+1}}^*)$ are parallel with respect to the connection $\nabla^{W_2}$ in the sense that
    \[\nabla\varphi_{t_i}^* = {\PT}_{\nu_{t_i} \rightarrow \mu_{t_i}^*}(\nabla \varphi_{t_i})\]
    where ${\PT}_{\nu_{t_i} \rightarrow \mu_{t_i}^*}$ is the parallel transport of $\nabla\varphi_{t_i}$ along the Wasserstein geodesic between $\nu_{t_i}$ and $\mu_{t_i}^*$.
\end{restatable}
% \ndfg{I think in the $\mathcal{P}_2(R^d)$ case these two notions are the same, no? ANother idea: for pre-treatment tests, we can actually test both assumptions and if there's a difference, we know that curvature plays a role in the setting. } 

In order for the counterfactual trajectory to be recoverable, we need an assumption analogous to the parallel trends assumption in Euclidean Differences-in-Differences. Since we are not considering objects living in a vector space, we have no notion of taking differences. Instead we will use Wasserstein parallel transport to formulate an analogous assumption. This assumption, which we call \textit{Wasserstein Parallel Trends}, is described in \Cref{Wasserstein parallel trends}, and we provide a visualization in \Cref{fig:did-parallel-trends}. Furthermore, we have the following key result, which establishes that when $\man = \mathbb{R}^d$, parallel curves in Wasserstein space have \textit{parallel means too}. 

\begin{restatable}[Wasserstein Parallel Trends recovers Parallel Trends in $\mathbb{R}^d$]{thm}{} \label{pt recovers parallel trends}
    Suppose $(\nu_t)_{t \in [0,1]}$ and $(\mu_t^*)_{t\in [0,1]}$ are curves of absolutely continuous measures in $\mathcal{P}_2(\mathbb{R}^d)$, and suppose that the curves have tangent velocity fields $\nabla \varphi_t, \nabla\varphi_t^*$ that are parallel with respect to $\nabla^{W_2}$ for all $t$, in the sense that 
    \[\nabla\varphi_t^* = {\PT}_{\nu_t \rightarrow \mu_t^*}(\nabla\varphi_t) \quad \forall \, t\in [0,1].\]
    Then it holds that for a.e. $t$ 
    \[ \frac{d}{dt}\int_{\mathbb{R}^d}x \,d\nu_t(x) = \frac{d}{dt} \int_{\mathbb{R}^d} x\, d\mu_t^*(x).\] 
\end{restatable}
We provide the proof of \Cref{pt recovers parallel trends} in \Cref{sec: proof of pt recovers parallel trends}. This result ensures that our assumptions recover the classical \textit{equibias} assumption, while also capturing changes in the global \say{shape} of the distribution. We believe that this result motivates the application of Wasserstein parallel transport and our proposed counterfactual dynamics procedure in Causal inference settings; in particular, augmenting DiD with Wasserstein parallel transport would enable estimation of treatment effects that manifest through changes in the \textit{shape} of the distribution \textit{in addition} to changes in the means.

\section{Experiments}\label{sec: experiments}

We now pivot to demonstrating the utility of our conceptual framework and proposed algorithms through simulations. We begin by showing simulation results of our approximation scheme on synthetic data, documented in \Cref{sec: synthetic data experiments}. Then, we apply our counterfactual dynamics prediction method to two single-cell RNA sequencing (scRNAseq) datasets, where we impute the gene-level dynamics of one biological system onto another. We provide these genomics experiments in \Cref{sec: genomics experiments}. For a detailed treatment of our exact computational implementation of approximate Wasserstein parallel transport (\Cref{alg: W parallel transport}) and the counterfactual dynamics prediction method (\Cref{alg: W trajectory reconstruction}), please see \Cref{appendix:implementation_details}. 

\subsection{Simulations} \label{sec: synthetic data experiments}

We begin by evaluating our proposed algorithms on synthetic data. In particular, we evaluate \Cref{alg: W trajectory reconstruction} on a stochastic system of time-evolving Gaussian measures, as in this regime we have access to closed form parallel transport expressions (which we derived in \Cref{sec: WPT with Gaussians}). We explore how the dimension $d$ affects the quality of the predicted trajectories as measured by Wasserstein distance, and we compare across methods. 

As mentioned, \Cref{fig: cf combined} illustrates an application of \Cref{alg: W trajectory reconstruction} to a system of Gaussian measures evolving in parallel. The top row consists of a sequence of measures where the mean evolves in time, while the bottom row illustrates a sequence where both the mean and the covariance evolves in time. We display the reconstructed parallel curves in warm colors for our method (\Cref{alg: W trajectory reconstruction}) and for two baselines: the first baseline applies the same change in mean as observed in the control trajectory to the counterfactual initial condition. The second baseline uses the trajectory induced by the Brenier map between successive time points on the control curve \textit{without} using parallel transport. In this setup we use the ground truth Brenier map between the underlying control Gaussian measures as extending the empirical Brenier map out of sample is nontrivial -- while this is not a tractable baseline in practice, we believe it to be illustrative.

\Cref{fig: cf combined} indicates that the cumulative error for the system of Gaussians with evolving means and stationary covariances grows at a comparable rate for our method and the baselines as expected. Since the covariances of the sequence of measures is unchanging, one would expect that the mean-sensitive baseline is successful. Moreover, the stationary nature of the covariances guarantees that the Brenier map is just a shift; this explains the good performance. For the system with a non-stationary covariance, however, we see a significantly faster growth in the cumulative error in the baseline methods. The mean-shifted baseline method evidently doesn't incorporate information about the change in the shape of the distribution over time, while the Brenier map baseline clearly doesn't lead to imputed dynamics that are parallel in any desirable sense. 

\begin{figure}[htbp]
    \centering
    % First Image
    \begin{subfigure}{\linewidth}
        \centering
        \includegraphics[width=\linewidth]{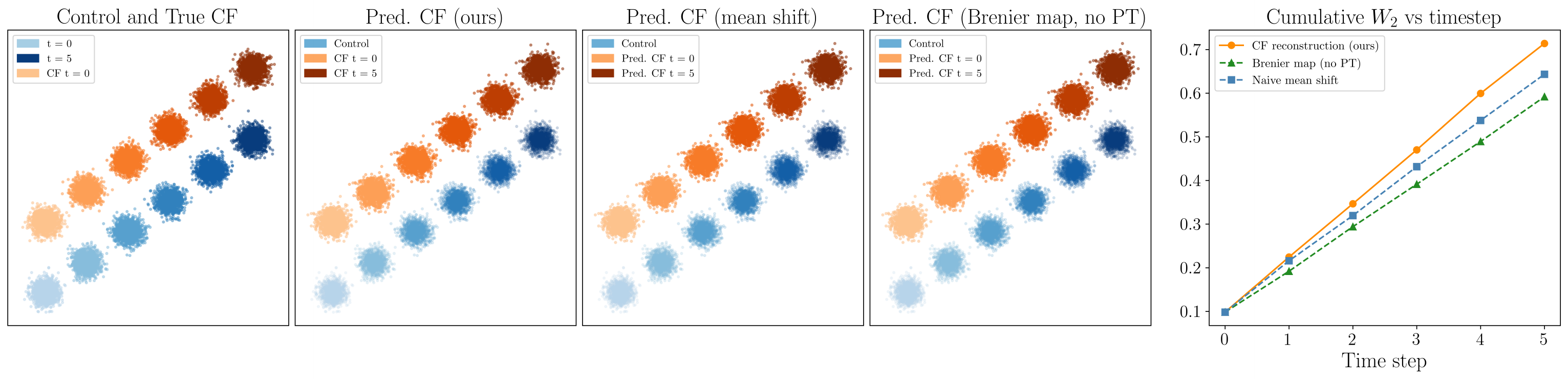}
        \caption{Parallel sequences of Gaussian measures with changing means.}
        \label{fig:linear_shift}
    \end{subfigure}
    
    \vspace{1em} % Adds a bit of vertical space between the images

    % Second Image
    \begin{subfigure}{\linewidth}
        \centering
        \includegraphics[width=\linewidth]{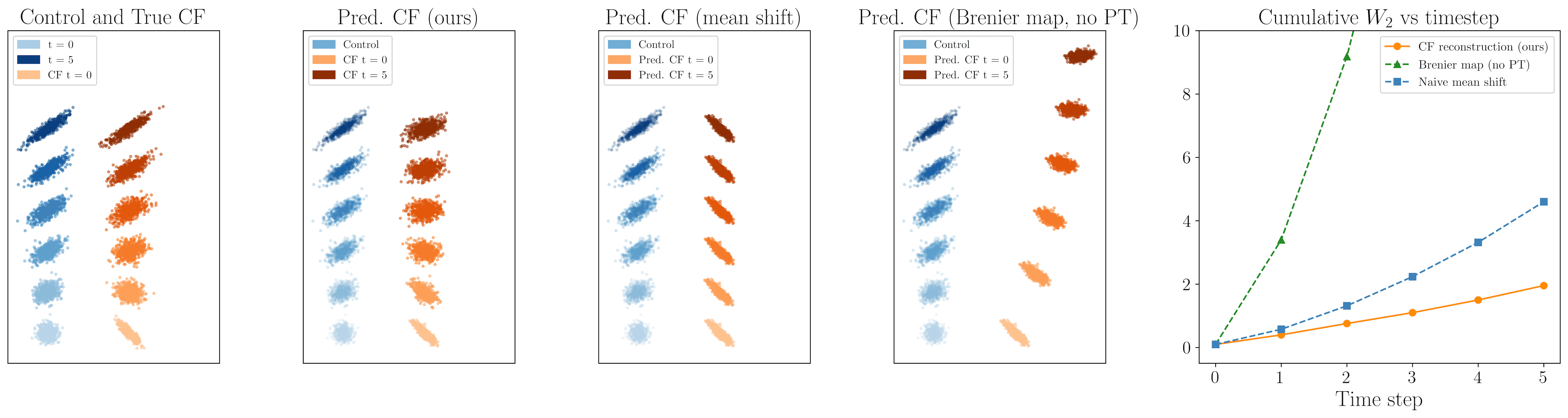}
        \caption{Parallel sequences of Gaussian measures with changing means and changing covariances.}
        \label{fig:cov_deform}
    \end{subfigure}
    
    \caption{Illustration of the parallel transported dynamics for two systems of time-evolving Gaussian measures using our method and two baselines.}
    \label{fig: cf combined}
\end{figure}

To get a sense of how the performance of our method depends on dimension, we also evaluate on Gaussian measures for $d \in \{2, 4, 8, 12, 16\}$. In particular, we compute the Wasserstein distance between the true and predicted pushforward of a parallel transported tangent between two random Gaussian measures. We report the results of our method, the mean-shift baseline and the (non parallel transported) Brenier map baseline as a function of dimension in \Cref{tab:dim_sweep}. In this table we add an \say{estimated} Gaussian baseline (labeled \say{Gauss PT}), which estimates the parameters of a Gaussian and then applies closed form Gaussian parallel transport (\Cref{Gaussian parallel transport}). We wish to highlight two key takeaways from this experiment -- firstly, we see that our method \textit{without} the Helmholtz projection consistently outperforms all other methods (except the estimated Gaussian baseline) for small $d$ ($d < 10$). The superior performance of the estimated Gaussian baseline can be attributed to the fact that WPT$^-$ is more general, and the estimated Gaussian baseline is applicable only for Gaussian measures. When $d > 10$, however, our method is comparable in performance to the simple mean-shift baseline; this is unsurprising, as the empirical mean estimates the true mean at a $\sqrt{n}$-rate, while plugin estimators for optimal transport maps are known to suffer from error rates that have an exponential dependence on the dimension $d$ \citep{deb2021rates, manole2024plugin}. We therefore suggest that for high dimensional problems, one should opt to employ the Gaussian estimation or the simple mean-shift baseline. 

\begin{table}[h]
\centering
\small
\begin{tabular}{ccccccccccc}
\toprule
Step & \multicolumn{5}{c}{$d=2$} & \multicolumn{5}{c}{$d=4$} \\
\cmidrule(lr){2-6}\cmidrule(lr){7-11}
& WPT & WPT$^-$ & Brenier & Shift & Gauss PT & WPT & WPT$^-$ & Brenier & Shift & Gauss PT \\
\midrule
1 & 0.606 & \underline{0.142} & 2.242 & 0.277 & \textbf{0.118} & 3.300 & \textbf{0.400} & 4.669 & 0.512 & \underline{0.401} \\
2 & 1.609 & \underline{0.121} & 3.740 & 0.418 & \textbf{0.101} & 7.490 & \underline{0.432} & 7.765 & 0.728 & \textbf{0.429} \\
3 & 2.988 & \underline{0.124} & 4.926 & 0.582 & \textbf{0.097} & 12.700 & \textbf{0.455} & 10.283 & 0.935 & \underline{0.462} \\
4 & 4.816 & \underline{0.145} & 5.965 & 0.729 & \textbf{0.139} & 18.462 & \underline{0.475} & 12.477 & 1.122 & \textbf{0.463} \\
5 & 7.163 & \underline{0.114} & 6.884 & 0.838 & \textbf{0.090} & 24.186 & \underline{0.486} & 14.463 & 1.264 & \textbf{0.466} \\
\bottomrule
\end{tabular}

\vspace{0.7em}

\begin{tabular}{ccccccccccc}
\toprule
Step & \multicolumn{5}{c}{$d=8$} & \multicolumn{5}{c}{$d=12$} \\
\cmidrule(lr){2-6}\cmidrule(lr){7-11}
& WPT & WPT$^-$ & Brenier & Shift & Gauss PT & WPT & WPT$^-$ & Brenier & Shift & Gauss PT \\
\midrule
1 & 5.285 & \textbf{1.277} & 7.091 & 1.333 & \underline{1.284} & 5.534 & 1.943 & 6.483 & \underline{1.920} & \textbf{1.897} \\
2 & 11.644 & \underline{1.446} & 12.026 & 1.697 & \textbf{1.423} & 12.555 & \underline{2.199} & 10.869 & 2.311 & \textbf{2.067} \\
3 & 18.130 & \underline{1.559} & 15.935 & 1.992 & \textbf{1.517} & 20.659 & \underline{2.415} & 14.640 & 2.707 & \textbf{2.236} \\
4 & 25.011 & \underline{1.673} & 19.322 & 2.269 & \textbf{1.625} & 29.407 & \underline{2.560} & 17.918 & 3.074 & \textbf{2.322} \\
5 & 31.928 & \underline{1.740} & 22.533 & 2.516 & \textbf{1.686} & 38.225 & \underline{2.681} & 20.841 & 3.391 & \textbf{2.405} \\
\bottomrule
\end{tabular}

\vspace{0.7em}

\begin{tabular}{cccccc}
\toprule
Step & \multicolumn{5}{c}{$d=16$} \\
\cmidrule(lr){2-6}
& WPT & WPT$^-$ & Brenier & Shift & Gauss PT \\
\midrule
1 & 6.660 & 2.651 & 8.642 & \textbf{2.489} & \underline{2.544} \\
2 & 14.538 & 3.048 & 13.875 & \underline{2.922} & \textbf{2.842} \\
3 & 23.431 & 3.341 & 18.104 & \underline{3.329} & \textbf{3.067} \\
4 & 32.890 & \underline{3.577} & 21.703 & 3.747 & \textbf{3.216} \\
5 & 42.623 & \underline{3.737} & 24.906 & 4.089 & \textbf{3.400} \\
\bottomrule
\end{tabular}

\caption{$W_2$ between predicted and true Gaussian CF (under Wasserstein Parallel Trends) across dimensions using $5000$ samples. WPT denotes Wasserstein parallel transport with projection (\Cref{alg: W trajectory reconstruction}), WPT$^-$ omits the projection step, Brenier applies the control OT map without transport, Shift applies the mean-shift baseline, and Gauss PT fits Gaussian parameters from samples and applies closed-form parallel transport. Lower is better; \textbf{bold} marks the best method at each step and \underline{underlining} marks the second-best.}
\label{tab:dim_sweep}
\end{table}

The second takeaway that we wish to convey is the apparent poor performance of our method when the Helmholtz projection \textit{is} included. We attribute this to the nonparametric rate of convergence associated with RKHS methods, which indicates that the risk of our estimator likely depends exponentially on $d$. This rate is slow enough to explain the poor performance of WPT with the Helmholtz projection even when $d$ is fairly small. 

\subsection{Real Data: Genomics} \label{sec: genomics experiments}

We now evaluate our method on two publicly available genomics datasets. Specifically, we apply our method to the dataset of human and chimp cerebral organoid cells undergoing developmental changes \citep{kanton2019organoid}, and a dataset of microglial cells from mice evolving over time in response to spared nerve injury (SNI) \citep{tansley2022single}. In both datasets, we have access to per-cell gene counts via single-cell transcriptomics -- as a result, we regard each cell $x$ as a point in $\mathbb{R}^d$, with $d$ being the total number of genes, where each coordinate quantifies the extent to which the cell expresses that gene. We regard a population of cells $\{x_1, \dots, x_n\}$ at a single time point as an empirical measure over $\mathbb{R}^d$, assigning a mass of $1/n$ to each datapoint. These time-varying empirical measures will be the object of study in our subsequent experiments.

\begin{table}[h]
\centering
\small
\begin{tabular}{ccccc}
\toprule
$n_{\mathrm{pcs}}$ & \multicolumn{2}{c}{mid ($\sim$ day 65)} & \multicolumn{2}{c}{late ($\sim$ day 110)} \\
\cmidrule(lr){2-3}\cmidrule(lr){4-5}
 & WPT$^-$ & Mean-shift & WPT$^-$ & Mean-shift \\
\midrule
2  & 1.262 & \textbf{1.230} & \textbf{0.817} & 1.367 \\
3  & \textbf{1.262} & 1.298 & \textbf{0.833} & 1.433 \\
5  & \textbf{1.294} & 1.376 & \textbf{0.896} & 1.621 \\
7  & \textbf{1.389} & 1.474 & \textbf{1.121} & 1.707 \\
10 & \textbf{1.537} & 1.578 & \textbf{1.428} & 1.851 \\
15 & \textbf{1.712} & 1.776 & \textbf{1.683} & 2.047 \\
20 & \textbf{1.822} & 1.981 & \textbf{1.897} & 2.211 \\
\bottomrule
\end{tabular}
\caption{$W_2$ between predicted and observed human organoid gene expressions per non-initial time point for WPT$^-$ and mean-shift baseline as a function of the number of principal components used.}
\label{tab:kanton_w2_npcs}
\end{table}

Before proceeding, we remark on a key subtlety that we ask the reader to keep in mind: in the following experiments, our goal is to verify whether or not these cellular systems develop in parallel in the Wasserstein sense when subjected to similar biological dynamics. We are not making any causal claims -- instead, we want these experiments to provide evidence that Wasserstein parallel trends is a reasonable assumption to adopt in genomics. We hope that this would provide \textit{support} for applying our method to biological systems to make meaningful causal claims about counterfactual dynamics \textit{downstream}.

\begin{figure}
    \centering
    \includegraphics[width=0.8\linewidth]{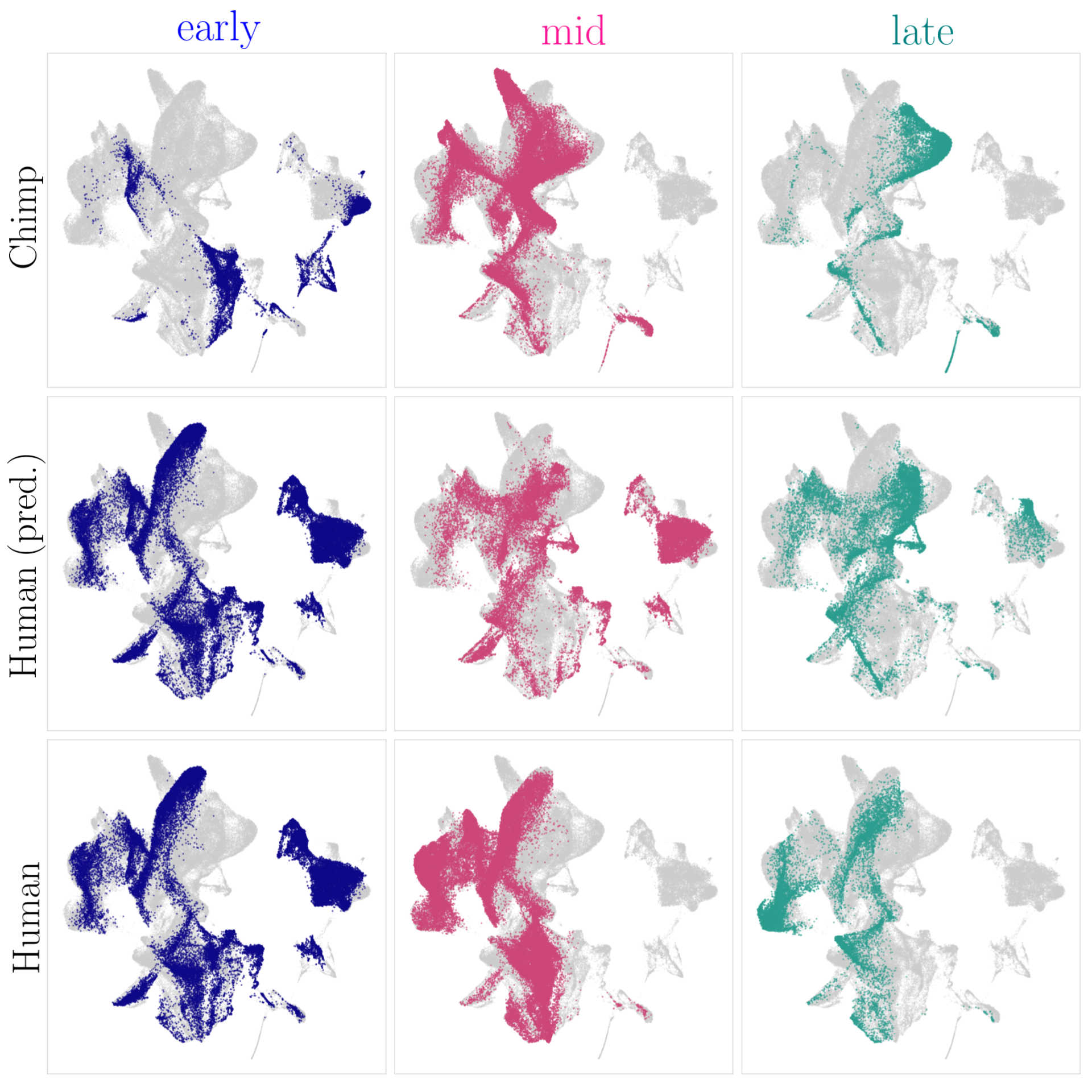}
    \caption{UMAP \citep{mcinnes2018umap} visualizations of chimp (top row), predicted human (middle row) and observed human (bottom row) organoid dynamics over time using $n_{\text{pcs}} = 15$.}
    \label{fig:kanton_umap}
\end{figure}

\begin{figure}
    \centering
    \includegraphics[width=\linewidth]{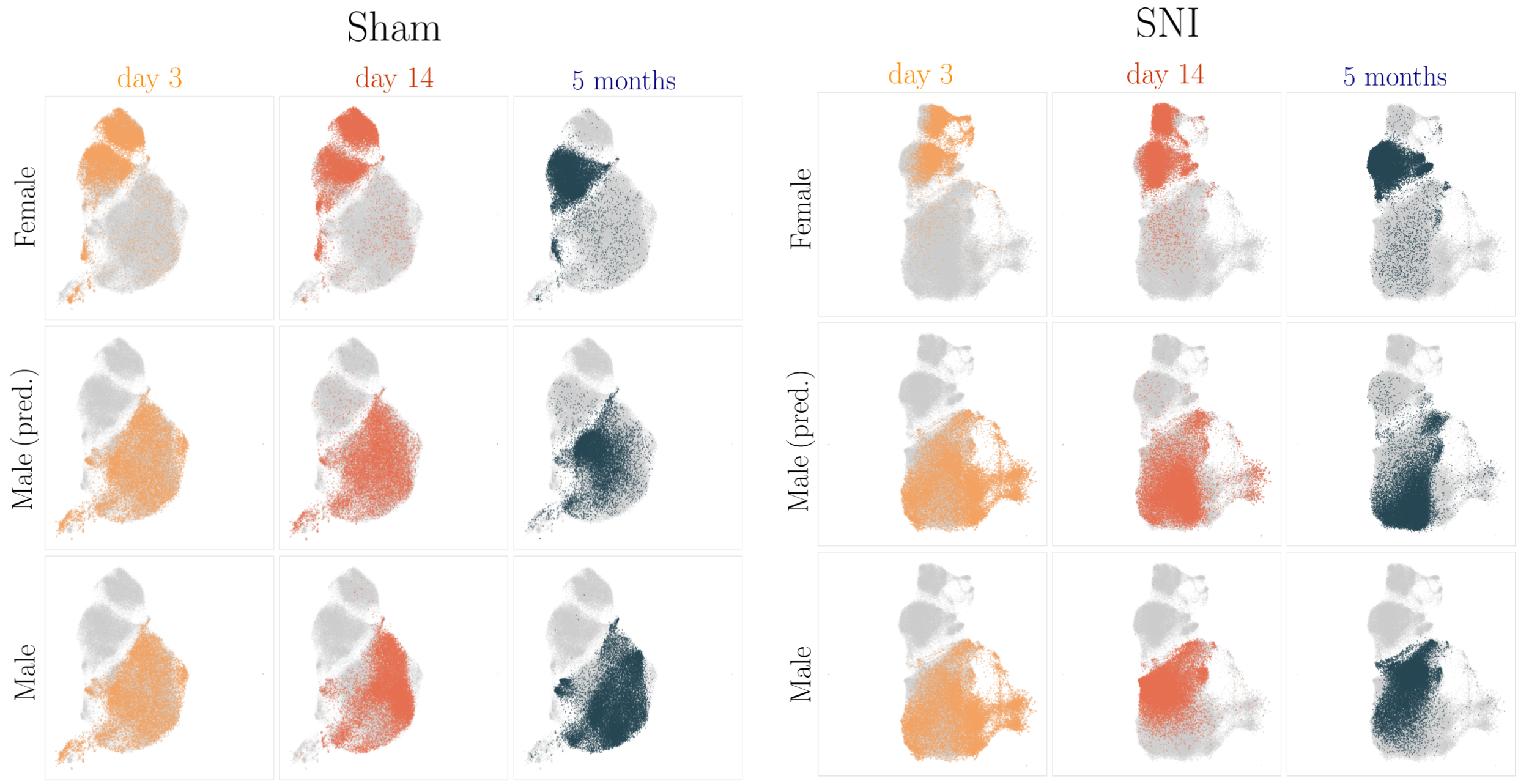}
    \caption{UMAP \citep{mcinnes2018umap} visualizations of female (top row), predicted male (middle row) and observed male (bottom row) microglial dynamics over time using $n_{\text{pcs}} = 15$. The panel on the left uses only cells from the untreated (sham) group, while the panel on the right uses only cells from the treated (SNI) group.}
    \label{fig:tansley}
\end{figure}
\begin{table}[h]
\centering
\small

\begin{subtable}[t]{0.48\textwidth}
\centering
\begin{tabular}{ccccc}
\toprule
$n_{\mathrm{pcs}}$ & \multicolumn{2}{c}{day 14} & \multicolumn{2}{c}{5 months} \\
\cmidrule(lr){2-3}\cmidrule(lr){4-5}
 & WPT$^-$ & Mean-shift & WPT$^-$ & Mean-shift \\
\midrule
2  & \textbf{0.489} & 0.601 & 0.600 & \textbf{0.575} \\
3  & 0.799 & \textbf{0.719} & \textbf{0.574} & 0.613 \\
5  & 0.893 & \textbf{0.854} & \textbf{0.971} & 1.120 \\
7  & 1.090 & \textbf{1.053} & \textbf{1.146} & 1.228 \\
10 & \textbf{1.346} & 1.354 & \textbf{1.434} & 1.507 \\
15 & 1.630 & \textbf{1.616} & \textbf{1.728} & 1.742 \\
20 & 1.916 & \textbf{1.883} & \textbf{1.980} & 1.984 \\
\bottomrule
\end{tabular}
\caption{Sham}
\label{tab:mf_w2_npcs_sham}
\end{subtable}
\hfill
\begin{subtable}[t]{0.48\textwidth}
\centering
\begin{tabular}{ccccc}
\toprule
$n_{\mathrm{pcs}}$ & \multicolumn{2}{c}{day 14} & \multicolumn{2}{c}{5 months} \\
\cmidrule(lr){2-3}\cmidrule(lr){4-5}
 & WPT$^-$ & Mean-shift & WPT$^-$ & Mean-shift \\
\midrule
2  & \textbf{0.711} & 0.714 & 0.879 & \textbf{0.838} \\
3  & \textbf{0.928} & 0.998 & \textbf{0.875} & 0.893 \\
5  & \textbf{1.113} & 1.176 & \textbf{1.126} & 1.355 \\
7  & 1.277 & \textbf{1.271} & \textbf{1.283} & 1.441 \\
10 & 1.458 & \textbf{1.429} & \textbf{1.488} & 1.589 \\
15 & 1.703 & \textbf{1.683} & \textbf{1.785} & 1.821 \\
20 & 1.930 & \textbf{1.905} & \textbf{2.021} & 2.026 \\
\bottomrule
\end{tabular}
\caption{SNI}
\label{tab:mf_w2_npcs_sni}
\end{subtable}

\caption{$W_2$ between predicted and observed male microglial gene expressions per non-initial timepoint for WPT$^-$ and mean-shift baseline as a function of the number of principal components used.}
\label{tab:mf_w2_npcs}
\end{table}

Having established the motivation for this section, we can now apply \Cref{alg: W trajectory reconstruction} to the dataset of cerebral organoids from \citet{kanton2019organoid}. In particular, we use WPT to impute the dynamics of the chimp organoids over time onto the human organoid initial conditions -- we then compute the Wasserstein distance between the imputed point clouds and the observed human organoid point clouds and compare the distances with those obtained when using the mean-shift baseline. Since there is no \say{treatment} in this setup, this experiment is not causal and we are not predicting a counterfactual. However, if the observed dynamics were a precursor to a treatment period, this experiment described would be a reasonable justification for the Wasserstein parallel trends assumption. 

We provide the Wasserstein distances between predicted and true point clouds in \Cref{tab:kanton_w2_npcs}, while we provide UMAP visualizations of the predicted dynamics produced by our method in \Cref{fig:kanton_umap} \citep{mcinnes2018umap}. As is standard in genomics, we project the datasets onto their top $n_{\text{pcs}}$ principal components, as the raw dataset contains on the order of $20000$ genes, most of which are not expressed. To get a sense of the dependence of the performance of our method on the dimension, we vary $n_{\text{pcs}}$ from $2$ to $20$ and include the errors for each in \Cref{tab:kanton_w2_npcs}. We find that our method produces dynamics that are closer in Wasserstein distance to the truth than the dynamics produced by the mean-shift baseline in all but one scenario. Since \Cref{pt recovers parallel trends} indicates that Wasserstein parallel transport gives rise to parallel trends of the means, we consider this strong evidence to suggest that the change in the \textit{shape} of point clouds captured by Wasserstein parallel trends is indeed biologically meaningful.

To provide further support for this notion, we apply the same experiment to a dataset of gene expression of microglial cells over time from \citet{tansley2022single}. In this experiment, we predict the evolution of microglial cells through gene expression space in males from the evolution in females. Since the study from which this data was obtained addresses a causal question -- how microglial gene expression changes in the presence of spared nerve injury (SNI) -- we run the experiment separately for the control (called \say{sham}) and the SNI group. We provide the Wasserstein distances between predicted point clouds and observed point clouds for our method and the mean shift baseline in \Cref{tab:mf_w2_npcs}, and we provide a visualization of the predictions produced by our method in \Cref{fig:tansley}. In this dataset we find that the mean-shift method closer in performance to our method than we saw in the dataset of cerebral organoids. With that said, this seems to hold only for the \textit{first} predicted time step -- for the second predicted time step we see that our method consistently outperforms the mean shift baseline.

\section{Conclusion}

In this work we proposed a conceptual framework for formalizing the notion of distribution-level parallel dynamics through the mechanism of parallel transport on the space of probability measures endowed with optimal transport geometry. In particular, we proposed a tractable procedure for approximating parallel transport on the space of probability measures over a Riemannian manifold $\man$. We also instantiated the method for $\man = \mathbb{R}^d$ in an algorithm to predict the dynamics of one system of time-evolving measures from another under the assumption of Wasserstein parallel trends, and we separately derived closed form expressions of Wasserstein parallel transport for Gaussian measures. Finally, we evaluated our parallel transport approximation scheme in simulation using the ground truth Gaussian parallel transport, and we deployed our dynamics prediction procedure on real genomics datasets of time-evolving cellular systems. 

While our work proposes and deploys a new procedure for imputing the dynamics of one statistical system onto another, there are many open questions that we have left unaddressed. In particular, we believe that a fruitful direction of future study would be an exploration of the statistical properties of Wasserstein parallel transport and, in turn, dynamics prediction with Wasserstein parallel transport. Understanding the rates of convergence, optimality properties and limiting laws would enable rigorous statistical inferences to be made when this framework is applied in practice. We believe this would be especially useful for Difference-in-Differences, as it would allow one to perform hypothesis testing on observed data from pretreatment periods to assess whether Wasserstein parallel trends is a reasonable assumption to adopt. 

In concurrent work, we have also been developing the theory of parallel transport on the space of non-negative Radon measures using \textit{Hellinger-Kantorovich} geometry -- such a theory would allow for a rigorous notion of parallel trends on the space of measures with total mass that is allowed to vary. This concurrent work has relied heavily on the results derived in this paper, as one can construct an equivalence (in a delicate sense) between Hellinger-Kantorovich geometry on $\mathbb{R}^d$ and Wasserstein geometry on a \textit{metric cone} of $\mathbb{R}^d$ \citep{liero2016optimal, liero2018optimal}. We believe that the completion of this theory of unbalanced parallel transport would be especially applicable to genomics, as cellular systems experience mass growth and shrinkage over time. As with Wasserstein parallel transport, we believe that an abundance of interesting statistical questions would follow.

\clearpage

{\bibliography{main.bib}}

\appendix

\section{Stability Theory for Wasserstein Parallel Transport}

For the stability theory in this work, we choose to take $\man = \mathbb{T}^d$, where $\mathbb{T}^d = \mathbb{R}^d/\mathbb{Z}^d$ is the \textit{flat torus.} This choice isn't uncommon in optimal transport theory, as it avoids extensive boundary issues that arise in the PDE machinery underpinning optimal transport (see \citet{manole2023central}, for instance). Thus, throughout this section of the appendix, we adopt the following definitions. For each absolutely continuous measure
\(\mu=\rho\,dx\in\mathcal P_2(\mathbb T^d)\), we define
\[
T_\mu\mathcal P_2(\mathbb T^d)
=
\overline{\{\nabla\varphi:\varphi\in C^\infty(\mathbb T^d)\}}^{L^2(\mu)},
\qquad
\Pi_\mu:L^2(\mu;\mathbb R^d)\to T_\mu\mathcal P_2(\mathbb T^d)
\]
where $C^\infty(\mathbb T^d)$ is the space of smooth periodic functions over $\mathbb{T}^d$. We also identify \(\Pi_\mu z\) with its unique representative \(\nabla\psi\) on \(\mathbb{T}^d\),
where \(\psi\in H^1(\mathbb{T}^d)\) solves
\[
\int_{\mathbb{T}^d} \rho\, \langle\nabla\psi, \nabla \xi\rangle\,dx
=
\int_{\mathbb{T}^d} \rho\, \langle z, \nabla \xi\rangle\,dx
\qquad
\forall \xi\in H^1(\mathbb{T}^d).
\]
We show this identification rigorously below. We note that this identification allows us to represent projected fields as gradients of Sobolev functions; in the original tangent space definition, the $L^2(\mu)$ closure prevents one from doing so.

\begin{restatable}[Tangent space realization]{lemma}{} \label{tangent space realization}
    Let $\mu$ be a probability measure on $\mathbb{T}^d$ that admits a Lebesgue density $\rho$ such that $0 < \lambda \leq \rho(x) \leq \Lambda < \infty$ for a.e. $x \in \mathbb{T}^d$, and define 
    \[T_\mu\mathcal P_2(\mathbb T^d)
    =
    \overline{\{\nabla\varphi:\varphi\in C^\infty(\mathbb T^d)\}}^{L^2(\mu)}.\]
    Then 
    \[T_\mu\mathcal{P}_2(\mathbb{T}^d) = \{\nabla \psi: \psi \in H^1(\mathbb{T}^d)\}.\]
    In particular, every $u \in T_\mu\mathcal{P}_2(\mathbb{T}^d)$ admits a representative of the form $u = \nabla \psi$ on $\mathbb{T}^d$ for some $\psi \in H^1(\mathbb{T}^d)$, where $\psi$ is unique up to an additive constant.  
\end{restatable}
\begin{proof}
    Because $0 < \lambda \leq \rho \leq \Lambda$ a.e. on $\mathbb{T}^d$, the norms $\|\cdot \|_{L^2(\mu)}$ and $\|\cdot\|_{L^2(\mathbb{T}^d)}$ are equivalent. Thus, it suffices to prove that 
    \[\overline{\{\nabla\varphi: \varphi \in C^\infty(\mathbb{T}^d)\}}^{L^2(\mathbb{T}^d)} = \{\nabla \psi : \psi \in H^1(\mathbb{T}^d)\}.\]
    To do so, we'll start by showing the inclusion $\subseteq$. Let $\varphi_n$ be a sequence of functions in $C^\infty(\mathbb{T}^d)$ such that $\nabla \varphi_n \rightarrow u$ in $L^2(\mathbb{T}^d; \mathbb{R}^d)$ for some $u \in L^2(\mathbb{T}^d; \mathbb{R}^d)$. We will show that $u$ is necessarily a gradient of a Sobolev function. To do so, set 
    \[\tilde \varphi_n \triangleq \varphi_n - \int_{\mathbb{T}^d}\varphi_n(x)\, dx\]
    as a de-meaned instance of $\varphi_n$. Observe that $\int_{\mathbb{T}^d}\tilde \varphi_n \,dx = 0,$ and $\nabla \tilde \varphi_n = \nabla \varphi_n$. By the Poincar\'e-Wirtinger inequality \citep{cioranescu1999introduction},
    \[\|\tilde \varphi_n - \tilde \varphi_m\|_{L^2(\mathbb{T}^d)} \leq C_{\text{p}}\|\nabla \tilde \varphi_n - \nabla \tilde \varphi_m\|_{L^2(\mathbb{T}^d)}.\]
    Thus,
    \[\|\tilde \varphi_n - \tilde \varphi_m\|_{H^1(\mathbb{T}^d)} \leq (1 + C_{\text{p}})\|\nabla \varphi_n - \nabla \varphi_m\|_{L^2(\mathbb{T}^d)}.\]
    Since $\nabla \varphi_n$ is convergent in $L^2(\mathbb{T}^d; \mathbb{R}^d)$ it is Cauchy. By the inequality above, we know that $\tilde \varphi_n$ must then be Cauchy in $H^1(\mathbb{T}^d)$. Since $H^1(\mathbb{T}^d)$ is a Banach space, we know there exists a $\psi \in H^1(\mathbb{T}^d)$ such that $\tilde \varphi_n \rightarrow \psi$ in $H^1(\mathbb{T}^d).$ This convergence in a Sobolev sense implies that $\nabla \tilde \varphi_n = \nabla \varphi_n \rightarrow \nabla \psi$ in $L^2(\mathbb{T}^d).$ Since $\nabla \varphi_n \rightarrow u$ we conclude that $u = \nabla \psi$. Now we will show the reverse inclusion $\supseteq$. Let $\psi \in H^1(\mathbb{T}^d)$. Since $C^\infty(\mathbb{T}^d)$ is dense in $H^1(\mathbb{T}^d)$, we know there exists a sequence $\varphi_n \in C^\infty(\mathbb{T}^d)$ such that $\varphi_n \rightarrow \psi$ in $H^1(\mathbb{T}^d).$ This implies that $\nabla \varphi_n \rightarrow \nabla \psi$ in $L^2(\mathbb{T}^d; \mathbb{R}^d)$. Thus, $\nabla \psi$ is in the $L^2(\mathbb{T}^d)$ closure of $\{\nabla \varphi: \varphi \in C^\infty(\mathbb{T}^d)\}$. To establish the uniqueness up to an additive constant of $\psi$, observe the following: if $\nabla \psi_1 = \nabla \psi_2 \in L^2(\mu)$, then $\psi_1 - \psi_2$ is constant a.e. on $\mathbb{T}^d$. Thus the potential $\psi$ is unique up to an additive constant. 
\end{proof}

\begin{restatable}[Projection realization]{thm}{} \label{projection realization}
    Let $\mu$ be a probability measure on $\mathbb{T}^d$ that admits a Lebesgue density $\rho$ such that $0 < \lambda \leq \rho(x) \leq \Lambda < \infty$ for a.e. $x \in \mathbb{T}^d$, and let 
    \[\Pi_\mu: 
     L^2(\mu; \mathbb{R}^d)\longrightarrow \overline{\{\nabla\varphi:\varphi\in C^\infty(\mathbb T^d)\}}^{L^2(\mu)}\]
     denote the orthogonal projection onto the Wasserstein tangent space on $\mathbb{T}^d$. Then for every element $z \in L^2(\mu; \mathbb{R}^d)$ there exists a unique $\psi \in H^1(\mathbb{T}^d)$ such that $\int_{\mathbb{T}^d} \psi\, dx = 0$, $\Pi_\mu z = \nabla \psi \in L^2(\mu; \mathbb{R}^d)$, and 
     \[\int_{\mathbb{T}^d}\rho \,\langle \nabla \psi , \nabla \xi  \rangle\, dx = \int_{\mathbb{T}^d}\rho \, \langle z, \nabla \xi \rangle\, dx \]
     for all $\xi \in H^1(\mathbb{T}^d).$
\end{restatable}
\begin{proof}
    Define 
    \[H_{{\diamond}}^1(\mathbb{T}^d) \triangleq \left\{\xi \in H^1(\mathbb{T}^d) : \int_{\mathbb{T}^d}\xi \, dx = 0\right\}.\]
    Since $H^1_\diamond(\mathbb{T}^d)$ is a subspace of $H^1(\mathbb{T}^d)$, it inherits the same norm. In this proof, our goal will be to apply the Lax-Milgram theorem (\Cref{lax-milgram}) with the Hilbert space $H^1_{\diamond}(\mathbb{T}^d)$ to establish the uniqueness of $\psi \in H^1_\diamond(\mathbb{T}^d).$  To do so, define the functionals
    \[B(\psi, \xi) \triangleq \int_{\mathbb{T}^d}\rho \, \langle \nabla \psi, \nabla \xi\rangle \, dx, \quad \text{and} \quad f(\xi) \triangleq \int_{\mathbb{T}^d}\rho \, \langle z, \nabla \xi\rangle\, dx.\]
    To apply the Lax-Milgram theorem, we need to show that (1) the functional $B$ is upper bounded by a quantity proportional to the product of the norms of its arguments, (2) lower bounded by a quantity proportional to the squared norm of its first argument, and (3) $f$ is a bounded linear functional on $H^1(\mathbb{T}^d).$ We will start by showing (1). Using the density bounds and Cauchy-Schwartz, 
    \begin{align*}
        |B(\psi, \xi)| \leq \Lambda\|\nabla \psi\|_{L^2(\mathbb{T}^d)} \|\nabla \xi\|_{L^2(\mathbb{T}^d)} \leq \Lambda \|\psi\|_{H^1(\mathbb{T}^d)} \|\xi\|_{H^1(\mathbb{T}^d)}.
    \end{align*}
    This proves (1). Applying the density lower bound and the Poincar\'e-Wirtinger inequality \citep{cioranescu1999introduction} yields,
    \begin{align*}
        |B(\psi, \psi)| \geq \lambda \|\nabla \psi\|_{L^2(\mathbb{T}^d)}^2 \geq C_\text{p}^{-2}\lambda\|\psi\|_{L^2(\mathbb{T}^d)}^2.
    \end{align*}
    Note that this step used the fact that $\psi$ necessarily has mean zero. This bound implies,
    \begin{align*}
        \|\psi\|_{H^1(\mathbb{T}^d)}^2 = \|\psi\|_{L^2(\mathbb{T}^d)}^2 + \|\nabla \psi\|_{L^2(\mathbb{T}^d)}^2 \leq \left(1 + C_{\text{p}}^2\right)\|\nabla \psi\|_{L^2(\mathbb{T}^d)}^2
    \end{align*}
    which further implies 
    \[|B(\psi, \psi)| \geq\frac{\lambda}{(1 + C_\text{p}^2)}\|\psi\|_{H^1(\mathbb{T}^d)}^2\]
    proving (2). For (3), observe that $f(\cdot) = \langle z, \nabla (\cdot)\rangle_{\mathbb{T}^d}$ is a linear functional. Applying the density bound $\rho \leq \Lambda$ again yields
    \begin{align*}
        |f(\xi)| \leq \Lambda\|z\|_{L^2(\mathbb{T}^d)}\|\nabla \xi\|_{L^2(\mathbb{T}^d)} \leq \Lambda\|z\|_{L^2(\mathbb{T}^d)}\| \xi\|_{H^1(\mathbb{T}^d)}.
    \end{align*}
    Thus, $f$ is a bounded linear functional. Now we can apply the Lax-Milgram theorem (\Cref{lax-milgram}) to obtain the unique $\psi \in H^1_\diamond(\mathbb{T}^d)$ such that 
    \[\int_{\mathbb{T}^d}\rho \, \langle \nabla \psi, \nabla \xi\rangle \, dx = \int_{\mathbb{T}^d}\rho \, \langle z, \nabla \xi\rangle\, dx\]
    for all $\xi \in H^1_\diamond(\mathbb{T}^d).$ Note that $\nabla \xi$ is unchanged when $\xi$ is changed by an additive constant, so this statement can be upgraded to all $\xi \in H^1(\mathbb{T}^d).$ By \Cref{tangent space realization}, we know that $\nabla \psi \in T_\mu\mathcal{P}_2(\mathbb{T}^d).$ To finish the proof, we now need to show that $\nabla \psi$ is indeed the orthogonal projection of $z \in L^2(\mu; \mathbb{R}^d)$ onto $T_\mu\mathcal{P}_2(\mathbb{T}^d).$ Let $\varphi \in C^\infty(\mathbb{T}^d)$. Since $\varphi$ is \textit{also} in $H^1(\mathbb{T}^d)$, we can apply the variational identity above to say
    \[\int_{\mathbb{T}^d}\rho \, \langle z - \nabla \psi, \nabla \varphi\rangle \, dx = 0.\]
    This is equivalent to saying $\langle z - \nabla \psi, \nabla \varphi\rangle_{L^2(\mu)} = 0$ for all $\varphi \in C^\infty(\mathbb{T}^d)$, which clearly implies that $z - \nabla \psi$ is orthogonal to gradients $C^\infty(\mathbb{T}^d)$. This also renders $z - \nabla \psi$ orthogonal to the $L^2(\mu)$ closure of the gradients, $T_\mu\mathcal{P}_2(\mathbb{T}^d).$ Thus, $\Pi_\mu z = \nabla \psi$, which proves the claim. 
\end{proof}

\subsection{The Stability Theorem} \label{sec: stability theory}

In this section we state and prove our stability theorem (\Cref{wpt stability}), which we leverage to obtain error bounds for \Cref{alg: W trajectory reconstruction}. We collect all supplementary results for this stability theorem in \Cref{sec: stability thm supplemental results}.

\begin{restatable}[Stability of Wasserstein parallel transport]{thm}{} \label{wpt stability}
Let $(\nu,\mu,\mu')$ be a triple in the admissible class $\mathcal C$, and let $v \in T_\nu\mathcal{P}_2(\mathbb{T}^d)$. Suppose that the Wasserstein geodesics $(\mu_t)_{t\in[0,1]}$, $ (\mu_t')_{t\in[0,1]}$ from $\nu$ to $\mu$ and from $\nu$ to $\mu'$ admit Lebesgue densities $\rho_t,\rho_t'$ and have tangent velocity fields $\phi_t,\phi_t'$, such that:

\begin{enumerate}
    \item \textbf{Uniform density bounds.} There exist constants $0<\lambda\le \Lambda<\infty$ such that for all $t\in[0,1]$,
    \[
    \lambda \le \rho_t(x),\,\rho_t'(x)\le \Lambda
    \qquad\text{for a.e. }x\in\mathbb{T}^d.
    \]
    
    \item \textbf{$H^1$-propagation of parallel fields.} There exists $C_{\rm par}>0$ such that for all $t\in[0,1]$,
    \[
    \|w_t\|_{H^1(\nu)}+\|w_t'\|_{H^1(\nu)}
    \le C_{\rm par}\|v\|_{H^1(\nu)}
    \]
    where $w_t = {\PT}_{\nu \rightarrow \mu_t}(v)$ and $w_t' = {\PT}_{\nu \rightarrow \mu_t'}(v)$.

    \item \textbf{Lipschitz stability of geodesic velocities.} There exists $C_{\rm vel}>0$ such that for all $t\in[0,1]$,
    \[
    \|\nabla\phi_t-\nabla\phi_t'\|_{L^\infty(\mathbb{T}^d)}
    \le C_{\rm vel}\,W_2(\mu,\mu').
    \]

    \item \textbf{Lipschitz stability of densities.} There exist constants $C_{L}>0, C_W > 0$ such that for all $t\in[0,1]$,
    \[
    \|\rho_t-\rho_t'\|_{L^\infty(\mathbb{T}^d)}
    \le C_{L}\,W_2(\mu,\mu'), \quad \text{and} \quad \|\rho_t-\rho_t'\|_{W^{1,\infty}(\mathbb{T}^d)}
    \le C_{W}\,W_2(\mu,\mu').
    \]

    \item \textbf{Lipschitz stability of the weighted Helmholtz projection.} There exist constants
    $C_{\Pi,0}, C_{\Pi,1}>0$ such that for all $t\in[0,1]$ and every
    $z\in L^2(\nu;\mathbb R^d)$,
    \begin{align*}
    \|\Pi_{\mu_t}(z)-\Pi_{\mu_t'}(z)\|_{L^2(\nu)}
    &\le
    C_{\Pi,0}\,\|\rho_t-\rho_t'\|_{L^\infty(\mathbb{T}^d)}\,\|z\|_{L^2(\nu)}.
    \end{align*}
    Moreover, for all $w \in H^1(\nu; \mathbb{R}^d)$
    \begin{align*}
        \|\Pi_{\mu_t}(w)-\Pi_{\mu_t'}(w)\|_{H^1(\nu)}
    &\le
    C_{\Pi,1}\,\|\rho_t-\rho_t'\|_{W^{1,\infty}(\mathbb{T}^d)}\,\|w\|_{H^1(\nu)}.
    \end{align*}
    \item \textbf{Uniform $W^{1,\infty}$ control of geodesic velocities.} There exists $M<\infty$ such that
    \[
    \sup_{t\in[0,1]}\Big(
    \|\nabla\phi_t\|_{W^{1,\infty}(\mathbb{T}^d)}
    +
    \|\nabla\phi_t'\|_{W^{1,\infty}(\mathbb{T}^d)}
    \Big)\le M.
    \]
    \item \textbf{Time regularity of the weighted Helmholtz projection.}
    For every absolutely continuous curve $z \in \text{AC}([0,1]; L^2(\mathbb{T}^d; \mathbb{R}^d))$, the curve $t \mapsto \Pi_{\mu_t}z_t$ also belongs to $\text{AC}([0,1]; L^2(\nu; \mathbb{R}^d)$. Moreover, for a.e. $t \in (0,1)$ we define
    \[(\partial_t\Pi_{\mu_t})(z_t) \triangleq \partial_t(\Pi_{\mu_t}z_t) - \Pi_{\mu_t}(\partial_tz_t)\]
    and this quantity satisfies the bound 
    \[
    \big\|(\partial_t\Pi_{\mu_t} )(z_t)\big\|_{L^2(\nu)}
    \le C_{\dot\Pi}\|z_t\|_{L^2(\nu)}.
    \]
\end{enumerate}
Then there exists a constant $C_{\rm WPT}>0$, depending only on the constants above, such that for every $v\in T_\nu\mathcal P_2(\mathbb{T}^d)\cap H^1(\nu;\mathbb R^d)$,
\[
\|\PT_{\nu\to\mu}(v)-\PT_{\nu\to\mu'}(v)\|_{L^2(\nu)}
\le
C_{\rm WPT}\,\|v\|_{H^1(\nu)}\,W_2(\mu,\mu').
\]
By norm compatibility, the same bound holds in $L^2(\mu)$ up to a multiplicative constant depending only on $\lambda,\Lambda$.
\end{restatable}
\noindent \textit{Proof.} Note that the uniform density bounds ensure that the normed spaces induced by each measure and any Wasserstein interpolants thereof are uniformly norm-equivalent. We will leverage this property extensively in this proof. Define the difference field $e_t = w_t - w_t'$, which is well defined pointwise on $\mathbb{T}^d$ despite the fact that $w_t$ and $w_t'$ live in different Wasserstein tangent spaces. By \Cref{Wasserstein parallel transport pde}, 
\begin{align*}
\Pi_{\mu_t}(\partial_t e_t) &= \Pi_{\mu_t}(\partial_t w_t) - \Pi_{\mu_t}(\partial_t w_t') \\
    &= -\Pi_{\mu_t}(\nabla w_t \nabla \phi_t) - \Pi_{\mu_t}(\partial_t w_t') \\
    &= -\Pi_{\mu_t}(\nabla w_t \nabla \phi_t) + \Pi_{\mu_t'}(\nabla w_t' \nabla \phi_t') + (\Pi_{\mu_t'} - \Pi_{\mu_t})(\partial_t w_t').
\end{align*}
Now adding and subtracting $\Pi_{\mu_t}(\nabla w_t'\nabla\phi_t)$ gives
\begin{equation}
\Pi_{\mu_t}(\partial_t e_t) = -\Pi_{\mu_t}(\nabla e_t \nabla \phi_t) - \Pi_{\mu_t}(\nabla w_t' \nabla(\phi_t - \phi_t'))
- (\Pi_{\mu_t} - \Pi_{\mu_t'})(\partial_tw_t' + \nabla w_t'\nabla \phi_t'). \label{eq: projected derivative}
\end{equation}
We will return to this expression shortly. First, define $\eta_t = \Pi_{\mu_t}e_t = w_t - \Pi_{\mu_t}(w_t')$ and $r_t = (I - \Pi_{\mu_t})(e_t)$, and observe that $e_t = \eta_t + r_t$. We will first control $r_t$,
\begin{align*}
    r_t &= (I - \Pi_{\mu_t})(w_t - w_t') \\
        &= (w_t - w_t') - \Pi_{\mu_t}(w_t - w_t') \\
        &= \Pi_{\mu_t}w_t - \Pi_{\mu_t'}w_t' - \Pi_{\mu_t}w_t + \Pi_{\mu_t}w_t' \\
        &= (\Pi_{\mu_t} - \Pi_{\mu_t'})(w_t')
\end{align*}
since $w_t = \Pi_{\mu_t}w_t$ and $w_t' = \Pi_{\mu_t'}w_t'$. By uniform norm equivalence of all measures of interest and assumptions 2, 4 and 5, \begin{align*}\|r_t\|_{L^2(\mu_t)} &\lesssim   C_{\Pi, 0}\|\rho_t - \rho_t'\|_{L^\infty(\mathbb{T}^d)}\|w_t'\|_{L^2(\nu)}\\
&\leq C_{\Pi,0} C_L C_{\text{par}}\|v\|_{H^1(\nu)}W_2(\mu, \mu').\end{align*}
Moreover, the second part of assumption 5 implies 
\begin{equation}
    \|r_t\|_{H^1(\mu_t)} \lesssim  C_{\Pi, 1}C_W C_{\text{par}}\|v\|_{H^1(\nu)}W_2(\mu, \mu').
\end{equation}
Since $\eta_t = \Pi_{\mu_t}e_t$, assumption 7 gives
\[\partial_t\eta_t = (\partial_t\Pi_{\mu_t})(e_t) + \Pi_{\mu_t}(\partial_te_t).\]
Projecting onto the tangent space at $\mu_t$ again yields
\[\Pi_{\mu_t}(\partial_t\eta_t) = \Pi_{\mu_t}((\partial_t\Pi_{\mu_t})(e_t)) + \Pi_{\mu_t}(\partial_te_t).\]
Plugging in \Cref{eq: projected derivative} and using the fact that $e_t = \eta_t + r_t$ yields
\begin{equation}
\Pi_{\mu_t}(\partial_t\eta_t) = -\Pi_{\mu_t}(\nabla \eta_t \nabla \phi_t)  + g_t \label{F_t}
\end{equation}
where
\begin{equation}
g_t = -\Pi_{\mu_t}(\nabla r_t \nabla \phi_t)  - \Pi_{\mu_t}(\nabla w_t' \nabla(\phi_t - \phi_t')) -(\Pi_{\mu_t} - \Pi_{\mu_t'})(\partial_tw_t'  + \nabla w_t'\nabla \phi_t') + \Pi_{\mu_t}((\partial_t\Pi_{\mu_t})(e_t)).
\end{equation}
Now we will show that 
\begin{equation*}
    \frac{1}{2}\frac{d}{dt}\|\eta_t\|_{L^2(\mu_t)}^2 = \langle \eta_t, g_t\rangle_{L^2(\mu_t)}.
\end{equation*}
We have
\begin{align*}
    \|\eta_t\|_{L^2(\mu_t)}^2 = \int_{\mathbb{T}^d}\|\eta_t(x)\|^2_2d\rho_t(x).
\end{align*}
Taking the time derivative yields
\begin{align*}
    \frac{1}{2}\frac{d}{dt}\|\eta_t\|^2_{L^2(\mu_t)} = \int_{\mathbb{T}^d} \langle \eta_t(x), \frac{d}{dt}\eta_t(x)\rangle \,d\rho_t(x) + \frac{1}{2}\int_{\mathbb{T}^d}\|\eta_t\|^2_2 (\partial_t\rho_t(x)) \,dx
\end{align*}
By the continuity equation along the geodesic $(\mu_t)$, we know that $\partial_t\rho_t = -\nabla \cdot(\rho_t\nabla\phi_t)$ and thus
\begin{align*}
    \frac{1}{2}\frac{d}{dt}\|\eta_t\|^2_{L^2(\mu_t)} = \int_{\mathbb{T}^d} \langle \eta_t(x), \frac{d}{dt}\eta_t(x)\rangle \,d\rho_t(x) - \frac{1}{2}\int_{\mathbb{T}^d} \|\eta_t\|^2_2 \nabla \cdot(\rho_t(x)\nabla\phi_t(x)) \,dx.
\end{align*}
Applying integration by parts to the second term indicates that 
\[\frac{1}{2}\int_{\mathbb{T}^d}\|\eta_t\|^2_2 \nabla \cdot(\rho_t(x)\nabla\phi_t(x)) \,dx = \int_{\mathbb{T}^d}\langle\eta_t, \nabla \eta_t\nabla\phi_t(x) \rangle  \, d\rho_t(x).\]
Note that no boundary term appears in the integration by parts formula as $\mathbb{T}^d$ is boundaryless. Plugging back in, we see that 
\[ \frac{1}{2}\frac{d}{dt}\|\eta_t\|_{L^2(\mu_t)}^2 = \langle \eta_t, \partial_t\eta_t\rangle_{L^2(\mu_t)} + 
\langle \eta_t, \nabla \eta_t \nabla \phi_t \rangle_{L^2(\mu_t)} = \langle \eta_t, \partial_t\eta_t +  \nabla \eta_t \nabla \phi_t \rangle_{L^2(\mu_t)}.\]
Now let $a_t = \partial_t\eta_t + \nabla\eta_t\nabla\phi_t$, and recall from \Cref{F_t} that $\Pi_{\mu_t}(a_t) = g_t$. Since $\eta_t \in T_{\mu_t}\mathcal{P}_2(\mathbb{T}^d)$, we know that for any $b \in L^2(\mu_t; \mathbb{R}^d)$, $\langle \eta_t, b_t\rangle_{L^2(\mu_t)} = \langle \eta_t, \Pi_{\mu_t}b_t\rangle_{L^2(\mu_t)}$. Thus, we indeed see that 
\[ \frac{1}{2}\frac{d}{dt}\|\eta_t\|_{L^2(\mu_t)}^2 = \langle \eta_t, g_t\rangle_{L^2(\mu_t)}\]
Now we will bound the four contributions to $g_t$. For the first term of $g_t$, using norm equivalence, assumption 6 and the Sobolev norm bound on $r_t$ obtained earlier yields
\begin{align*}
    \|\Pi_{\mu_t}(\nabla r_t\nabla \phi_t)\|_{L^2(\mu_t)} &\lesssim  \|\nabla r_t \nabla \phi_t\|_{L^2(\nu)} \\
    &\leq \|\nabla \phi_t\|_{W^{1, \infty}(\mathbb{T}^d)}\|r_t\|_{H^1(\nu)} \\
    &\leq MC_{\Pi,1}C_W C_{\rm par}\|v\|_{H^1(\nu)}W_2(\mu,\mu').
\end{align*}
For the second term, due to assumptions 2 and 3, we have
\begin{align*}
    \|\Pi_{\mu_t}(\nabla w_t' \nabla(\phi_t - \phi_t'))\|_{L^2(\mu_t)} \lesssim  \|\nabla w_t'\|_{L^2(\nu)}\|\nabla\phi_t - \nabla\phi_t'\|_{L^{\infty}(\mathbb{T}^d)} \leq C_\text{par}C_{\text{vel}}\|v\|_{H^1(\nu)}W_2(\mu, \mu').
\end{align*}
For the third term, applying assumption 5 gives
\begin{align*}
    \left\|(\Pi_{\mu_t} - \Pi_{\mu_t'})(\partial_tw_t'  + \nabla w_t'\nabla \phi_t')\right\|_{L^2(\mu_t)} \lesssim  C_{\Pi,0}\|\rho_t - \rho_t'\|_{L^{\infty}(\mathbb{T}^d)}\|\partial_tw_t' + \nabla w_t'\nabla \phi_t'\|_{L^2(\nu)}.
\end{align*}
Now let $h_t' \triangleq \partial_tw_t' + \nabla w_t'\nabla \phi_t'$. Since $w_t' = \Pi_{\mu_t'}w_t'$, applying assumption 7 and \Cref{lem:pt_ac} yields
\[\partial_t w_t' = \Pi_{\mu_t'}(\partial_tw_t') + (\partial_t\Pi_{\mu_t'})(w_t').\]
Since $w_t'$ is assumed to be parallel along $\mu_t'$, $\Pi_{\mu_t'}(\partial_t w_t') = -\Pi_{\mu_t'}(\nabla w_t' \nabla \phi_t')$. Thus,
\[h_t' = (\partial_t\Pi_{\mu_t'})(w_t') +(\text{id} - \Pi_{\mu_t'})(\nabla w_t' \nabla \phi_t').\]
Since $\Pi_{\mu_t'}$ is an orthogonal projection
\begin{align*}
    \|h_t'\|_{L^2(\nu)} &\leq \|(\partial_t\Pi_{\mu_t'})(w_t')\|_{L^2(\nu)} + \|(\text{id} - \Pi_{\mu_t'})(\nabla w_t' \nabla \phi_t')\|_{L^2(\nu)} \\
    &\leq C_{\dot \Pi}\|w_t'\|_{L^2(\nu)} + \|\nabla w_t'\|_{L^2(\nu)}\|\nabla \phi_t'\|_{W^{1, \infty}(\mathbb{T}^d)} \\
    &\leq C_{\dot \Pi}\|w_t'\|_{H^1(\nu)} + M\|w_t'\|_{H^1(\nu)} \\
    &\lesssim \|v\|_{H^1(\nu)}.
\end{align*}
Combining this with assumption 4 yields
\begin{align*}
    \left\|(\Pi_{\mu_t} - \Pi_{\mu_t'})(\partial_tw_t'  + \nabla w_t'\nabla \phi_t')\right\|_{L^2(\mu_t)} \lesssim  \|v\|_{H^1(\nu)}W_2(\mu, \mu').
\end{align*}
For the final term, assumption 7, \Cref{lem:pt_ac} and the bound on the $L^2(\nu)$ norm of $r_t$ implies
\begin{align*}
    \|\Pi_{\mu_t}((\partial_t\Pi_{\mu_t})e_t)\|_{L^2(\mu_t)} \lesssim  C_{\dot \Pi}\|\eta_t + r_t\|_{L^2(\nu)} \lesssim  C_{\dot \Pi}(\|\eta_t\|_{L^2(\mu_t)} + C_{\Pi, 0} C_L C_{\text{par}}\|v\|_{H^1(\nu)}W_2(\mu, \mu')).
\end{align*}
Combining the bounds for all terms and applying Cauchy-Schwartz yields
\begin{equation*}
    \frac{1}{2}\frac{d}{dt}\|\eta_t\|^2_{L^2(\mu_t)} \leq C_1\|\eta_t\|_{L^2(\mu_t)}^2 + C_2 \|v\|_{H^1(\nu)}\|\eta_t\|_{L^2(\mu_t)}  W_2(\mu, \mu')
\end{equation*}
for some $C_1$ and $C_2$ depending only on theorem constants. For all $t$ such that $\eta_t \neq 0$ we can cancel $\|\eta_t\|_{L^2(\mu_t)}$ and say
\[\frac{d}{dt}\|\eta_t\|_{L^2(\mu_t)} \leq C_1\|\eta_t\|_{L^2(\mu_t)} + C_2\|v\|_{H^1(\nu)}W_2(\mu, \mu').\]
Since $\eta_0 \equiv 0$ we can apply Gr\"onwall's inequality to say $\|\eta_t\|_{L^2(\mu_t)} \leq C_3\|v\|_{H^1(\nu)}W_2(\mu, \mu')$ for some $C_3$ dependent only on theorem constants. Finally, we have
\begin{align*}
    \|e_t\|_{L^2(\mu_t)} \leq \|\eta_t\|_{L^2(\mu_t)} + \|r_t\|_{L^2(\mu_t)} \leq C_{\text{WPT}}\|v\|_{H^1(\nu)}W_2(\mu, \mu').
\end{align*}
Since $e_t = w_t - w_t' = {\PT}_{\nu \rightarrow \mu_t}(v) - {\PT}_{\nu \rightarrow \mu_t'}(v)$, and since the spaces under $\mu_t, \mu, \nu$ are all norm-equivalent, the proof is complete. \qed{}

\begin{restatable}{lemma}{}
    Under \Cref{all necessary W assumptions}, the hypotheses of \Cref{wpt stability} are satisfied.
\end{restatable}
\noindent \textit{Proof.} We will go through each assumption and show that it follows from an assumption in \Cref{all necessary W assumptions}. Assumption 1 is stated exactly in \ref{density regularity conditions}. Assumption 2 follows from \Cref{prop:h1} coupled with \ref{density regularity conditions} and \ref{geodesic regularity conditions}. Assumptions 3 and 4 are stated directly in \ref{geodesic lipschitz conditions}. Assumption 5 is guaranteed by \Cref{L2 lipschitz stability} and \Cref{sobolev lipschitz stability}. Assumption 6 is stated directly in \ref{geodesic regularity conditions}, and assumption 7
is guaranteed by \ref{density regularity conditions} and \Cref{lem:time_reg_weighted_projection}.
\qed{}

\subsection{Supplementary Results} \label{sec: stability thm supplemental results}

\begin{restatable}[Lipschitz stability of weighted Helmholtz projection in $L^2$]{lemma}{} \label{L2 lipschitz stability}
    Let
    $\rho,\rho'$ satisfy
    \[
    0<\lambda\le \rho(x),\rho'(x)\le \Lambda<\infty
    \qquad\text{a.e. on }\mathbb{T}^d.
    \]
    Let $\Pi_\rho,\Pi_{\rho'}$ denote the weighted Helmholtz projections onto the
    $L^2(\rho)$- and $L^2(\rho')$-closures of gradient fields, respectively.
    Then there exists a constant $C_{\Pi, 0}=C(\lambda,\Lambda)$ such that for every
    $z\in L^2(\mathbb{T}^d;\mathbb R^d)$,
    \[
    \|\Pi_\rho z-\Pi_{\rho'} z\|_{L^2(\mathbb{T}^d)}
    \le
    C_{\Pi, 0}\,\|\rho-\rho'\|_{L^\infty(\mathbb{T}^d)}\,\|z\|_{L^2(\mathbb{T}^d)}.
    \]
    If, in addition, the reference measure $\nu=\rho_\nu dx$ satisfies
    $0<\lambda_\nu\le \rho_\nu\le \Lambda_\nu<\infty$, then equivalently
    \[
    \|\Pi_\rho z-\Pi_{\rho'} z\|_{L^2(\nu)}
    \le
    C_{\Pi, 0}\,\|\rho-\rho'\|_{L^\infty(\mathbb{T}^d)}\,\|z\|_{L^2(\nu)}.
    \]
\end{restatable}
\noindent \textit{Proof.} Let $u \triangleq \Pi_\rho z$ and let $u' \triangleq \Pi_{\rho'}z$ and write $u = \nabla \psi$, $u' = \nabla \psi'$, and $w = u - u'$ (which can be done due to \Cref{projection realization}). Recall that the norm equivalence of the $L^2(\rho)$ and $L^2(\rho')$ spaces ensure that the image spaces of the projections necessarily coincide. The first order conditions of the two Helmholtz projections are
\begin{align*}
    \int_{\mathbb{T}^d} \langle u, \nabla \xi\rangle\, d\rho &= \int_{\mathbb{T}^d} \langle z, \nabla\xi \rangle \, d\rho  \\
    \int_{\mathbb{T}^d} \langle u', \nabla \xi\rangle\, d\rho' &= \int_{\mathbb{T}^d} \langle z, \nabla\xi \rangle \, d\rho'
\end{align*}
for all $\xi \in H^1(\mathbb{T}^d).$ Subtracting the first order conditions implies that for all $\xi \in H^1(\mathbb{T}^d)$,
\begin{align*}
    \int_{\mathbb{T}^d} \langle u, \nabla \xi\rangle\, d\rho =   \int_{\mathbb{T}^d} \langle u', \nabla \xi\rangle\, d\rho'+\int_{\mathbb{T}^d} \langle z, \nabla\xi \rangle \, d\rho - \int_{\mathbb{T}^d} \langle z, \nabla\xi \rangle \, d\rho'.
\end{align*}
Subtracting $\langle u', \nabla \xi \rangle_{L^2(\rho)}$ from both sides yields
\begin{align*}
    \int_{\mathbb{T}^d} \langle w, \nabla \xi\rangle\, d\rho &= \int_{\mathbb{T}^d} \langle u', \nabla \xi\rangle\, d\rho'+\int_{\mathbb{T}^d} \langle z, \nabla\xi \rangle \, d\rho - \int_{\mathbb{T}^d} \langle z, \nabla\xi \rangle \, d\rho' - \int_{\mathbb{T}^d} \langle u', \nabla \xi\rangle\, d\rho \\
    &=\int_{\mathbb{T}^d} (z - u')(\rho - \rho') \cdot \nabla \xi.
\end{align*}
Now we'll test with $\xi = \psi - \psi' \triangleq \delta$. Since $\nabla \delta = u - u' = w$,
\begin{align*}
    \int_{\mathbb{T}^d}  \|w\|_2^2 \,d\rho = \int_{\mathbb{T}^d}(\rho - \rho')(z-u') \cdot w \, dx.
\end{align*}
Under the assumption that $\rho(x) \geq \lambda$, through application of H\"older's inequality we obtain
\begin{align*}
    \lambda \|w\|_{L^2(\mathbb{T}^d)}^2 \leq \int_{\mathbb{T}^d}(\rho - \rho')(z-u') \cdot w \, dx \leq \|\rho - \rho'\|_{L^\infty(\mathbb{T}^d)}\|z-u'\|_{L^2(\mathbb{T}^d)}\|w\|_{L^2(\mathbb{T}^d)}
\end{align*}
which implies the bound 
\[\|w\|_{L^2(\mathbb{T}^d)} \leq \frac{1}{\lambda}\|\rho - \rho'\|_{L^\infty(\mathbb{T}^d)}\|z-u'\|_{L^2(\mathbb{T}^d)}.\]
Since $u' = \Pi_{\rho'}z$ is an orthogonal projection in $L^2(\rho')$, $\|u'\|_{L^2(\rho')} \leq \|z\|_{L^2(\rho')}.$ Thus
\[\|u'\|_{L^2(\mathbb{T}^d)} \leq \lambda^{-1/2}\|u'\|_{L^2(\rho')} \leq \lambda^{-1/2}\|z\|_{L^2(\rho')} = \left(\frac{\Lambda}{\lambda}\right)^{1/2}\|z\|_{L^2(\mathbb{T}^d)}.\]
Thus, we recover the stability bound
\[\|\Pi_\rho z-\Pi_{\rho'} z\|_{L^2(\mathbb{T}^d)}
    \le
    \frac{1 + \sqrt{\Lambda/\lambda}}{\lambda}\,\|\rho-\rho'\|_{L^\infty(\mathbb{T}^d)}\,\|z\|_{L^2(\mathbb{T}^d)}.\]
\qed{}

\begin{restatable}[Lipschitz stability of weighted Helmholtz projection in $H^1$]{lemma}{} \label{sobolev lipschitz stability}
Assume that
\[
\rho,\rho' \in W^{1,\infty}(\mathbb{T}^d),\qquad
0<\lambda\le \rho,\rho' \le \Lambda.
\]
Assume also that the densities are uniformly bounded in a Sobolev sense,
\[\|\rho\|_{W^{1, \infty}(\mathbb{T}^d)} + \|\rho'\|_{W^{1, \infty}(\mathbb{T}^d)} \leq K.\]
Then there exists $C_{\Pi, 1}>0$ depending only on the regularity constants such that
for every $z\in H^1(\mathbb{T}^d;\mathbb R^d)$,
\[
\|\Pi_\rho z-\Pi_{\rho'} z\|_{H^1(\mathbb{T}^d)}
\le
C_{\Pi, 1}\,\|\rho-\rho'\|_{W^{1,\infty}(\mathbb{T}^d)}\,\|z\|_{H^1(\mathbb{T}^d)}.
\]
\end{restatable}
\noindent \textit{Proof.} As derived in \Cref{L2 lipschitz stability}, the first order optimality conditions of the projections give rise to the equation
\begin{align}
    \nabla \cdot (\rho w) = \nabla \cdot ((\rho - \rho')(z-u')) \label{eq: focs}
\end{align}
where $w = u - u'$, $u = \Pi_\rho z$ and $u' = \Pi_{\rho'}z$. If we identify $u = \nabla\psi$ and $u'=\nabla \psi'$ (as guaranteed by \Cref{projection realization}), then we can define $\delta \triangleq \psi - \psi'$. By the arguments in \Cref{L2 lipschitz stability}, one can show
\[\|w\|_{L^2(\mathbb{T}^d)} \leq C\|\rho - \rho'\|_{L^\infty(\mathbb{T}^d)}\|z - u'\|_{L^2(\mathbb{T}^d)}\]
and ultimately one can bound the right hand side by $C\|\rho - \rho'\|_{L^\infty(\mathbb{T}^d)}\|z\|_{L^2(\mathbb{T}^d)}.$ To obtain the desired Sobolev stability bound, we can differentiate $w$ spatially and bound the operator norm of its gradient. We can do this by differentiating the first order conditions in \cref{eq: focs} -- differentiating the left-hand side with respect to the $k$-th coordinate yields
\[ \partial_k (\nabla \cdot(\rho \nabla\delta)) = \nabla \cdot(\partial_k(\rho\nabla \delta))\]
since $w = \nabla \delta$. Applying the product rule yields
\[ \nabla \cdot(\partial_k(\rho\nabla \delta)) = \nabla \cdot(\partial_k\rho\nabla \delta + \rho \nabla (\partial_k\delta)).\]
Differentiating the right-hand side of \Cref{eq: focs} using $G = (\rho - \rho')(z - u')$ yields
\[\partial_k(\nabla\cdot ((\rho - \rho')(z - u')) = \nabla \cdot (\partial_kG).\]
Combining, and rearranging slightly results in,
\[-\nabla \cdot( \rho\nabla (\partial_k\delta)) = \nabla \cdot (\partial_k\rho\nabla \delta) - \nabla \cdot (\partial_kG)\]
and when we test with a test function $\xi \in H^1(\mathbb{T}^d)$, we integrate the above against $\xi$. Ultimately, we want to apply the test function $\xi = \partial_k \delta$, but to do so we first need to show $\partial_k \delta \in H^1(\mathbb{T}^d)$ or, equivalently, $\psi, \psi' \in H^2(\mathbb{T}^d)$. To see that $\psi, \psi'$ are in $H^2(\mathbb{T}^d)$, we can apply the product rule for divergences to rewrite the first order conditions for $\psi$ as
\[-\rho \Delta\psi - \langle \nabla \psi, \nabla \rho\rangle = - \rho \nabla \cdot z - \langle z, \nabla \rho\rangle\]
which, after rearranging, yields
\[\Delta\psi = \nabla \cdot z + \rho^{-1}\langle \nabla \rho, z - \nabla \psi\rangle.\]
Note that $\nabla \cdot z \in L^2(\mathbb{T}^d)$ because $z \in H^1(\mathbb{T}^d)$, and $\rho^{-1}\langle \nabla \rho, z - \nabla \psi\rangle \in L^2(\mathbb{T}^d)$ because $\rho \in W^{1, \infty}(\mathbb{T}^d)$, $\rho \geq \lambda > 0$ and both $z, \nabla \psi \in L^2(\mathbb{T}^d)$. Thus, we can conclude that $\Delta \psi \in L^2(\mathbb{T}^d)$ as well. By Theorem 15.1 of \citet{dyatlov2022ellipticnotes}, $\Delta \psi \in L^2(\mathbb{T}^d)$
implies that $\psi \in H^2(\mathbb{T}^d)$. One can apply the same argument to conclude that $\psi' \in H^2(\mathbb{T}^d)$ as well. Having established this, we can now take $\xi = \partial_k\delta \in L^2(\mathbb{T}^d)$ as our test function, yielding
\[\int_{\mathbb{T}^d} \|\nabla (\partial_k \delta)\|_2^2 \, d\rho = \int_{\mathbb{T}^d} \langle \nabla (\partial_k\delta), \nabla \delta\rangle \, \partial_k\rho \, dx - \int_{\mathbb{T}^d}\langle \nabla (\partial_k\delta),  \partial_kG\rangle  \, dx.\]
Applying H\"older's inequality and using the fact that $\rho \geq \lambda$ yields the bound
\[\lambda\|\nabla (\partial_k\delta)\|^2_{L^2(\mathbb{T}^d)} \leq \|(\partial_k \rho)\nabla \delta\|_{L^2(\mathbb{T}^d)}\|\nabla (\partial_k\delta)\|_{L^2(\mathbb{T}^d)} + \|\nabla (\partial_k\delta)\|_{L^2(\mathbb{T}^d)}\|\partial_kG\|_{L^2(\mathbb{T}^d)}.\]
Thus,
\[\|\nabla (\partial_k\delta)\|_{L^2(\mathbb{T}^d)} \leq \frac{1}{\lambda}\left(\|(\partial_k \rho)\nabla \delta\|_{L^2(\mathbb{T}^d)} + \|\partial_kG\|_{L^2(\mathbb{T}^d)}\right)\]
and if we sum over $k$ and bound further we obtain
\[\|D^2\delta\|_{L^2(\mathbb{T}^d)} \lesssim  \|\nabla \rho\|_{L^\infty(\mathbb{T}^d)}\|\nabla \delta\|_{L^2(\mathbb{T}^d)} + \|G\|_{H^1(\mathbb{T}^d)}\]
where the left-hand side is the integrated Frobenius norm of the Hessian of $x \mapsto \delta(x)$. Since $w = \nabla \delta$, $\|w\|_{H^1(\mathbb{T}^d)} \lesssim  \|\nabla \delta\|_{L^2(\mathbb{T}^d)} + \|D^2\delta\|_{L^2(\mathbb{T}^d)}$, which implies
\begin{equation}\label{sobolev norm on w intermediate}
\|w\|_{H^1(\mathbb{T}^d)} \leq C\left((1+\|\nabla \rho\|_{L^\infty(\mathbb{T}^d)})\|\nabla \delta\|_{L^2(\mathbb{T}^d)} + \|G\|_{H^1(\mathbb{T}^d)}\right). 
\end{equation}
From \Cref{L2 lipschitz stability}, we know that $\|\nabla \delta\|_{L^2(\mathbb{T}^d)} = \|w\|_{L^2(\mathbb{T}^d)} \lesssim  \|\rho - \rho'\|_{L^\infty(\mathbb{T}^d)}\|z - u'\|_{L^2(\mathbb{T}^d)}$, and $\|G\|_{H^1(\mathbb{T}^d)} \leq C\|\rho - \rho'\|_{W^{1, \infty}(\mathbb{T}^d)}\|z - u'\|_{H^1(\mathbb{T}^d)} \leq C\|\rho - \rho'\|_{W^{1, \infty}(\mathbb{T}^d)}(\|z\|_{H^1(\mathbb{T}^d)} + \|u'\|_{H^1(\mathbb{T}^d)}).$ Thus, we require a bound of the form $\|u'\|_{H^1(\mathbb{T}^d)} \lesssim \|z\|_{H^1(\mathbb{T}^d)}.$ We can obtain this bound by differentiating the first order conditions of the projection of $z$,
\[\nabla \cdot (\rho' \nabla \psi') = \nabla \cdot (\rho 'z)\]
yielding
\[\nabla \cdot (\partial_k\rho' \nabla \psi' + \rho'\partial_k\nabla \psi') = \nabla \cdot (\partial_k\rho' z + \rho'\partial_kz).\]
Applying the test function $\partial_k\psi'$ and moving around terms yields
\begin{align*}
    \int_{\mathbb{T}^d} \|\nabla (\partial_k\psi')\|_2^2
    \,d\rho' &= - \int_{\mathbb{T}^d} \partial_k\rho' \langle \nabla (\partial_k\psi'), \nabla \psi'\rangle \, dx+ \int_{\mathbb{T}^d} (\partial_k\rho')\langle \nabla (\partial_k\psi'), z\rangle \, dx\\
    &\hspace{73mm}+ \int_\mathbb{T}^d\langle \nabla (\partial_k\psi'), \partial_kz\rangle \, d\rho' \\
    &\leq \|\partial_k \rho'\|_{L^\infty(\mathbb{T}^d)}\|\nabla (\partial_k\psi')\|_{L^2(\mathbb{T}^d)}\|\nabla \psi'\|_{L^2(\mathbb{T}^d)}\\
    & \hspace{30mm} + \|\partial_k \rho'\|_{L^\infty(\mathbb{T}^d)}\|\nabla (\partial_k\psi')\|_{L^2(\mathbb{T}^d)}\|z\|_{L^2(\mathbb{T}^d)}\\
    &\hspace{60mm} + \Lambda\|\nabla(\partial_k \psi')\|_{L^2(\mathbb{T}^d)}\|\partial_kz\|_{L^2(\mathbb{T}^d)}.
\end{align*}
where the second line used H\"older's inequality and the bound $\rho \leq \Lambda$. Note that the left hand side can be bounded from below by $\lambda \|\nabla (\partial_k\psi')\|_{L^2(\mathbb{T}^d)}^2$. Thus, after cancellation we obtain
\begin{align*}
    \|\nabla (\partial_k\psi')\|_{L^2(\mathbb{T}^d)} \leq \frac{1}{\lambda}\left(\|\partial_k \rho'\|_{L^\infty(\mathbb{T}^d)}\|\nabla \psi'\|_{L^2(\mathbb{T}^d)} + \|\partial_k \rho'\|_{L^\infty(\mathbb{T}^d)}\|z\|_{L^2(\mathbb{T}^d)} + \Lambda\|\partial_kz\|_{L^2(\mathbb{T}^d)} \right).
\end{align*}
Summing over $k$, we get
\[\|\nabla u'\|_{L^2(\mathbb{T}^d)} \lesssim \|u'\|_{L^2(\mathbb{T}^d)} + \|z\|_{H^1(\mathbb{T}^d)}.\]
Now, using the fact that $\Pi_{\rho'}$ is an orthogonal projection implies that $\|u'\|_{L^2(\mathbb{T}^d)} \lesssim \|z\|_{L^2(\mathbb{T}^d)} \leq \|z\|_{H^1(\mathbb{T}^d)}$ yielding 
\[\|\nabla u'\|_{L^2(\mathbb{T}^d)} \lesssim \|z\|_{H^1(\mathbb{T}^d)}.\]
Combining this with the $L^2$ bound on $u'$ yields
\[\|u'\|_{H^1(\mathbb{T}^d)} \lesssim  \|z\|_{H^1(\mathbb{T}^d).}\]
Plugging this bound back into \Cref{sobolev norm on w intermediate} yields
\begin{align*}
    \|w\|_{H^1(\mathbb{T}^d)} \leq C\left((1+\|\nabla \rho\|_{L^\infty(\mathbb{T}^d)})\|w\|_{L^2(\mathbb{T}^d)} + C\|\rho - \rho'\|_{W^{1, \infty}(\mathbb{T}^d)}\|z\|_{H^1(\mathbb{T}^d)}\right)
\end{align*}
and using $\|w\|_{L^2(\mathbb{T}^d)} \lesssim \|\rho - \rho'\|_{W^{1, \infty}(\mathbb{T}^d)}\|z\|_{H^1(\mathbb{T}^d)}$ (\Cref{L2 lipschitz stability}) gives the final bound
\[\|\Pi_\rho z-\Pi_{\rho'} z\|_{H^1(\mathbb{T}^d)}
\le
C_{\Pi, 1}\,\|\rho-\rho'\|_{W^{1,\infty}(\mathbb{T}^d)}\,\|z\|_{H^1(\mathbb{T}^d)}.\]
for some $C_{\Pi, 1} > 0$ dependent only on regularity constants. \qed{}

\begin{restatable}[Time regularity of the weighted Helmholtz projection]{lemma}{}
\label{lem:time_reg_weighted_projection}
Let $\mu_t$ admit a Lebesgue density $\rho_t$ for $t \in [0, 1]$ such that:
\begin{enumerate}
    \item there exist constants
    $0<\lambda\le \Lambda<\infty$ such that for all $t\in[0,1]$,
    \[
    \lambda \le \rho_t(x)\le \Lambda
    \qquad\text{for a.e. }x\in\mathbb{T}^d.
    \]
    \item the map $t \mapsto \rho_t$ is absolutely continuous as a map from $[0,1]$ into $L^\infty(\mathbb{T}^d)$ and there exists
    $K_\rho<\infty$ such that,
    \[
    \sup_{t\in[0,1]}\|\partial_t \rho_t\|_{L^\infty(\mathbb{T}^d)}\le K_\rho.
    \]
\end{enumerate}
Then there exists a constant $C_{\dot \Pi}>0$, depending only on
$\lambda,\Lambda,$ and $K_\rho$, such that the following hold.
\begin{enumerate}
    \item For every $z\in L^2(\mathbb{T}^d;\mathbb R^d)$ and all
    $s,t\in[0,1]$,
    \[
    \|\Pi_{\mu_t}z-\Pi_{\mu_s}z\|_{L^2(\mathbb{T}^d)}
    \le
    C_{\dot\Pi}|t-s|\,\|z\|_{L^2(\mathbb{T}^d)}.
    \]
    In particular,
    \[
    \|\Pi_{\mu_t}-\Pi_{\mu_s}\|_{\mathcal L(L^2(\mathbb{T}^d;\mathbb R^d))}
    \le
    C_{\dot\Pi}|t-s|
    \]
    where the norm is the \textit{operator} norm on bounded linear maps $L^2(\mathbb{T}^d; \mathbb{R}^d) \rightarrow L^2(\mathbb{T}^d; \mathbb{R}^d).$
    \item For every absolutely continuous curve
    $z_\cdot \in AC([0,1];L^2(\mathbb{T}^d;\mathbb R^d))$, the curve
    $t\mapsto \Pi_{\mu_t}z_t$ belongs to
    $AC([0,1];L^2(\mathbb{T}^d;\mathbb R^d))$.
    Moreover, for a.e. $t\in[0,1]$, define
    \[
    (\partial_t\Pi_t)z_t
    \;\triangleq\;
    \partial_t(\Pi_{\mu_t}z_t)-\Pi_{\mu_t}(\partial_t z_t).
    \]
    Then
    \[
    \|(\partial_t\Pi_t)z_t\|_{L^2(\mathbb{T}^d)}
    \le
    C_{\dot\Pi}\|z_t\|_{L^2(\mathbb{T}^d)}
    \qquad\text{for a.e. }t\in[0,1].
    \]
\end{enumerate}
By norm equivalence under the density bounds, the same conclusions hold in
$L^2(\nu)$ up to multiplicative constants depending only on $\lambda,\Lambda$.
\end{restatable}

\begin{proof}
We will start with the Lipschitz bound on the map $t \mapsto \Pi_{\mu_t}z$. Fix $z\in L^2(\mathbb{T}^d;\mathbb R^d)$ and define $u_t \triangleq \Pi_{\mu_t}z$. By first order optimality conditions, $u_t=\nabla \psi_t$ and $\psi_t$ satisfy
\[
\nabla\cdot(\rho_t u_t) = \nabla\cdot(\rho_t z).
\]
Equivalently, for every $\varphi\in H^1(\mathbb{T}^d)$,
\[
\int_{\mathbb{T}^d} \rho_t \langle u_t,\nabla\varphi\rangle\,dx
=
\int_{\mathbb{T}^d} \rho_t \langle z,\nabla\varphi\rangle\,dx.
\]
We first demonstrate the standard $L^2$ bound on the projection. Since $u_t$ is the
$L^2(\mu_t)$-orthogonal projection of $z$ onto the closed subspace
$T_{\mu_t}\mathcal P_2(\mathbb{T}^d)$, we have
\[
\|u_t\|_{L^2(\mu_t)}\le \|z\|_{L^2(\mu_t)}.
\]
Using the density bounds,
\[
\lambda \|u_t\|_{L^2(\mathbb{T}^d)}^2
\le
\|u_t\|_{L^2(\mu_t)}^2
\le
\|z\|_{L^2(\mu_t)}^2
\le
\Lambda \|z\|_{L^2(\mathbb{T}^d)}^2,
\]
hence
\[
\|u_t\|_{L^2(\mathbb{T}^d)}
\le
\sqrt{\frac{\Lambda}{\lambda}}\,\|z\|_{L^2(\mathbb{T}^d)}.
\]
We now prove that the map $t \mapsto \Pi_{\mu_t}z$ is Lipschitz in $L^2(\mathbb{T}^d; \mathbb{R}^d).$ Let $h > 0$ be such that $[t, t+h] \subseteq [0,1]$. Applying the first order conditions at times $t$ and $t+h$ and subtracting gives,
\[\int_{\mathbb{T}^d}\rho_{t+h} \, \langle u_{t+h} - u_t, \nabla \varphi\rangle \, dx = \int_{\mathbb{T}^d}(\rho_{t+h}-\rho_t)\langle z-u_t, \nabla \varphi \rangle\, dx\]
for all $\varphi \in H^1(\mathbb{T}^d).$ Since $u_s=\nabla\psi_s$ for each $s$, define
\[
\eta_h:=\frac{\psi_{t+h}-\psi_t}{h}\in H^1(\mathbb{T}^d),
\qquad
\nabla\eta_h=\frac{u_{t+h}-u_t}{h}.
\]
Taking $\varphi=\eta_h$ in the preceding identity and dividing by $h$ yields
\[
\int_{\mathbb{T}^d}\rho_{t+h}\left\|\frac{u_{t+h}-u_t}{h}\right\|_2^2\,dx
=
\int_{\mathbb{T}^d}\frac{\rho_{t+h}-\rho_t}{h}
\left\langle z-u_t,\frac{u_{t+h}-u_t}{h}\right\rangle\,dx.
\]
Using the lower density bound and applying H\"older's inequality,
\[
\lambda\left\|\frac{u_{t+h}-u_t}{h}\right\|_{L^2(\mathbb{T}^d)}^2
\le
\left\|\frac{\rho_{t+h}-\rho_t}{h}\right\|_{L^\infty(\mathbb{T}^d)}
\|z-u_t\|_{L^2(\mathbb{T}^d)}
\left\|\frac{u_{t+h}-u_t}{h}\right\|_{L^2(\mathbb{T}^d)}.
\]
Therefore,
\[
\left\|\frac{u_{t+h}-u_t}{h}\right\|_{L^2(\mathbb{T}^d)}
\le
\frac{1}{\lambda}
\left\|\frac{\rho_{t+h}-\rho_t}{h}\right\|_{L^\infty(\mathbb{T}^d)}
\|z-u_t\|_{L^2(\mathbb{T}^d)}.
\]
Since $\sup_{t\in [0,1]}\|\partial_t\rho_t\|_{L^\infty(\mathbb{T}^d)} \leq K_\rho$ and $t \mapsto \rho_t$ is absolutely continuous as an $L^\infty(\mathbb{T}^d)$ valued map, we have
\[\|\rho_{t+h} - \rho_t\|_{L^\infty(\mathbb{T}^d)} \leq K_\rho h\]
and thus,
\[
\left\|\frac{u_{t+h}-u_t}{h}\right\|_{L^2(\mathbb{T}^d)}
\le
\frac{K_\rho}{\lambda}\|z-u_t\|_{L^2(\mathbb{T}^d)}.
\]
Applying Minkowski's inequality and the bound on $\|u_t\|_{L^2(\mathbb{T}^d)}$ yields
\[
\|z-u_t\|_{L^2(\mathbb{T}^d)}
\le
\|z\|_{L^2(\mathbb{T}^d)}+\|u_t\|_{L^2(\mathbb{T}^d)}
\le
\left(1+\sqrt{\frac{\Lambda}{\lambda}}\right)\|z\|_{L^2(\mathbb{T}^d)}.
\]
Combining and multiplying through by $h$ therefore results in 
\[
\|u_{t+h}-u_t\|_{L^2(\mathbb{T}^d)}
\le
\frac{K_\rho}{\lambda}
\left(1+\sqrt{\frac{\Lambda}{\lambda}}\right)|h|\,\|z\|_{L^2(\mathbb{T}^d)}.
\]
Thus, statement (1) holds with \[C_{\dot \Pi} = \frac{K_\rho}{\lambda}
\left(1+\sqrt{\frac{\Lambda}{\lambda}}\right).\]
Therefore the map $t \mapsto \Pi_{\mu_t}z$ is Lipschitz as a map from $[0,1]$ to $L^2(\mathbb{T}^d; \mathbb{R}^d)$. Taking the supremum over all $z$ such that $\|z\|_{L^2(\mathbb{T}^d)} = 1$ yields the operator norm bound 
\[
\|\Pi_{\mu_t} - \Pi_{\mu_s}\|_{\mathcal{L}(L^2(\mathbb{T}^d))}
\le
C_{\dot \Pi}|t - s|.
\]
Now let $z_{(\cdot)} \in \text{AC}([0,1]; L^2(\mathbb{T}^d; \mathbb{R}^d))$ and define $y_t \triangleq \Pi_{\mu_t}(z_t).$ Since $z_{(\cdot)}$ is absolutely continuous, $z_t - z_s= \int_{s}^t \partial_r z_r \, dr$. Note that
\[y_t - y_s = (\Pi_{\mu_t} - \Pi_{\mu_s})(z_s) + \Pi_{\mu_t}(z_t-z_s)\]
implies the bound
\[
\|y_t-y_s\|_{L^2(\mathbb{T}^d)}
\le
C_{\dot\Pi}|t-s|\,\|z_s\|_{L^2(\mathbb{T}^d)}
+
\|\Pi_{\mu_t}\|_{\mathcal L(L^2(\mathbb{T}^d))}\int_s^t\|\partial_r z_r\|_{L^2(\mathbb{T}^d)}\,dr.
\]
Since
\[
\|\Pi_{\mu_t}\|_{\mathcal L(L^2(\mathbb{T}^d;\mathbb R^d))}
\le \sqrt{\frac{\Lambda}{\lambda}}
\]
and $z_{(\cdot)}$ is continuous on $[0,1]$, the right-hand side is controlled by
an $L^1$ function of $r$. Therefore $y_{(\cdot)}\in AC([0,1];L^2(\mathbb{T}^d;\mathbb R^d))$. Next define
\[
w_t
\triangleq
y_t-\int_0^t \Pi_{\mu_r}(\partial_r z_r)\,dr.
\]
Since both $y_{(\cdot)}$ and the integral term are absolutely continuous,
$w_{(\cdot)}$ is absolutely continuous as well. Moreover, for $0\le s<t\le 1$,
\begin{align*}
w_t-w_s
&=
\Pi_{\mu_t}z_t-\Pi_{\mu_s}z_s-\int_s^t \Pi_{\mu_r}(\partial_r z_r)\,dr \\
&=
(\Pi_{\mu_t}-\Pi_{\mu_s})z_s
+
\Pi_{\mu_t}\!\left(\int_s^t \partial_r z_r\,dr\right)
-
\int_s^t \Pi_{\mu_r}(\partial_r z_r)\,dr \\
&=
(\Pi_{\mu_t}-\Pi_{\mu_s})z_s
+
\int_s^t(\Pi_{\mu_t}-\Pi_{\mu_r})(\partial_r z_r)\,dr.
\end{align*}
Therefore,
\begin{align*}
\|w_t-w_s\|_{L^2(\mathbb{T}^d)}
&\le
C_{\dot\Pi}|t-s|\,\|z_s\|_{L^2(\mathbb{T}^d)}
+
C_{\dot\Pi}\int_s^t |t-r|\,\|\partial_r z_r\|_{L^2(\mathbb{T}^d)}\,dr.
\end{align*}
Since $w_{(\cdot)}$ is absolutely continuous, it is differentiable for a.e. $t$.
For such $t$, define
\[
(\partial_t\Pi_t)z_t \triangleq \partial_t w_t
=
\partial_t(\Pi_{\mu_t}z_t)-\Pi_{\mu_t}(\partial_t z_t).
\]
It remains to prove the bound on this defect term. Let $t$ be a point at which $w_{(\cdot)}$ is differentiable. For $h$ such that
$[t,t+h]\subset [0,1]$, the identity above gives
\[
w_{t+h}-w_t
=
(\Pi_{\mu_{t+h}}-\Pi_{\mu_t})z_t
+
\int_t^{t+h}(\Pi_{\mu_{t+h}}-\Pi_{\mu_r})(\partial_r z_r)\,dr.
\]
Dividing by $h$ and taking norms,
\[
\left\|\frac{w_{t+h}-w_t}{h}\right\|_{L^2(\mathbb{T}^d)}
\le
C_{\dot\Pi}\|z_t\|_{L^2(\mathbb{T}^d)}
+
C_{\dot\Pi}\frac{1}{|h|}
\int_t^{t+h}|t+h-r|\,\|\partial_r z_r\|_{L^2(\mathbb{T}^d)}\,dr.
\]
Since $|t+h-r|\le |h|$ on the interval of integration,
\[
\frac{1}{|h|}
\int_t^{t+h}|t+h-r|\,\|\partial_r z_r\|_{L^2(\mathbb{T}^d)}\,dr
\le
\int_t^{t+h}\|\partial_r z_r\|_{L^2(\mathbb{T}^d)}\,dr
\to 0
\]
as $h\to 0$. Hence
\[
\|\partial_t w_t\|_{L^2(\mathbb{T}^d)}
\le
C_{\dot\Pi}\|z_t\|_{L^2(\mathbb{T}^d)}
\]
for a.e. $t$, i.e.
\[
\|(\partial_t\Pi_t)z_t\|_{L^2(\mathbb{T}^d)}
\le
C_{\dot\Pi}\|z_t\|_{L^2(\mathbb{T}^d)}.
\]
This proves item (2). The final statement in $L^2(\nu)$ follows from the uniform equivalence between
$L^2(\mathbb{T}^d)$ and $L^2(\nu)$ under the density bounds.
\end{proof}

\begin{restatable}[$H^1$ propagation of parallel fields]{prop}{}\label{prop:h1}
    Let $(\mu_t)_{t\in[0,1]}$ be a regular geodesic in $\mathcal{P}_2(\mathbb{T}^d)$ with tangent velocity field $\nabla\phi_t$, and suppose that each $\mu_t$ admits a Lebesgue density $\rho_t$. Let $w_t = \PT_{\nu \to \mu_t}(v)$ be the parallel transport of $v \in T_\nu \mathcal{P}_2(\mathbb{T}^d) \cap H^1(\mathbb{T}^d;\mathbb{R}^d)$ along the geodesic. Suppose that there exist constants $0 < \lambda \le \Lambda < \infty$ such that $\lambda \le \rho_t(x) \le \Lambda$ for a.e.\ $x \in \mathbb{T}^d$ and all $t \in [0,1]$, it holds that \[\displaystyle\sup_{t \in [0,1]} \|\nabla\phi_t\|_{W^{1,\infty}(\mathbb{T}^d)} \le M,\] 
    and it holds that  \[\displaystyle\sup_{t \in [0,1]} \|\nabla \log \rho_t\|_{L^\infty(\mathbb{T}^d)} \le K_{\log}.\] Then there exists $C_{\mathrm{par}} > 0$ depending only on $\lambda, \Lambda, M, K_{\log}$ such that
    \[
        \|w_t\|_{H^1(\mathbb{T}^d)} \le C_{\mathrm{par}}\,\|v\|_{H^1(\mathbb{T}^d)} \qquad \forall\, t \in [0,1].
    \]
    Moreover, $w \in L^\infty(0,1;H^1(\mathbb T^d;\mathbb R^d))
    \cap W^{1,\infty}(0,1;L^2(\mathbb T^d;\mathbb R^d))$, with
    \[
    \|w\|_{L^\infty(0,1;H^1(\mathbb T^d))}
    +\|\partial_t w\|_{L^\infty(0,1;L^2(\mathbb T^d))}
    \le C\|v\|_{H^1(\mathbb T^d)}.
    \]
\end{restatable}

\begin{proof}
    Our proof technique will be as follows: we will construct a \textit{Galerkin approximation} (\citet[\S 7.1.2]{evans2022partial}) to the parallel field $w_t$, which approximates the $w_t$ by an orthogonal projection onto a finite-dimensional subspace. In particular, let $\{e_k\}_{k \geq 1}$ denote a smooth mean-zero orthonormal basis of the eigenfunctions of the Laplacian on $\mathbb{T}^d$,
    \[-\Delta e_k = \lambda_ke_k, \quad \int_{\mathbb{T}^d}e_k \, dx = 0. \]
    Define the finite dimensional subspaces
    \[V_m \triangleq \text{span}\{\nabla e_1, \dots, \nabla e_m\} \subset C^\infty(\mathbb{T}^d; \mathbb{R}^d).\]
    By \Cref{projection realization}, we know that $v \in T_\nu \mathcal{P}_2(\mathbb{T}^d) \cap H^1(\mathbb{T}^d;\mathbb{R}^d)$ is necessarily a gradient field of a potential $\psi \in H^1(\mathbb{T}^d)$. The fact that $v \in H^1(\mathbb{T}^d; \mathbb{R}^d)$ implies that $\nabla \cdot v = \Delta \psi \in L^2(\mathbb{T}^d).$ By Theorem 15.1 of \citet{dyatlov2022ellipticnotes}, $\Delta \psi \in L^2(\mathbb{T}^d)$ implies that $\psi \in H^2(\mathbb{T}^d)$. Now define the \textit{approximate} potential and gradient field
    \[\psi_m \triangleq \sum_{k = 1}^m \langle \psi, e_k\rangle_{L^2(\mathbb{T}^d)} \cdot e_k, \quad v_m\triangleq \nabla \psi_m.\]
    Then $v_m \in V_m$, and 
    \[ \psi - \psi_m = \sum_{k > m}\langle \psi, e_k\rangle_{L^2(\mathbb{T}^d)} \cdot e_k \quad \implies \quad\|\psi - \psi_m\|_{H^2(\mathbb{T}^d)}^2 \asymp \sum_{k > m}(1 + \lambda_k)^2|\langle \psi, e_k \rangle |^2 \overset{m \rightarrow \infty}\longrightarrow 0 \]
    by the Fourier characterization of Sobolev spaces on $\mathbb{T}^d$ (\citet{shkoller2009lp_sobolev_notes}, Definition 5.6). Thus,
    \[\|v - v_m\|_{H^1(\mathbb{T}^d)}  = \|\nabla \psi - \nabla \psi_m\|_{H^1(\mathbb{T}^d)} \leq \| \psi -  \psi_m\|_{H^2(\mathbb{T}^d)} \rightarrow 0.\]
    For each $m$, we construct $w_m:[0,1]\to V_m$ by requiring
    \begin{equation}
    \int_{\mathbb T^d}
    \rho_t(x)\,
    \big\langle \partial_t w_m(t,x)+\nabla w_m(t,x)\,\nabla\phi_t(x),\,\xi(x)\big\rangle\,dx
    =0
    \qquad\forall \xi\in V_m
    \label{eq:galerkin_weak}
    \end{equation}
    for a.e. $t\in[0,1]$, together with the requirement $w_m(0)=v_m.$ Now choose a basis $\{\xi_1,\dots,\xi_{N_m}\}$ of $V_m$ and write
    \[
    w_m(t)=\sum_{\ell=1}^{N_m} a_\ell^{(m)}(t)\,\xi_\ell.
    \]
    Then \cref{eq:galerkin_weak} is equivalent to the linear ODE system
    \[
    M_m(t)\,a_m'(t)+B_m(t)\,a_m(t)=0,
    \]
    where
    \[
    (M_m(t))_{ij}=\int_{\mathbb T^d}\rho_t \langle \xi_j,\xi_i\rangle\,dx,
    \qquad
    (B_m(t))_{ij}=\int_{\mathbb T^d}\rho_t \langle \nabla \xi_j\,\nabla\phi_t,\xi_i\rangle\,dx.
    \]
    Because $\lambda\le \rho_t\le \Lambda$, the matrix $M_m(t)$ is uniformly positive definite,
    \[
    a^\top M_m(t)a
    =\int_{\mathbb T^d}\rho_t \Big\|\sum_{j=1}^{N_m} a_j \xi_j\Big\|_2^2 dx
    \ge \lambda \Big\|\sum_{j=1}^{N_m} a_j\xi_j\Big\|_{L^2(\mathbb T^d)}^2.
    \]
    Moreover, $M_m$ and $B_m$ are bounded measurable in $t$. Hence, by Carath\'eodory's existence theorem, there exists a unique
    \[
    a_m\in AC([0,1];\mathbb R^{N_m}),
    \]
    and therefore
    \[
    w_m\in AC([0,1];V_m)\subset AC([0,1];L^2(\mathbb T^d;\mathbb R^d)).
    \]
    Since $w_m(t)\in V_m$, we may choose $\xi=w_m(t)$ in \cref{eq:galerkin_weak}. This gives
    \[
    \int_{\mathbb T^d}\rho_t \langle \partial_t w_m,w_m\rangle\,dx
    +
    \int_{\mathbb T^d}\rho_t \langle \nabla w_m\,\nabla\phi_t,w_m\rangle\,dx
    =0.
    \]
    Because $(\rho_t,\nabla\phi_t)$ solves the continuity equation on $\mathbb T^d$,
    \[
    \partial_t\rho_t+\nabla\cdot(\rho_t\nabla\phi_t)=0
    \]
    in the distributional sense, and because $\mathbb T^d$ has no boundary, we obtain
    \begin{align*}
    \frac{1}{2}\frac{d}{dt}\int_{\mathbb T^d} \rho_t \|w_m\|_2^2\,dx
    &=
    \int_{\mathbb T^d}\rho_t \langle \partial_t w_m,w_m\rangle\,dx
    +\frac12\int_{\mathbb T^d}(\partial_t\rho_t)\|w_m\|_2^2\,dx\\
    &=
    -\int_{\mathbb T^d}\rho_t \langle \nabla w_m\,\nabla\phi_t,w_m\rangle\,dx
    +\frac12\int_{\mathbb T^d}(\partial_t\rho_t)\|w_m\|_2^2\,dx\\
    &=
    -\frac12\int_{\mathbb T^d}\rho_t \nabla\phi_t\cdot\nabla(\|w_m\|_2^2)\,dx
    +\frac12\int_{\mathbb T^d}(\partial_t\rho_t)\|w_m\|_2^2\,dx\\
    &=
    \frac12\int_{\mathbb T^d}\nabla\cdot(\rho_t\nabla\phi_t)\,\|w_m\|_2^2\,dx
    +\frac12\int_{\mathbb T^d}(\partial_t\rho_t)\|w_m\|_2^2\,dx\\
    &=0.
    \end{align*}
    Thus, $\|w_m(t)\|_{L^2(\mu_t)}=\|v_m\|_{L^2(\nu)} $ for all $t\in[0,1].$ Now let $P_{m,t}:L^2(\mu_t;\mathbb R^d)\to V_m$ denote the $L^2(\mu_t)$ orthogonal projection
    onto $V_m$. Since \cref{eq:galerkin_weak} stipulates precisely that
    \[
    \partial_t w_m + P_{m,t}\bigl(\nabla w_m\,\nabla\phi_t\bigr)=0
    \qquad\text{in }V_m,
    \]
    we have $\partial_t w_m = - P_{m,t}\bigl(\nabla w_m\,\nabla\phi_t\bigr).$
    By contractivity of orthogonal projection in $L^2(\mu_t)$,
    \[
    \|\partial_t w_m\|_{L^2(\mu_t)}
    \le \|\nabla w_m\,\nabla\phi_t\|_{L^2(\mu_t)}
    \le \|\nabla\phi_t\|_{L^\infty}\|\nabla w_m\|_{L^2(\mu_t)}
    \le M\|\nabla w_m\|_{L^2(\mu_t)}.
    \]
    Thus
    \begin{equation}
    \|\partial_t w_m\|_{L^2(\mathbb T^d)}
    \le M\sqrt{\frac{\Lambda}{\lambda}}\,\|\nabla w_m\|_{L^2(\mathbb T^d)}.
    \label{eq:dtwm_bound}
    \end{equation}
    Define the weighted gradient energy
    \[
    E_m(t)\triangleq \int_{\mathbb T^d} \rho_t(x)\,\|\nabla w_m(t,x)\|_F^2\,dx.
    \]
    The density bounds imply
    \begin{equation}
    \lambda \|\nabla w_m(t)\|_{L^2(\mathbb T^d;F)}^2
    \le E_m(t)
    \le \Lambda \|\nabla w_m(t)\|_{L^2(\mathbb T^d;F)}^2.
    \label{eq:Em_equiv}
    \end{equation}
    Since $w_m=\nabla\psi_m \in V_m$, we have $-\Delta w_m = \nabla(-\Delta\psi_m)\in V_m$. Thus we may choose $\xi=-\Delta w_m(t)$ in \cref{eq:galerkin_weak}. Writing $0=\text{I}_1(t)+\text{I}_2(t)$, where
    \[
    \text{I}_1(t)\triangleq \int_{\mathbb T^d}\rho_t \langle \partial_t w_m,-\Delta w_m\rangle\,dx,
    \quad \text{and} \quad
    \text{I}_2(t)\triangleq \int_{\mathbb T^d}\rho_t \langle \nabla w_m\,\nabla\phi_t,-\Delta w_m\rangle\,dx,
    \]
    we bound the two terms separately. For $\text{I}_1$, we can integrate by parts on $\mathbb T^d$, yielding
    \begin{align*}
    \text{I}_1(t)
    &=
    \sum_{j=1}^d \int_{\mathbb T^d}\rho_t \langle \partial_j\partial_t w_m,\partial_j w_m\rangle\,dx
    +
    \sum_{j=1}^d \int_{\mathbb T^d} (\partial_j\rho_t)\,\langle \partial_t w_m,\partial_j w_m\rangle\,dx\\
    &=
    \frac12 \frac{d}{dt} E_m(t)
    -\frac12\int_{\mathbb T^d}(\partial_t\rho_t)\,\|\nabla w_m\|_F^2\,dx
    +
    R_{1,m}(t),
    \end{align*}
    where
    \[
    R_{1,m}(t)\triangleq
    \sum_{j=1}^d \int_{\mathbb T^d} (\partial_j\rho_t)\,\langle \partial_t w_m,\partial_j w_m\rangle\,dx.
    \]
    Using $\nabla\rho_t=\rho_t\nabla\log\rho_t$, the bound on $\nabla\log\rho_t$, and \cref{eq:dtwm_bound}, we obtain
    \begin{align*}
    |R_{1,m}(t)| &\le \|\nabla\log\rho_t\|_{L^\infty}
    \|\partial_t w_m\|_{L^2(\mu_t)}
    \|\nabla w_m\|_{L^2(\mu_t)}\\
    &\le K_{\log}\,M\,E_m(t).
    \end{align*}
    For $\text{I}_2$, integrate by parts again:
    \begin{align*}
    \text{I}_2(t)
    &=
    \sum_{j=1}^d
    \int_{\mathbb T^d}
    \rho_t \,\big\langle \partial_j(\nabla w_m\,\nabla\phi_t),\,\partial_j w_m\big\rangle\,dx
    +
    \sum_{j=1}^d
    \int_{\mathbb T^d}
    (\partial_j\rho_t)\,\big\langle \nabla w_m\,\nabla\phi_t,\,\partial_j w_m\big\rangle\,dx\\
    &\triangleq  \text{I}_{21}(t)+\text{I}_{22}(t).
    \end{align*}
    The second term satisfies
    \[
    |\,\text{I}_{22}(t)|
    \le \|\nabla\log\rho_t\|_{L^\infty}\|\nabla\phi_t\|_{L^\infty} E_m(t)
    \le K_{\log} M E_m(t).
    \]
    Expanding yields
    \[
    \partial_j(\nabla w_m\,\nabla\phi_t)
    =
    (\partial_j\nabla w_m)\,\nabla\phi_t
    +
    \nabla w_m\,(\partial_j\nabla\phi_t).
    \]
    Hence $\text{I}_{21}(t)=\text{I}_{211}(t)+\text{I}_{212}(t)$ where
    \begin{align*}
        \text{I}_{211}(t)
        &\triangleq
        \sum_{j=1}^d
        \int_{\mathbb T^d}
        \rho_t \,\big\langle (\partial_j\nabla w_m)\,\nabla\phi_t,\,\partial_j w_m\big\rangle\,dx, \\
        \text{I}_{212}(t)
        &\triangleq
        \sum_{j=1}^d
        \int_{\mathbb T^d}
        \rho_t \,\big\langle \nabla w_m\,(\partial_j\nabla\phi_t),\,\partial_j w_m\big\rangle\,dx.
    \end{align*}
    To bound $\text{I}_{212}(t)$ we simply use $\|\nabla^2\phi_t\|_{L^\infty}\le M$, yielding $|\text{I}_{212}(t)|\le M E_m(t).$ For $\text{I}_{211}$, observe that
    \[
    \sum_{j=1}^d \big\langle (\partial_j\nabla w_m)\,\nabla\phi_t,\,\partial_j w_m\big\rangle
    =
    \frac12 \nabla\phi_t\cdot \nabla\bigl(\|\nabla w_m\|_F^2\bigr).
    \]
    Therefore
    \begin{align*}
    \text{I}_{211}(t)
    &=
    \frac12\int_{\mathbb T^d}\rho_t\,\nabla\phi_t\cdot\nabla\bigl(\|\nabla w_m\|_F^2\bigr)\,dx\\
    &=
    -\frac12\int_{\mathbb T^d}\nabla\cdot(\rho_t\nabla\phi_t)\,\|\nabla w_m\|_F^2\,dx\\
    &=
    \frac12\int_{\mathbb T^d}(\partial_t\rho_t)\,\|\nabla w_m\|_F^2\,dx,
    \end{align*}
    again by the continuity equation. Combining the preceding bounds gives
    \[
    \text{I}_2(t)
    =
    \frac12\int_{\mathbb T^d}(\partial_t\rho_t)\,\|\nabla w_m\|_F^2\,dx
    +
    R_{2,m}(t),
    \quad \text{where} \quad 
    |R_{2,m}(t)|\le M(1+K_{\log})E_m(t).
    \]
    Since $I_1(t)+I_2(t)=0$, the $\partial_t\rho_t$ terms cancel, and we conclude that
    \[
    \frac12 \frac{d}{dt}E_m(t)
    \le \bigl(M+2MK_{\log}\bigr) E_m(t)
    \qquad\text{for a.e. }t\in[0,1].
    \]
    Applying Gr\"onwall's inequality (\Cref{gronwall}) yields,
    \[
    E_m(t)\le e^{2(M+2MK_{\log})t} \cdot E_m(0)
    \qquad\forall t\in[0,1].
    \]
    Using \cref{eq:Em_equiv} at $t=0$ and the convergence $v_m\to v$ in $H^1$,
    \[
    E_m(0)=\int_{\mathbb T^d}\rho_0 \|\nabla v_m
    |_F^2\,dx
    \le \Lambda \|v_m\|_{H^1(\mathbb T^d)}^2
    \le C\|v\|_{H^1(\mathbb T^d)}^2
    \]
    for all sufficiently large $m$. Hence
    \[
    \sup_{m\ge 1}\sup_{t\in[0,1]}\|w_m(t)\|_{H^1(\mathbb T^d)}
    \le C\|v\|_{H^1(\mathbb T^d)}
    \]
    for a constant $C$ depending only on $\lambda,\Lambda,M,K_{\log}$. Combining this with \cref{eq:dtwm_bound} yields
    \[
    \sup_{m\ge 1}\|\partial_t w_m\|_{L^\infty(0,1;L^2(\mathbb T^d))}
    \le C\|v\|_{H^1(\mathbb T^d)}.
    \]
    The uniform bounds obtained above imply
    \[
    \sup_m \|w_m\|_{L^\infty(0,1;H^1(\mathbb T^d))}
    +
    \sup_m \|\partial_t w_m\|_{L^\infty(0,1;L^2(\mathbb T^d))}
    <\infty.
    \]
    For each fixed $t \in [0,1]$, the sequence $\{w_m(t)\}_m$ is bounded in
    $H^1(\mathbb{T}^d;\mathbb{R}^d)$ by the uniform bound on
    $w_m$. Since the embedding
    \[
    H^1(\mathbb{T}^d;\mathbb{R}^d)\hookrightarrow L^2(\mathbb{T}^d;\mathbb{R}^d)
    \]
    is compact by the Rellich--Kondrachov theorem (\citet[\S 5.7]{evans2022partial}),
    it follows that for each fixed $t \in [0,1]$, the set
    $\{w_m(t)\}_m$ is relatively compact in $L^2(\mathbb{T}^d;\mathbb{R}^d)$.
    Moreover, for all $s,t \in [0,1]$,
    \[
    \|w_m(t)-w_m(s)\|_{L^2(\mathbb{T}^d)}
    \le \int_s^t \|\partial_r w_m(r)\|_{L^2(\mathbb{T}^d)}\,dr
    \le C|t-s|,
    \]
    where $C$ is independent of $m$. Thus $\{w_m\}_m$ is equicontinuous as a family
    of maps from $[0,1]$ into $L^2(\mathbb{T}^d;\mathbb{R}^d)$, and for each fixed
    $t$ its image is relatively compact in $L^2(\mathbb{T}^d;\mathbb{R}^d)$.
    By the Arzel\`a--Ascoli theorem, after passing to a subsequence we obtain
    \[
    w_m \to w
    \qquad\text{strongly in } C([0,1];L^2(\mathbb{T}^d;\mathbb{R}^d)).
    \] In particular, since $w_m(0)=v_m$ and $v_m\to v$ in $L^2(\mathbb T^d;\mathbb R^d)$, we obtain $w(0) = v$. On the other hand, the uniform $L^\infty(0,1;H^1(\mathbb{T}^d))$ bound on $w_m$ allows us to apply the Banach-Alaoglu theorem \citep[Theorem 3.15]{rudin1991functional} to extract a \textit{further} subsequence along which
    \[
    w_{m} \overset{*}{\rightharpoonup} w
    \quad\text{in }L^\infty(0,1;H^1(\mathbb T^d;\mathbb R^d))
    \]
    where $\overset{*}{\rightharpoonup}$ denotes convergence in the weak-* topology. It therefore follows that 
    \[
    \nabla w_{m} \overset{*}{\rightharpoonup} \nabla w
    \quad\text{in }L^\infty(0,1;L^2(\mathbb T^d;\mathbb R^d))
    \]
    along the same subsequence. By the uniform bound
    \[
    \sup_m \|\partial_t w_m\|_{L^\infty(0,1;L^2(\mathbb T^d))}
    < \infty,
    \]
    the Banach--Alaoglu theorem \citep[Theorem 3.15]{rudin1991functional} yields, after passing to a further subsequence
    (not relabeled), the existence of
    \[
    g \in L^\infty(0,1;L^2(\mathbb T^d;\mathbb R^d))
    \]
    such that
    \[
    \partial_t w_m \overset{*}{\rightharpoonup} g
    \qquad\text{in }L^\infty(0,1;L^2(\mathbb T^d;\mathbb R^d)).
    \]
    For every $\psi \in L^2(\mathbb T^d;\mathbb R^d)$ and every
    $\zeta \in C_c^\infty(0,1)$, since
    $w_m \in W^{1,\infty}(0,1;L^2(\mathbb T^d;\mathbb R^d))$, we have
    \[
    \int_0^1 \langle \partial_t w_m(t),\psi\rangle_{L^2}\,\zeta(t)\,dt
    =
    -\int_0^1 \langle w_m(t),\psi\rangle_{L^2}\,\zeta'(t)\,dt.
    \]
    Passing to the limit as $m\to\infty$, using the weak-* convergence of
    $\partial_t w_m$ and the strong convergence
    $w_m \to w$ in $C([0,1];L^2(\mathbb T^d;\mathbb R^d))$, yields
    \[
    \int_0^1 \langle g(t),\psi\rangle_{L^2}\,\zeta(t)\,dt
    =
    -\int_0^1 \langle w(t),\psi\rangle_{L^2}\,\zeta'(t)\,dt.
    \]
    Thus $g$ is the weak time derivative of $w$, i.e. $\partial_t w = g$. In particular,
    \[
    w \in W^{1,\infty}(0,1;L^2(\mathbb T^d;\mathbb R^d)).
    \]
    Moreover, since $g$ is a weak-* limit of a sequence bounded by
    $C\|v\|_{H^1(\mathbb T^d)}$ in $L^\infty_tL^2_x$, we have
    \[
    \|\partial_t w\|_{L^\infty(0,1;L^2(\mathbb T^d))}
    \le C\|v\|_{H^1(\mathbb T^d)}.
    \]
    Let $\eta$ be any mean-zero trigonometric polynomial and set $\xi=\nabla\eta$, and let $\zeta \in C^\infty_c(0,1)$. For all $m$
    large enough, $\xi\in V_m$, so \cref{eq:galerkin_weak} gives
    \[
    \int_0^1\!\!\int_{\mathbb T^d}
    \rho_t \,\big\langle \partial_t w_m+\nabla w_m\,\nabla\phi_t,\xi\big\rangle
    \zeta(t)\,dx\,dt=0
    \qquad
    \forall \zeta\in C_c^\infty(0,1).
    \]
    Integrating the first term by parts in time yields
    \begin{multline*}
    0 =
    -\int_0^1\int_{\mathbb T^d}
    \rho_t\langle w_m,\xi\rangle \zeta'(t)\,dx\,dt
    -\int_0^1\int_{\mathbb T^d}
    (\partial_t\rho_t)\langle w_m,\xi\rangle \zeta(t)\,dx\,dt \\
    +\int_0^1\int_{\mathbb T^d}
    \rho_t\langle \nabla w_m\,\nabla\phi_t,\xi\rangle \zeta(t)\,dx\,dt.
    \end{multline*}
    Passing to the limit, using the strong convergence of $w_m$ in $C([0,1];L^2)$ for the first two terms
    and the weak-$*$ convergence of $\nabla w_m$ in $L^\infty_tL^2_x$ for the last term, we obtain
    \begin{multline*}
    0
    =
    -\int_0^1\int_{\mathbb T^d}
    \rho_t\langle w,\xi\rangle \zeta'(t)\,dx\,dt
    -\int_0^1\int_{\mathbb T^d}
    (\partial_t\rho_t)\langle w,\xi\rangle \zeta(t)\,dx\,dt \\
    +\int_0^1\int_{\mathbb T^d}
    \rho_t\langle \nabla w\,\nabla\phi_t,\xi\rangle \zeta(t)\,dx\,dt.
    \end{multline*}
        Since $w \in W^{1,\infty}(0,1;L^2(\mathbb T^d;\mathbb R^d))$, we may apply
    the Banach-valued integration-by-parts formula with the test function
    \[
    \Psi(t) \triangleq \rho_t \cdot \xi\cdot \zeta(t)
    \in W^{1,1}(0,1;L^2(\mathbb T^d;\mathbb R^d)).
    \]
    This gives
    \[
    \int_0^1\int_{\mathbb T^d}
    \rho_t \langle \partial_t w,\xi\rangle \zeta(t)\,dx\,dt
    =
    -\int_0^1\int_{\mathbb T^d}
    \rho_t\langle w,\xi\rangle \zeta'(t)\,dx\,dt
    -\int_0^1\int_{\mathbb T^d}
    (\partial_t\rho_t)\langle w,\xi\rangle \zeta(t)\,dx\,dt.
    \]
    Therefore the preceding identity is equivalent to
    \[
    \int_0^1\int_{\mathbb T^d}
    \rho_t \big\langle \partial_t w + \nabla w\,\nabla\phi_t,\xi\big\rangle
    \zeta(t)\,dx\,dt = 0
    \qquad\forall\,\zeta\in C_c^\infty(0,1).
    \]
    For $\eta\in C^\infty(\mathbb T^d)$, let $S_N\eta$ be its Fourier partial sums. Then
    $S_N\eta$ is a trigonometric polynomial for each $N$, and $S_N\eta\to\eta$ in
    $C^1(\mathbb T^d)$, since the Fourier series of $\eta$ and of each $\partial_j\eta$
    converge uniformly \citep[Proposition 1.6]{casselman2016fourier}. Hence
    \[
    \nabla S_N\eta \to \nabla\eta
    \qquad\text{uniformly on }\mathbb T^d,
    \]
    so trigonometric gradients are dense in the space of smooth periodic gradient fields.
    It then follows that
    \[
    \Pi_{\mu_t}\bigl(\partial_t w_t+\nabla w_t\,\nabla\phi_t\bigr)=0
    \]
    in the distributional sense on $(0,1)$. Moreover, because each $w_m(t)\in V_m\subset T_{\mu_t}\mathcal P_2(\mathbb T^d)$
    and $T_{\mu_t}\mathcal P_2(\mathbb T^d)$ is closed in $L^2(\mu_t)$, the strong $L^2$ convergence
    implies
    \[
    w_t\in T_{\mu_t}\mathcal P_2(\mathbb T^d)
    \qquad\forall t\in[0,1].
    \]
    Thus $w$ is a parallel field along $(\mu_t)$ with initial value $v$. By uniqueness of parallel
    transport along regular geodesics \citep[page 104]{ambrosio2012user}, we must have
    \[
    w_t = \PT_{\nu\to\mu_t}(v)
    \qquad\forall t\in[0,1].
    \]
    Finally, the uniform $H^1$ bound passes to the limit by weak lower semicontinuity. In particular, for a.e.
    $t\in[0,1]$,
    \[
    \|w_t\|_{H^1(\mathbb T^d)}
    \le \liminf_{m\to\infty}\|w_m(t)\|_{H^1(\mathbb T^d)}
    \le C\|v\|_{H^1(\mathbb T^d)}.
    \]
    Since $w\in C([0,1];L^2)$, this bound extends from a full-measure subset of times to all
    $t\in[0,1]$ by taking sequences $t_n\to t$ from that full-measure set and using weak compactness
    in $H^1$ together with strong convergence in $L^2$. Therefore
    \[
    \|w_t\|_{H^1(\mathbb T^d)} \le C_{\mathrm{par}}\|v\|_{H^1(\mathbb T^d)}
    \qquad\forall t\in[0,1],
    \]
    as claimed.
\end{proof}

\begin{restatable}{corollary}{} \label{lem:pt_ac}
    Under the hypotheses of \Cref{prop:h1}, the parallel transport field \[w_t=\PT_{\nu\to\mu_t}(v)\] belongs to
    \[
    W^{1,\infty}(0,1;L^2(\mathbb T^d;\mathbb R^d))
    \subset AC([0,1];L^2(\mathbb T^d;\mathbb R^d)).
    \]
    Under the uniform density bounds, the same conclusion holds with
    $L^2(\nu;\mathbb R^d)$ in place of $L^2(\mathbb T^d;\mathbb R^d)$.
    The same statement holds for any second regular geodesic $(\mu_t')$ satisfying the
    same hypotheses.
\end{restatable}
\begin{proof}
The proof of \Cref{prop:h1} shows, in addition to the pointwise $H^1$ bound,
that the limit parallel transport field satisfies $w \in W^{1,\infty}(0,1;L^2(\mathbb T^d;\mathbb R^d)).$ Thus, $w \in AC([0,1];L^2(\mathbb T^d;\mathbb R^d)).$ Under the uniform density bounds, the norms of
$L^2(\mathbb T^d;\mathbb R^d)$ and $L^2(\nu;\mathbb R^d)$ are uniformly
equivalent, so the same conclusion holds with $L^2(\nu;\mathbb R^d)$ in place
of $L^2(\mathbb T^d;\mathbb R^d)$. The same argument applies verbatim to any
second regular geodesic $(\mu_t')$ satisfying the same hypotheses.
\end{proof}

\begin{restatable}[Lax-Milgram Theorem, {\citet[§~6.2.1]{evans2022partial}}]{thm}{} \label{lax-milgram}
    Let $H$ be a real Hilbert space with inner product $\langle \cdot, \cdot \rangle_H$ and norm $\|\cdot \|_H$. Assume that $B: H \times H \rightarrow \mathbb{R}$ is a bilinear mapping for which there exist constants $\alpha, \beta > 0$ such that 
    \[|B(u, v)| \leq \alpha \|u\|_H\|v\|_H \quad \forall\, u,v \in H\]
    and 
    \[\beta \|u\|^2_H \leq B(u, u) \quad \forall \, u \in H.\]
    Finally, let $f: H \rightarrow \mathbb{R}$ be a bounded linear functional on $H$. Then there exists a unique element $u \in H$ such that 
    \[B(u, v) = \langle f, v\rangle_H\]
    for all $v \in H.$
\end{restatable}

\section{Proofs of Results From the Text} \label{sec: proof of main results}

\subsection{Proof of \Cref{Gaussian parallel transport}} \label{sec: proof of Gaussian parallel transport}
In this derivation we will abuse notation and denote $\mu_t$ as both the density and the probability measure associated with the Gaussian of interest -- since Gaussians admit a density, there should be no confusion. Given two Gaussian probability measures 
\[\mu_0 = N(m_0, \Sigma_0),  \quad \mu_1 = N(m_1, \Sigma_1)\]
with $\Sigma_0, \Sigma_1 \in \mathbb{S}_+^d$, the optimal transport map and tangent vector are
\[T(x) = m_1 - Bm_0 + Bx, \quad v(x) = m_1 - Bm_0 + (B - I_d)x\]
and $B$ is the \textit{Brenier matrix} given by 
\[B = \Sigma_0^{-1/2}\left(\Sigma_0^{1/2}\Sigma_1\Sigma_0^{1/2}\right)^{1/2}\Sigma_0^{-1/2}\]
as derived in \citet{ojm/1326291215}. We can obtain the interpolating geodesic by pushing $\mu_0$ forward through the map 
\[F_t(x) = m_t + M_t(x - m_0)\]
where $m_t = (1-t)m_0 + tm_1$ and $M_t = (1-t)I_d + tB$,  which implies that the interpolating measures are given by
\[\mu_t = (F_t)_\#\mu_0 = N(m_t, M_t\Sigma_0M_t^{\top}).\]
Now we want to obtain the Eulerian velocity $\nabla\varphi_t$ of the geodesic. Since $\man = \mathbb{R}^d$ we have 
\begin{align*}
    \nabla\varphi_t(y) = ((T - I_d) \circ F^{-1}_t)(y) &= m_1 - Bm_0 + (B - I_d)(F_t^{-1}(y)) \\
    &= (m_1 - m_0) + \left(B - I_d\right)M_t^{-1}\left(y - m_t\right).
\end{align*}
Now consider an affine tangent vector of the form $v_t(x) = a_t + A_t(x-m_t) \in T_{\mu_t}\mathcal{P}_2(\mathbb{R}^d), A_t \in \mathbb{S}^d$ along the geodesic $\mu_t$. Note that we restrict our search for a parallel field to affine maps as they preserve Gaussianity -- furthermore, we require $A_t$ to be symmetric to ensure that $v_t$ is truly a gradient field. For $v_t$ to be the parallel transport along the geodesic $(F_t)_\#\mu_0$, it needs to be the a.e. solution to the PDE
\[\nabla\cdot \left(\mu_t \left(\partial_tv_t + \nabla v_t \cdot \nabla\varphi_t \right)\right) = 0.\]
We have $\partial_tv_t = \dot a_t + \dot A_t (x - m_t) - A_t(m_1 - m_0)$ and $\nabla v_t = A_t$, which taken together yield the condition
\[\nabla \cdot (\mu_tw_t) = 0, \quad w_t \triangleq \dot a_t + \left(\dot A_t + A_t(B-I_d)M_t^{-1}\right)(x-m_t).\]
Applying the product rule for divergence,
\begin{align*}
    \nabla\cdot (\mu_tw_t) &=  \mu_t (\nabla \cdot w_t) + \langle w_t, \nabla \mu_t\rangle = \mu_t\left(\nabla \cdot w_t + \left\langle w_t, \nabla \log \mu_t\right\rangle\right).
\end{align*}
For the first term, 
\[\nabla \cdot w_t = \text{tr}(\nabla w_t) = \text{tr}\left(\dot A_t + A_t(B - I_d)M_t^{-1}\right) \triangleq \text{tr}\left(C_t\right).\]
For the second term, we have $\nabla \log \mu_t = -M_t^{-\top}\Sigma^{-1}_0M_t^{-1}\left(x - m_t\right) \triangleq - Q_t(x-m_t).$ Putting it together, since $\mu_t > 0$ we know that $v_t$ is a parallel field along $\mu_t$ if  
\[\text{tr}(C_t) - \dot a_t^{\top}Q_t(x-m_t) - (x-m_t)C_t^{\top}Q_t(x-m_t) = 0 \quad \forall \, x \in \mathbb{R}^d.\]
Thus, we require that all coefficients of the second-order polynomial above be zero. Since $Q_t$ is invertible, $\dot a_t \equiv 0$. A sufficient set of conditions on $\dot A_t$ and $A_t$ is described by 
\[\text{tr}\left(C_t\right) = 0 \quad \text{and} \quad C_t^{\top}Q_t \text{ is skew-symmetric}.\]
We will start with the second condition. Expanding, we have
\begin{align*}
\left(\dot A_t + A_t(B - I_d)M_t^{-1}\right)^{\top}Q_t + Q_t^{\top}\left(\dot A_t + A_t(B - I_d)M_t^{-1}\right) &= 0 \\
\implies \dot A_tQ_t + Q_t\dot A_t = S_t^{\top}A_tQ_t + Q_tA_tS_t \quad \text{with} \quad S_t \triangleq (I_d - B)&M_t^{-1}. \tag{\textasteriskcentered}
\end{align*}
Note that $\dot A_t$ can be recovered with well established tools, as (\textasteriskcentered) is simply an instantiation of the continuous Lyapunov equation. The trace condition $\text{tr}(C_t) = 0$ follows from skew symmetry, as skew symmetry of $C_t^{\top} Q_t$ and the fact that $Q_t$ is positive semi-definite guarantees that $C_t$ has zero trace. Therefore, 
\[v_t(x) = a_0 + A_t(x - m_t) \quad \text{with} \quad A_t =  A_0 + \int_0^t\dot A_s \, ds\]
is the parallel transport of $v_0(x)$ along the curve $\mu_t$. \qed{}

\subsection{Proof of \Cref{gradient representer thm}} \label{sec: proof of gradient representer thm}

Firstly, Theorem 1 of \citet{zhou2008derivative} proves the derivative-reproducing property $\partial_{j} f(x) = \langle \partial_{j} K(x, \cdot), f\rangle_{\mathcal{H}}$ where the partial derivative is taken with respect to the $j$-th coordinate of the second argument. Now consider the function class $\text{span}\{\partial_{j}K(x_i, \cdot) : i \in \{1, \dots , n\}, j \in \{1, \dots, d\}\}$. Any function $f \in \mathcal{H}$ can therefore be decomposed as $f = f^{\parallel} + f^{\perp}$ where $f^\parallel$ is in the class and $f^\perp$ is in the orthogonal complement -- this implies that $\|f\|_{\mathcal{H}}^2 = \|f^{\parallel}\|_{\mathcal{H}}^2 + \|f^\perp\|_{\mathcal{H}}^2 \geq \|f^{\parallel}\|_{\mathcal{H}}^2$, and 
\begin{align*}
    \partial_{j}f(x) &= \langle \partial_{j}K(x, \cdot), f \rangle _{\mathcal{H}} \tag{reproducing property} \\
    &= \langle \partial_{j}K(x, \cdot), f^{\parallel} \rangle _{\mathcal{H}} + \langle \partial_{j}K(x, \cdot), f^{\perp} \rangle _{\mathcal{H}} \\
    &= \langle \partial_{j}K(x, \cdot), f^{\parallel} \rangle _{\mathcal{H}} \\
    &= \partial_{j}f^\parallel(x) \tag{reproducing property}
\end{align*}
for any $x \in \{x_1, \dots, x_n\}$ by definition of $f^\perp$. Together, these imply that the optimization problem in \Cref{eq: sample gradient RKHS problem} is minimized by choosing $\hat f_\lambda \in \text{span}\{\partial_{j} K(x_i, \cdot): i \in \{1, \dots, n\}, j \in \{1, \dots, d\}\}.$ This means it can be written as 
\begin{align*}
    \hat f_\lambda (\cdot ) = \sum_{i = 1}^n \langle c_i, \nabla_2 K(x_i, \cdot) \rangle \quad \text{and} \quad \nabla \hat f_\lambda (x_j) = \sum_{i = 1}^n \nabla_2^2K(x_i, x_j) c_i
\end{align*}
where the inner product is the standard Euclidean inner product on $\mathbb{R}^d$ and $\nabla_i$ denotes the gradient with respect to the $i$-th argument. 
\qed{}

\subsection{Proof of \Cref{suff cond for density}} \label{sec: proof of suff cond for density}
Let 
\[\mathcal{F}_\mu \triangleq \left\{\nabla f : f \in \mathcal{H}\right\} \quad \text{and} \quad \mathcal{A}_\mu \triangleq \{\nabla f: f \in C^\infty_c(\mathbb{R}^d)\}\]
We will first show that for any $\nabla f \in \mathcal{F}_\mu$ there exists a sequence $\nabla f_m \in \mathcal{A}_\mu$ such that $\nabla f_m \rightarrow \nabla f$ in the $L^2(\mu)$ topology. This will imply the inclusion $\mathcal{F}_\mu \subseteq \overline{\mathcal{A}_\mu}^{L^2(\mu)}$.  By assumption, we know that $\mathcal{H} = H^\tau(\mathbb{R}^d)$ as sets and they are norm equivalent. Since
\[H^\tau(\mathbb{R}^d) = \overline{C_c^\infty(\mathbb{R}^d)}^{H^\tau}\]
by definition, we know that $C^\infty_c(\mathbb{R}^d)$ is dense in $\mathcal{H}$ with respect to the Sobolev norm $H^\tau$. This means that for every $f \in \mathcal{H}$ there exists a sequence $f_m \in C^{\infty}_c$ such that 
\[\|f_m - f\|_{H^{\tau}(\mathbb{R}^d)} \rightarrow 0.\]
Now we will show that this implies density with respect to the $L^2(\mu)$ norm. The assumption that $\tau \geq 1$ implies that $\|f\|_{H^{\tau}(\mathbb{R}^d)}^2 = \sum_{|\alpha| \leq k}\|D^\alpha f\|_{L^2(\mathbb{R}^d)}^2 \geq \|\nabla f\|_{L^2(\mathbb{R}^d)}^2$, implying that the operator $\nabla: H^\tau(\mathbb{R}^d) \rightarrow L^2(\mathbb{R}^d)$ is Lipschitz with constant $1$. Thus, if $f_m \rightarrow f$ in $H^\tau(\mathbb{R}^d)$ then
\[\|\nabla f_m - \nabla f\|_{L^2(\mathbb{R}^d)} \rightarrow 0.\]
Since the Lebesgue density $\rho$ of $\mu$ is bounded, convergence in $L^2(\mathbb{R}^d)$ implies convergence in $L^2(\mu)$. In particular, $\|\nabla f_m - \nabla f\|_{L^2(\mu)} \rightarrow 0$. Thus, for any $\nabla f \in \mathcal{F}_\mu$ there exists a sequence in $\mathcal{A}_\mu$ converging to it. This implies the inclusion 
\[\mathcal{F}_\mu \subseteq \overline{\mathcal{A}_\mu}^{L^2(\mu)} \implies \overline{\mathcal{F}_\mu}^{L^2(\mu)}\subseteq \overline{\mathcal{A}_\mu}^{L^2(\mu)}.\]
Now we will see that the reverse conclusion is trivial. By assumption $\mathcal{H} \cong H^\tau(\mathbb{R}^d)$  and $ H^\tau (\mathbb{R}^d) = \overline{C_c^\infty(\mathbb{R}^d)}^{H^\tau}$, which implies that $C^\infty_c(\mathbb{R}^d) \subseteq \mathcal{H}$. Taking the $L^2(\mu)$ closure finishes the proof.
\qed{}

\subsection{Proof of \Cref{one step parallel transport approx result}.} \label{sec: proof of one step parallel transport approx}
Let $v \in L^2(\nu)$ be a vector field on $M$ and define the following function, 
\begin{equation*}
    L^2(\nu_s) \ni \tau_0^s(v)(x) \triangleq \begin{cases}
        \text{the parallel transport of $v(F_s^{-1}(x))$ on $\man$ along the curve} \\
        r \mapsto F_r(F_s^{-1}(x)) \text{ from $r = 0$ to $r = s$ }.
    \end{cases}
\end{equation*}
By the triangle inequality, we have
\begin{align*}
    \|w_s - \tau_0^s(v)\|_{L^2(\nu_s)} \leq \|w_s - \Pi_{\nu_s}(\tau_0^s(v))\|_{L^2(\nu_s)} + \|\Pi_{\nu_s}(\tau_0^s(v)) - {\PT}_{\nu \rightarrow \nu_s}(v)\|_{L^2(\nu_s)}.
\end{align*}
By equations 4.18 and 4.10 of \citet{gigli2012second}, we have that the second term satisfies
\begin{align*}
    \|\Pi_{\nu_s}(\tau_0^s(v)) - {\PT}_{\nu \rightarrow \nu_s}(v)\|_{L^2(\nu_s)} \leq \left(e^{\int_0^1{\operatorname{Lip}}(\nabla \psi_r)dr} - 1\right)^2\|v\|_{L^2(\nu)}\left(\int_0^s{\operatorname{Lip}}(\nabla\psi_r)dr\right)^2
\end{align*}
where ${\operatorname{Lip}}(\nabla\psi_r)$ is the largest Lipschitz constant of the velocity field generating the geodesic $\mathbf{exp}_{\nu}(su).$ Now we will handle the first term in the triangle inequality. Note that $\Pi_{\nu_s}: L^2(\nu_s) \rightarrow T_{\nu_s}\mathcal{P}_{2}(\man)$ is a projection onto a closed and linear subset, and it is therefore non-expansive. This allows us to say
\begin{align*}
    \|w_s - \Pi_{\nu_s}(\tau_0^s(v))\|_{L^2(\nu_s)} &\leq \left\|s^{-1}j_{\nu, u}\left(0, v\right)(s)  - \tau_0^s(v) \right\|_{L^2(\nu_s)} \\
    &= \bigg(\int \|s^{-1}j_{F_s^{-1}(x), u(F^{-1}_s(x))}\left(0, v(F_s^{-1}(x))\right)(s) \\
    &\hspace{45mm}- {\PT}^\man_{F_s^{-1}(x) \rightarrow x}(v(F_s^{-1}(x))) \|_g^2 \,d\nu_s(x)\bigg)^{1/2} \\
    &= \left(\int \|s^{-1}j_{y, u(y)}\left(0, v(y)\right)(s)  - {\PT}^\man_{y \rightarrow F_s(y)}(v(y)) \|_g^2 \,d\nu(y)\right)^{1/2}.
\end{align*}
Now we'll apply \Cref{fanning approximation} to say 
\begin{align*}
    \|s^{-1}j_{y, u(y)}\left(0, v(y)\right)(s)  - {\PT}^{\man}_{y \rightarrow F_s(y)}(v(y))\|_g \leq As^2\|v(y)\|_g
\end{align*}
for $s$ sufficiently small, which implies 
\begin{align*}
    \|w_s - \Pi_{\nu_s}(\tau_0^s(v))\|_{L^2(\nu_s)} &\leq \left(\int A^2s^4\|v\|_g^2 \,d\nu\right)^{1/2}  = As^2\|v\|_{L^2(\nu)}.
\end{align*}
The second part of the result follows from the fact that $t \mapsto \nu_t$ is assumed to be strongly regular and thus admits a velocity field with a spatial Lipschitz constant that integrates to $O(s)$ in the bound of the second term.
\qed{}

\subsection{Proof of \Cref{full approximation of parallel transport along regular curves}.} \label{sec: proof of full approximation of parallel transport along regular curves}
By \Cref{one step parallel transport approx result} we know that for each $k \in \{1\, \dots , N\}$ we have \[\|\widehat{{\PT}}_k(w) - {\PT}_k(w)\|_{L^2(\nu_{ks})} \leq Cs^2\]
for some $C > 0$ independent of $k$. Now let $e_k \triangleq \hat v_k - v_k.$ By the triangle inequality
\begin{align*}
    \|\hat v_k - v_k\|_{L^2(\nu_{sk})}  &= \|\widehat{{\PT}}_k(\hat v_{k-1}) - {\PT}_k(v_{k-1})\|_{L^2(\nu_{sk})}\\ &\leq \|\widehat{{\PT}}_k(\hat v_{k-1}) - {\PT}_k(\hat v_{k-1})\|_{L^2(\nu_{sk})} + \|{\PT}_k(\hat v_{k-1}) - {\PT}_k(v_{k-1})\|_{L^2(\nu_{sk})}.
\end{align*}
Note that the first term is bounded by $Cs^2\|\hat v_{k - 1}\|_{L^2(\nu_{(k-1)s})}$ due to \Cref{one step parallel transport approx result}. Since parallel transport is an isometry, we have that the second term can be rewritten to
\[\|{\PT}_k(\hat v_{k-1}) - {\PT}_k(v_{k-1})\|_{L^2(\nu_{sk})} = \|\hat v_{k-1} - v_{k-1}\|_{L^2(\nu_{s(k-1)})} = \|e_{k-1}\|_{L^2(\nu_{s(k-1)})}.\]
Putting it together, we have the recurrence relation 
\[\|e_k\|_{L^2(\nu_{sk})} \leq Cs^2\|\hat v_{k - 1}\|_{L^2(\nu_{(k-1)s})} + \|e_{k-1}\|_{L^2(\nu_{s(k-1)})}.\]
By triangle inequality and the isometry of parallel transport, we have 
\[\|\hat v_{k-1}\|_{L^2(\nu_{(k-1)s})} \leq \|\hat v_{k-1} - v_{k - 1}\|_{L^2(\nu_{(k-1)s})} + \|v_{k - 1}\|_{L^2(\nu_{(k-1)s})} = \|e_{k-1}\|_{L^2(\nu_{(k-1)s})} + \|v\|_{L^2(\nu)}.\]
This implies 
\[\|e_k\|_{L^2(\nu_{sk})} \leq Cs^2\|v\|_{L^2(\nu)} + \|e_{k-1}\|_{L^2(\nu_{s(k-1)})}\left( 1 + Cs^2\right).\]
One can verify that the recurrence relation implies
\begin{multline} \label{eq: error recursion}
\|e_k\|_{L^2(\nu_{sk})} \leq Cs^2\|v\|_{L^2(\nu)}\sum_{i = 0}^k(1 + Cs^2)^i\\ = Cs^2\|v\|_{L^2(\nu)} \frac{(1 + Cs^2)^{k} - 1}{Cs^2} = \|v\|_{L^2(\nu)}\left((1 + Cs^2)^{k} - 1\right).\end{multline}
Since $k \leq N = 1/s$ it holds that $(1 + Cs^2)^k \leq (1 + Cs^2)^{1/s}  = 1 + Cs + O(s^2)$, where the last equality is shown in \Cref{recurrence relation bound}. It follows that $\|e_k\|_{L^2(\nu_{sk})} = O(s)$
for all $k \leq N$. 
\qed{}

\subsection{Proof of \Cref{assumption implications}} \label{sec: proof of assumption implications}
Under the assumptions of the result statement, the optimal transport coupling between $\mu$ and $\nu$ is supported on a map of the form $\nabla\varphi$ (\Cref{brenier}). For quadratic OT (as we have), the potential $\varphi$ solves the Monge-Ampere equation 
\[\text{det}(\nabla^2\varphi(x)) = f(x)/g(\nabla \varphi(x)) \quad \text{s.t.} \quad \nabla\varphi(\Omega) = \Omega^*.\]
By Theorem 1.1 of \citet{trudinger2013note}, $\nabla\varphi$ is a $C^2$ diffeomorphism from $\overline{\Omega}$ to $\overline{\Omega^*}$. It follows that $\nabla \varphi$ is bi-Lipschitz on $\overline{\Omega}$ since $\overline{\Omega}$ is compact. Now we will compute the Eulerian tangent field generating the geodesic and show that it is spatially Lipschitz. Define the interpolation map $F_t(x) \triangleq (1-t)x + t\nabla\varphi(x)$, and observe that its gradient is given by $\nabla F_t = (1-t)\text{id} + t\nabla^2 \varphi$.  Since $\nabla\varphi$ is a diffeomorphism between compact domains, $\nabla^2\varphi$ must have eigenvalues bounded away from zero 
\[\inf_{x \in \overline{\Omega}}\lambda_{\min}(\nabla^2\varphi) \geq m > 0\]
implying that 
\begin{equation} \lambda_{\min}(\nabla F_t) \geq 1 - t + t\cdot \lambda_{\min}(\nabla ^2\varphi) \geq \min\{1, m\} > 0 \label{min singular value of flow map jacobian}
\end{equation}
implying that $F_t$ is invertible. Moreover, since $\nabla\varphi$ is $C^2$ on a compact domain, $\|\nabla^2\varphi\|_{\text{op}}$ must be bounded on $\Omega$,
\[\sup_{x \in \Omega} \|\nabla^2\varphi(x)\|_{\text{op}} \leq M < \infty.\]
The Eulerian velocity field generating the geodesic is given by $v_t(y) = ((\nabla\varphi - \text{id}) \circ F^{-1}_t)(y)$, which implies \[\nabla v_t(y) = \nabla^2 \varphi(x)\cdot (\nabla F_t(x))^{-1} - (\nabla F_t(x))^{-1} = (\nabla^2 \varphi(x) - I_d)\cdot (\nabla F_t(x))^{-1}\quad x = F_t^{-1}(y).\] Define $\Omega_t = F_t(\Omega)$ and observe that
\[\sup_{y \in \Omega_t}\|\nabla v_t(y)\|_2 \leq \sup_{x\in \Omega}\|\nabla^2\varphi(x) - I_d\|_{\text{op}}\|(\nabla F_t(x))^{-1}\|_{\text{op}} \leq \frac{M+1}{\min\{1, m\}} < \infty.\]
Since the bound holds for all $t$, we know that ${\operatorname{Lip}}(v_t) \leq (M+1)/\min\{1,m\}$ for all $t$. Thus,
\[\int_0^1 {\operatorname{Lip}}(v_t) \, dt < \infty.\]
Now it remains to show that $\int_0^1\|v_t\|_{L^2(\mu_t)}^2 dt < \infty.$ We can bound this directly as follows,
\begin{align*}
    \int_0^1 \|v_t\|_{L^2(\mu_t)}^2 dt &= \int_0^1\int_\Omega\left\|\nabla\varphi(x) - x\right\|_2^2 \, d\mu \, dt \tag{$\mu_t = (F_t)_\#\mu$} \\
    &\leq \int_0^1 \sup_{x \in \Omega}\|\nabla \varphi(x) - x\|_2^2 \, dt \\
    & < \infty
\end{align*}
by compactness and continuity of $\nabla \varphi$. It follows that the geodesic induced by $\nabla \varphi$ is indeed regular.
\qed{}

\subsection{Proof of \Cref{one step error of W traj reconstruction}.} \label{proof of one step error of W traj reconstruction}
In this proof let $\exp$ denote the exponential map on $\mathbb{T}^d$. By the triangle inequality
\begin{align*}
    W_2(\hat\mu_{i+1}^*, \mu_{i+1}^*) &= W_2(\exp( \hat{v})_\#\hat{\mu}_i^*, \exp( \nabla\varphi_i^*)_\#\mu_i^*) \\
    &\leq W_2(\exp(\hat{v})_\#\hat{\mu}_i^*, \exp(\nabla\varphi_i^*)_\#\hat{\mu}_i^*) + W_2(\exp(\nabla\varphi_i^*)_\#\hat{\mu}_i^*, \exp(\nabla \varphi_i^*)_\#\mu_i^*).
\end{align*}
For the second term, recall that for any measurable map $F$ that is $L$-Lipschitz we have $W_2(F_\#\mu, F_\#\nu) \leq L\cdot W_2(\mu, \nu).$ To see this, let $\gamma^*$ be the optimal coupling of $\mu$ and $\nu$ and observe that
\begin{align*}
    W_2^2(F_\#\mu, F_\#\nu) &\leq \int_{\man \times \man}d^2(F(x), F(y))\,d\gamma^*(x,y) \\
    &\leq L^2\int_{\man \times \man}d^2(x, y)\,d\gamma^*(x,y) \\
    &= L^2\cdot W_2^2(\mu, \nu).
\end{align*}
Thus, we need to determine the Lipschitz constant of the map $\exp(\nabla\varphi_i^*)$. By assumption, $\nabla\varphi^*_i$ is spatially Lipschitz and the exponential map on $\mathbb{T}^d$ is globally Lipschitz with a constant $1$. Thus,
\[W_2(\exp(\nabla\varphi_i^*)_\#\hat{\mu}_i, \exp(\nabla \varphi_i^*)_\#\mu_i^*) \le (1 + {\operatorname{Lip}}(\nabla \varphi_i^*))\cdot W_2(\hat\mu_i^*, \mu_i^*).\]
For the first term in the original inequality, we can choose the suboptimal coupling 
\[(\exp_X(\hat{v}(X)), \exp_X(\nabla\varphi_i^*(X)))\] 
where $X \sim \hat\mu_i^*$. Thus,
\begin{align*}
    W_2^2(\exp(\hat v)_\#\hat{\mu}_i^*, \exp(\nabla\varphi_i^*)_\#\hat{\mu}_i^*) &\leq \int_{\man}d^2_{\man}\big(\exp_x(\hat{v}(x)), \exp_x(\nabla \varphi_i^*(x))\big) \, d\hat\mu_i^*(x) \\
    &\leq \int_{\man} \left\|\hat{v}(x) - \nabla\varphi_i^*(x)\right\|_{2}^2 \, d\hat\mu_i^*(x) \\
    &\leq  \|\hat{v} - \nabla\varphi_i^*\|_{L^2(\hat\mu_i^*)}^2.
\end{align*}
Recall that $\hat{v} = \widehat{{\PT}}_{\nu_i \rightarrow \hat\mu_i^*}(\nabla \varphi_i)$, while $\nabla\varphi_i^* = {\PT}_{\nu_i \rightarrow \mu_i^*}(\nabla \varphi_i)$. Thus, we will split up the discrepancy between the two with Minkowski's inequality
\begin{align*}
    \|\hat{v} - \nabla\varphi_i^*\|_{L^2(\hat\mu_i^*)} \leq \|\hat{v} - {\PT}_{\nu_i \rightarrow \hat\mu_i^*}(\nabla \varphi_i)\|_{L^2(\hat\mu_i^*)} + \|{\PT}_{\nu_i \rightarrow \hat\mu_i^*}(\nabla \varphi_i) - {\PT}_{\nu_i \rightarrow \mu_i^*}(\nabla\varphi_i)\|_{L^2(\hat\mu_i^*)}.
\end{align*}
\Cref{full approximation of parallel transport along regular curves} and the norm equivalence of $L^2(\mu_i^*; \mathbb{R}^d)$ and $L^2(\hat\mu_i^*; \mathbb{R}^d)$ imply that the first term is $O(N^{-1})$. The second term represents the stability Wasserstein parallel transport to perturbations of its destination measure. By the norm-equivalence of all measures in the admissible class and \Cref{wpt stability}, we know that the second term is bounded by $C_{\text{WPT}}\|\nabla \varphi_i\|_{H^1(\mathbb{T}^d)}W_2(\hat \mu_i^*, \mu_i^*).$ Putting it together, we obtain
\[W_2(\hat\mu_{i+1}^*, \mu_{i+1}^*) \leq \left((1 + {\operatorname{Lip}}(\nabla \varphi_i^*)) + \|\nabla \varphi_i\|_{H^1(\mathbb{T}^d)}C_{\text{WPT}} \right) W_2(\hat \mu_i^*, \mu_i^*) + O(N^{-1}).\]

\qed{}

\subsection{Proof of \Cref{pt recovers parallel trends}} \label{sec: proof of pt recovers parallel trends}
Let $R> 0$, let $a \in \mathbb{R}^d$ and define the test function $\psi_R(x) \triangleq \langle x, a \rangle \cdot \chi_R(x) \in C^1_c(\mathbb{R}^d)$ where $\chi_R(x) \in C^1(\mathbb{R}^d)$ and $\chi_R \equiv 1$ on $B(0,R)$ and $\chi_R \equiv 0$ outside $B(0, 2R)$ and $\|\nabla\chi_R(x)\|\lesssim R^{-1}$ for all $x$. This means that
\[\nabla \psi_R(x) = a\cdot\chi_R(x) + \langle x, a \rangle \nabla \chi_R(x)\]
Since $\|x\|_2/R \leq 2$ on the support of $\nabla \chi_R$, we can say
\[|\nabla \psi_R(x)| \le C\|a\|_2\]
for some $C$ independent of $R$, and $\nabla \psi_R(x) \rightarrow a$ as $R \rightarrow \infty$. Since $\nabla \varphi_t \in L^2(\nu_t; \mathbb{R}^d)$ for all $t$, we know that 
\[\int_0^1\int_{\mathbb{R}^d}\|\nabla \varphi_t\|_2 \, d\nu_t dt  < \infty.\]
Applying the dominated convergence theorem yields,
\[\int_{\mathbb{R}^d} \langle \nabla \psi_R, \nabla\varphi_t\rangle \, d\nu_t \rightarrow \int_{\mathbb{R}^d}\langle a, \nabla\varphi_t\rangle\, d\nu_t.\]
Further note that $\|\psi_R(x)\|_2\leq\|a\|_2\|x\|_2$, and for all $t \in [0, 1]$
\[\int_{\mathbb{R}^d}\|x\|_2\|a\|_2 \, d\nu_t \leq \|a\|_2M\]
for some $M$ independent of $t$. Applying the dominated convergence theorem again implies that
\[\int_{\mathbb{R}^d}\psi_R(x)\, d\nu_t \rightarrow \int_{\mathbb{R}^d}\langle x, a\rangle \, d\nu_t.\]
Further let $\eta(t) \in C_c^1((0, 1))$ and define $\phi_R(x, t) \triangleq \eta(t)\psi_R(x)$. Since $\nabla\varphi_t$ is the tangent velocity at $\nu_t$, we have
\begin{align*}
    \int_0^1 \int_{\mathbb{R}^d}\partial_t\phi_R \, d\nu_t dt + \int_0^1\int_{\mathbb{R}^d} \langle\nabla \phi_R, \nabla \varphi_t\rangle \,d\nu_tdt &= 0
\end{align*}
which implies 
\[\int_0^1\eta'(t)\underbrace{\int_{\mathbb{R}^d}\psi_R \, d\nu_t}_{\triangleq F_R(t)}dt  = -\int_0^1\eta(t)\underbrace{\int_{\mathbb{R}^d} \langle \nabla \psi_R(x), \nabla \varphi_t(x)\rangle \, d\nu_t}_{\triangleq G_R(t)}dt.\]
Thus, $G_R(t)$ is the \textit{weak} derivative of $F_R(t)$ -- in particular, this means that for a.e. $t$,
\[\frac{d}{dt} \int_{\mathbb{R}^d} \chi_R(x)\,\langle a, x\rangle \, d\nu_t(x) = \int_{\mathbb{R}^d} \langle \nabla \psi_R(x), \nabla \varphi_t(x)\rangle \, \, d\nu_t.\]
Passing to the limit yields
\[\frac{d}{dt} \int_{\mathbb{R}^d}\,\langle a, x\rangle \, d\nu_t(x) = \underbrace{\int_{\mathbb{R}^d} \langle a, \nabla \varphi_t(x)\rangle \, \, d\nu_t}_{T_{t}}.\]
By the same logic, we arrive at the analogous conclusion for $\mu_t^*$. In particular, 
\[\frac{d}{dt}\int_{\mathbb{R}^d} \langle a, x\rangle \, d\mu_t^*(x) = \underbrace{\int_{\mathbb{R}^d} \langle a, \nabla \varphi_t^*\rangle \, d\mu_t^*}_{T^*_t}. \]
Thus if we can show $T_t = T^*_t$, then 
\[\frac{d}{dt} \int_{\mathbb{R}^d} \langle a, x\rangle \, d\nu_t(x) = \frac{d}{dt} \int_{\mathbb{R}^d} \langle a, x\rangle \, d\mu_t^*(x) \quad  \forall \, a \in \mathbb{R}^d \]
which implies the theorem statement. By the isometry of parallel transport,
\begin{align*}
    T_t^* &= \langle a, \nabla\varphi_t^* \rangle_{L^2(\mu_t^*)} = \langle {\PT}_{\mu_t^* \rightarrow \nu_t}(a), \nabla\varphi_t\rangle_{L^2(\nu_t)}.
\end{align*}
By \Cref{affine PT preserves constant fields}, we know that ${\PT}_{\mu_t^* \rightarrow \nu_t}(a) = a$, implying that $T_t^* = T_t$  almost everywhere. \qed{} 

\begin{restatable}[Wasserstein Parallel Transport preserves constant fields.]{prop}{} \label{affine PT preserves constant fields}
    Let $(\lambda_t)_{t \in [0, 1]}$ be a regular curve of measures in $\mathcal{P}_2(\mathbb{R}^d)$ with a tangent velocity field $\nabla \varphi_t$, and let $a \in \mathbb{R}^d$. Then for any $t_0, t_1 \in [0, 1]$ we have 
    \[{\PT}_{\lambda_{t_0} \rightarrow \lambda_{t_1}}(a) = a.\]
\end{restatable}
\noindent \textit{Proof.} To prove the statement, we need to show that
\[\nabla \cdot (\lambda_t(\partial_ta + \nabla a \cdot \nabla \varphi_t)) = 0\]
for a.e. $t$. Since $a$ is unchanging in $t$ or spatially, both terms in the parentheses are zero and the equation is trivially satisfied. \qed{}

\subsection{Supplementary Results}

\begin{restatable}[Gr\"onwall's Inequality]{lemma}{} \label{gronwall}
    Suppose $\frac{d}{dt}f(t) \leq u(t)f(t) + c$ for $t \geq 0$ where $c$ is a constant. Then 
    \[f(t) \leq f(0)\exp\left(\int_0^t u(s)\,ds\right) + c\int_0^t\exp \left(\int_s^tu(r)\,dr\right)\,ds.\]
\end{restatable}
\noindent \textit{Proof.} Let $\mu(t) = \exp(-\int_0^t u(s)\,ds)$ and observe that $\frac{d}{dt}\mu(t) = -\mu(t) u(t).$ Now multiply both sides of the inequality in the hypothesis by $\mu(t)$,
\[\mu(t)\frac{d}{dt}f(t) \leq \mu(t)u(t)f(t) + c\mu(t)\]
and observe that $\frac{d}{dt}(\mu(t)f(t)) = -\mu(t)u(t)f(t) + \mu(t)\frac{d}{dt}f(t).$ Thus,
\[\frac{d}{dt}\left(\mu(t)f(t)\right) \leq c\mu(t).\]
Integrating from $s = 0$ to $s = t$ gives
\[\mu(t)f(t) \leq f(0) + c\int_0^t \mu(s) \, ds \]
and multiplying through by $\mu(t)^{-1}$ gives
\[f(t) \leq f(0)\exp\left(\int_0^tu(s)\, ds\right) + \int_0^t\exp\left(\int_s^t u(r)\,dr\right)\,ds.\]

\qed{}

\begin{restatable}[Bounded Linear Operator]{prop}{} \label{bounded linear operator}
    Define the linear operator $A: \mathcal{H} \rightarrow L^2(\mu; \mathbb{R}^d)$ to be $Af = \nabla f$. Under the hypotheses of \Cref{RKHS consistency}, $A$ is a bounded linear operator. 
\end{restatable}
\noindent \textit{Proof.} By the derivative reproducing property shown in \Cref{gradient representer thm}, we know that $\partial_{j}f(x) = \langle \psi_j(x), f \rangle_{\mathcal{H}}$. Thus for all $x \in \mathbb{R}^d$
\[\|\nabla f(x)\|_2^2 = \sum_{i = 1}^d|\langle \psi_i(x), f\rangle_{\mathcal{H}}|^2 \leq \|f\|_{\mathcal{H}}^2\sum_{i = 1}^d\|\psi_i(x)\|_{\mathcal{H}}^2 \leq \|f\|_{\mathcal{H}}^2\kappa^2.\]
It follows that $\|\nabla f\|_{L^2(\mu)}^2 \leq \kappa^2 \|f\|_{\mathcal{H}}^2$ implying that $\|A\|_{\text{op}} \leq \kappa$ and $A$ is indeed a bounded linear operator. \qed{}

\begin{restatable}{lemma}{}
    For $s \rightarrow 0$ and a fixed constant $C > 0$, $(1 + Cs^2)^{1/s} = 1 + Cs + O(s^2).$ \label{recurrence relation bound}
\end{restatable}
\noindent \textit{Proof.} We will start by obtaining an asymptotic expression for the log of the quantity of interest,
\begin{align*}
    \log\left((1 + Cs^2)^{1/s}\right) &= \frac{1}{s}\log(1 + Cs^2) \\
    &= \frac{1}{s}\left(Cs^2 - \frac{C^2s^4}{2} + \frac{C^3s^6}{3} + O(s^8)\right)
\end{align*}
since $\log(1 + u) = u - \frac{u^2}{2} + \frac{u^3}{3} +O(u^4)$ for $u \rightarrow 0$. Multiplying out and exponentiating yields
\begin{align*}
    (1 + Cs^2)^{1/s} &= \exp\left(Cs - \frac{C^2s^3}{2} + \frac{C^3s^5}{3} + O(s^7)\right) \\
    &= 1 + Cs + O(s^2)
\end{align*}
since $e^u = 1 + u + O(u^2)$ for $u \rightarrow 0$.

\qed{}

\begin{restatable}{lemma}{} \label{L2}
Let $(\Omega,\|\cdot\|_\Omega)$ be a normed measurable space, let
$\lambda_0,\lambda_1\in\mathcal P_2(\Omega)$, and let
$v:\Omega\to\Omega$ be measurable and spatially Lipschitz. If
$v\in L_2(\lambda_0;\Omega)$, then also $v\in L_2(\lambda_1;\Omega)$.
Moreover,
\[
\|v\|_{L_2(\lambda_1)}^2
\le
2\|v\|_{L_2(\lambda_0)}^2
+
2{\operatorname{Lip}}(v)^2 W_2^2(\lambda_0,\lambda_1).
\]
\end{restatable}

\noindent\textit{Proof.}
Let $\gamma$ be an optimal coupling of $\lambda_0$ and $\lambda_1$, and let
$(X,Y)\sim \gamma$. Then
\begin{align*}
\|v(Y)\|_\Omega^2
&\le 2\|v(X)\|_\Omega^2 + 2\|v(Y)-v(X)\|_\Omega^2 \\
&\le 2\|v(X)\|_\Omega^2 + 2{\operatorname{Lip}}(v)^2\|Y-X\|_\Omega^2.
\end{align*}
Taking expectations yields
\begin{align*}
\int \|v\|_\Omega^2\,d\lambda_1
= \mathbb E\|v(Y)\|_\Omega^2
&\le
2\mathbb E\|v(X)\|_\Omega^2
+
2{\operatorname{Lip}}(v)^2 \mathbb E\|X-Y\|_\Omega^2 \\
&=
2\int \|v\|_\Omega^2\,d\lambda_0
+
2{\operatorname{Lip}}(v)^2 W_2^2(\lambda_0,\lambda_1)
<\infty.
\end{align*}
Thus $v\in L_2(\lambda_1;\Omega)$.
\qed{}

\section{RKHS Derivation} \label{sec: RKHS empirical objective derivation}

\paragraph{RKHS Empirical Objective.} 
We now derive the empirical objective and the associated linear system for the RKHS-based
Helmholtz projection estimator introduced in \Cref{sec: helmholtz decomp}. Let $\mathcal{H}$ be a scalar RKHS on \(\mathbb{R}^d\) with a twice-differentiable kernel
\(K \in C^2(\mathbb{R}^d \times \mathbb{R}^d)\). By \Cref{gradient representer thm}, any minimizer of
\[
\hat f_\lambda \in \arg\min_{f \in \mathcal{H}}
\left(
\frac{1}{n}\sum_{i=1}^n \|\nabla f(x_i) - v(x_i)\|_2^2
+ \lambda \|f\|_{\mathcal{H}}^2
\right)
\]
admits a representation of the form
\[
\hat f_\lambda(\cdot)
=
\sum_{i=1}^n \langle c_i, \nabla_2 K(x_i,\cdot)\rangle,
\]
where \(\nabla_2\) denotes the gradient with respect to the second argument of \(K\) and $c_i \in \mathbb{R}^d$.
Equivalently, writing \(c_i = (c_{i1},\dots,c_{id})^{\top}\),
\[
\hat f_\lambda(\cdot)
=
\sum_{i=1}^n \sum_{a=1}^d c_{ia}\,\partial_{2,a}K(x_i,\cdot).
\]
Define the stacked coefficient vector and the stacked target vector
\[
c \triangleq (c_1^{\top},\dots,c_n^{\top})^{\top} \in \mathbb{R}^{nd}, \qquad v \triangleq \bigl(v(x_1)^\top,\dots,v(x_n)^\top\bigr)^\top \in \mathbb{R}^{nd}.
\]
\paragraph{Gradients of the representer expansion.} Differentiating \(\hat f_\lambda\) with respect to its argument gives, for each \(j \in \{1,\dots,n\}\),
\[\nabla \hat f_\lambda(x_j)
=
\sum_{i=1}^n \nabla_2^2 K(x_i,x_j)c_i,
\]
where \(\nabla_2^2 K(x_i,x_j)\in\mathbb{R}^{d\times d}\) is the Hessian with respect to the
second argument. Thus, if we define the block matrix \(D \in \mathbb{R}^{nd\times nd}\) by
\[
D_{ji} \triangleq \nabla_2^2 K(x_i,x_j)\in\mathbb{R}^{d\times d},
\]
for $i,j \in \{1,\dots,n\}$, then the stacked fitted gradients satisfy
\[
\begin{bmatrix}
\nabla \hat f_\lambda(x_1)\\
\vdots\\
\nabla \hat f_\lambda(x_n)
\end{bmatrix}
=
Dc.
\]
\paragraph{RKHS norm of the representer expansion.} Now we will compute \(\|\hat f_\lambda\|_\mathcal{H}^2\). By bilinearity of the RKHS inner product,
\begin{align*}
\|\hat f_\lambda\|_{\mathcal{H}}^2
&=
\left\langle
\sum_{i=1}^n \sum_{a=1}^d c_{ia}\,\partial_{2,a}K(x_i,\cdot),
\sum_{j=1}^n \sum_{b=1}^d c_{jb}\,\partial_{2,b}K(x_j,\cdot)
\right\rangle_{\mathcal{H}} \\
&=
\sum_{i,j=1}^n \sum_{a,b=1}^d
c_{ia}c_{jb}
\left\langle
\partial_{2,a}K(x_i,\cdot),
\partial_{2,b}K(x_j,\cdot)
\right\rangle_{\mathcal{H}}.
\end{align*}
By the derivative reproducing property,
\[
\left\langle
\partial_{2,a}K(x_i,\cdot),
\partial_{2,b}K(x_j,\cdot)
\right\rangle_{\mathcal{H}}
=
\partial_{1,a}\partial_{2,b}K(x_i,x_j).
\]
Therefore, if we define the block Gram matrix \(G \in \mathbb{R}^{nd\times nd}\) by
\[
G_{ij}
\triangleq
\nabla_1\nabla_2 K(x_i,x_j)\in\mathbb{R}^{d\times d},
\]
for $i,j \in \{1,\dots,n\}$ then
\[
\|\hat f_\lambda\|_{\mathcal{H}}^2 = c^\top G c.
\]
\paragraph{The empirical objective in coefficient form.}
Substituting the expressions above into the empirical objective yields
\[
J(c)
=
\frac{1}{n}\|Dc - v\|_2^2 + \lambda\, c^\top G c.
\]
Expanding \(J(c)\),
\[
J(c)
=
\frac{1}{n}(Dc-v)^\top(Dc-v) + \lambda c^\top G c.
\]
Differentiating with respect to \(c\) gives
\[
\nabla_c J(c)
=
\frac{2}{n}D^\top(Dc-v) + 2\lambda Gc.
\]
Hence any minimizer \(c^\star\) satisfies the normal equations
\[
D^\top D\,c^\star + n\lambda G\,c^\star = D^\top v,
\]
or equivalently,
\[
\bigl(D^\top D + n\lambda G\bigr)c^\star = D^\top v.
\]
Whenever \(D^\top D + n\lambda G\) is invertible, the minimizer is unique and given by
\[
c^\star
=
\bigl(D^\top D + n\lambda G\bigr)^{-1}D^\top v.
\]
Note that the simplified one-matrix formula
\[
J(c)=\frac{1}{n}\|Hc-v\|_2^2+\lambda c^\top H c,
\qquad
c^\star=(H+n\lambda I_{nd})^{-1}v,
\]
is recovered under the additional identification \(D=G=H\) (with \(H\) symmetric).

\section{Implementation Details}
\label{appendix:implementation_details}

Our implementation of the algorithms described in \Cref{sec: wpt approx}, \Cref{sec: helmholtz decomp} and \Cref{sec:counterfactual_dyn} is available at
\[
\texttt{\url{https://github.com/TristanSaidi/WassersteinPT}}.
\]
The codebase contains the main routines for approximate Wasserstein parallel transport and
counterfactual dynamics prediction, together with supporting geometry and experiment code. In
particular, the public repository includes source files \texttt{pt.py}, \texttt{cf\_recon.py}, and
\texttt{geom.py}.

\paragraph{Empirical optimal transport and tangent estimation.}
In the theoretical development, \Cref{alg: W parallel transport} is written in terms of the Brenier map and the
Wasserstein logarithmic map. In the empirical setting, however, we work with empirical measures
\[
\hat \nu = \sum_{i=1}^n a_i \delta_{x_i},
\qquad
\hat \mu = \sum_{j=1}^m b_j \delta_{y_j},
\]
and we compute an optimal coupling
\[
\Gamma^\star \in \mathbb{R}_{+}^{n\times m}
\]
for the quadratic transport cost. Concretely, \(\Gamma^\star\) is obtained by solving the discrete
optimal transport problem
\[
\min_{\Gamma \in \Pi(a,b)} \sum_{i=1}^n \sum_{j=1}^m \Gamma_{ij}\|x_i-y_j\|_2^2,
\]
where \(\Pi(a,b)\) denotes the set of nonnegative matrices with row sums \(a\) and column sums
\(b\). Given the optimal coupling, we convert it into a tangent vector by barycentric projection. That is,
for each source support point \(x_i\), we define the empirical transport target
\[
\bar y_i
\triangleq
\frac{\sum_{j=1}^m \Gamma^\star_{ij} y_j}{\sum_{j=1}^m \Gamma^\star_{ij}},
\]
and then set the empirical tangent vector to be
\[
\hat v(x_i)
\triangleq
\bar y_i - x_i
=
\frac{1}{a_i}\sum_{j=1}^m \Gamma^\star_{ij}(y_j-x_i).
\]
Thus, in practice, the logarithmic map is approximated by the displacement induced by the
barycentric projection of the empirical OT plan. When the discrete coupling is supported on a map,
this reduces to the usual pointwise displacement \(T(x_i)-x_i\); when mass from \(x_i\) splits
across multiple targets, the barycentric projection provides the canonical single-vector summary
used by the implementation.

\paragraph{Weighted aggregation when the barycentric projection is not injective.}
A practical complication arises in the iterative parallel transport procedure when the empirical
barycentric projection is not injective. In that case, multiple source support points may be mapped
to the same destination support point after one transport step. A naive pointwise pullback or
pushforward of tangent vectors is then ill-posed, because there is no unique source vector at
the destination location. To handle this, the implementation performs a weighted aggregation of the transported vectors. Suppose the current tangent field is represented by vectors \(z_1,\dots,z_n\) on source locations
\(x_1,\dots,x_n\), and suppose the one-step empirical transport from the current support to the next
support is represented by a coupling matrix \(\Gamma\). For each destination point \(y_j\), we define
the transported tangent by the conditional barycenter
\[
\tilde z_j
\triangleq
\frac{\sum_{i=1}^n \Gamma_{ij} z_i}{\sum_{i=1}^n \Gamma_{ij}}.
\]
Equivalently, the new vector at a destination support point is the mass-weighted average of all
incoming vectors under the one-step coupling. This aggregation rule has two desirable properties. First, it is consistent with the deterministic
case: if the step is induced by an injective map, then each destination point has exactly one
preimage and the formula above simply recovers the transported vector associated with that
preimage. Second, when several source points merge, it preserves the correct mass weighting induced
by the empirical transport plan rather than arbitrarily selecting one source vector. In this
sense, the discrete implementation uses the coupling itself to define the natural empirical analogue
of transporting a tangent field through a non-invertible step. We also note that under standard regularity assumptions the barycentric projection of the empirical optimal coupling converges to the population Brenier map in the large-sample limit \citep[Theorem 2.2 and Corollary 2.3]{deb2021rates}.

\paragraph{RKHS projection via random Fourier features.}
\Cref{sec: helmholtz decomp} describes the Helmholtz-Hodge projection step through a kernel gradient regression problem in an RKHS. The exact formulation leads to a linear system involving a block kernel matrix whose size scales with both the sample size and the ambient dimension. In practice, this becomes a computational bottleneck in the high-dimensional settings considered in our experiments. To reduce both runtime and memory usage, the implementation replaces the exact kernel expansion with a random Fourier feature approximation. Specifically, for a shift-invariant kernel \(K\), we use
a feature map $\phi:\mathbb{R}^d \to \mathbb{R}^D$
such that
\[
K(x,y) \approx \phi(x)^{\top} \phi(y),
\]
where \(D\) is the number of random features. We then parameterize the scalar potential as
\[
f_\theta(x) = \theta^{\top} \phi(x),
\qquad \theta \in \mathbb{R}^D,
\]
so that the projected vector field is represented by
\[
\nabla f_\theta(x)
=
J_\phi(x)^{\top} \theta,
\]
with \(J_\phi(x)\) denoting the Jacobian of the feature map. Under this approximation, the empirical projection problem becomes a finite-dimensional ridge
regression problem in the feature parameters \(\theta\):
\[
\min_{\theta \in \mathbb{R}^D}
\frac{1}{n}\sum_{i=1}^n
\bigl\|
J_\phi(x_i)^{\top} \theta - v(x_i)
\bigr\|_2^2
+
\lambda \|\theta\|_2^2.
\]
This avoids forming the full \(nd \times nd\) kernel matrix from the exact representer expansion,
and instead works with feature matrices whose width is the chosen number of random features.
Consequently, the memory complexity is reduced from kernel-matrix storage to feature-matrix
storage, and the resulting linear algebra is substantially faster in the regimes relevant to our
experiments.

\end{document}